\documentclass{article} %
\usepackage{iclr2025_conference,times}

\usepackage{hyperref}
\usepackage{url}

\PassOptionsToPackage{unicode}{hyperref}
\usepackage{amsmath}
\usepackage{amsthm}
\usepackage{subfigure}
\usepackage{upref}
\usepackage{mathtools}
\usepackage{amssymb}
\usepackage{mathtools}
\usepackage{booktabs}
\usepackage[vlined,linesnumbered,ruled]{algorithm2e}
\let\oldnl\nl%
\newcommand{\nonl}{\renewcommand{\nl}{\let\nl\oldnl}}
\usepackage{thmtools, thm-restate}

\declaretheorem{lemma}

\theoremstyle{definition}
\newtheorem{remark}{Remark}

\theoremstyle{definition}

\newtheorem{assumption}{Assumption}

\usepackage [operators,sets,landau,complexity]{cryptocode}
\usepackage{bm}
\renewcommand{\vec}[1]{\bm{#1}}
\DeclareMathOperator*{\E}{\mathbb{E}}
\DeclareMathOperator*{\Var}{\mathrm{Var}}

\def\NN{{\mathbb N}}
\def\RR{{\mathbb R}}

\newcommand{\tr}{{\rm tr}}
\usepackage{bbm}
\usepackage{physics2} %
\usephysicsmodule{ab}
\usepackage{derivative}
\usepackage{blkarray}
\newcommand{\1}{\mathbbm{1}}
\newcommand{\inprod}[2]{\left\langle #1, #2 \right\rangle}
\newcommand\numberthis{\addtocounter{equation}{1}\tag{\theequation}}
\usepackage{xcolor}
\hypersetup{
    colorlinks,
    linkcolor={red!50!black},
    citecolor={blue!50!black},
    urlcolor={blue!80!black}
}
\usepackage{xspace}
\newcommand{\UniCLUB}{\texttt{UniCLUB}\xspace}
\newcommand{\PhaseUniCLUB}{\texttt{PhaseUniCLUB}\xspace}
\newcommand{\UniSCLUB}{\texttt{UniSCLUB}\xspace}
\newcommand{\SACLUB}{\texttt{SACLUB}\xspace}
\newcommand{\SASCLUB}{\texttt{SASCLUB}\xspace}
\usepackage{cleveref}
\usepackage{colortbl}
\usepackage{enumitem}
\usepackage{tablefootnote}
\usepackage{threeparttable}
\usepackage{diagbox}

\title{Demystifying Online Clustering of Bandits: Enhanced Exploration Under Stochastic and Smoothed Adversarial Contexts}

\newcommand*\samethanks[1][\value{footnote}]{\footnotemark[#1]}
\author{Zhuohua Li\thanks{Zhuohua Li and Maoli Liu contributed equally to this work. Maoli Liu is the corresponding author.}, Maoli Liu\samethanks, Xiangxiang Dai, John C.S. Lui\\
Department of Computer Science and Engineering\\
The Chinese University of Hong Kong\\
Shatin, N.T., Hong Kong SAR \\
\texttt{\{zhli,mlliu,xxdai23,cslui\}@cse.cuhk.edu.hk}
}

\iclrfinalcopy %
\begin{document}

\maketitle

\begin{abstract}
  The contextual multi-armed bandit (MAB) problem is crucial in sequential decision-making.
  A line of research, known as online clustering of bandits, extends contextual MAB by grouping similar users into clusters, utilizing shared features to improve learning efficiency.
  However, existing algorithms, which rely on the upper confidence bound (UCB) strategy, struggle to gather adequate statistical information to accurately identify unknown user clusters.
  As a result, their theoretical analyses require several strong assumptions about the ``diversity'' of contexts generated by the environment, leading to impractical settings, complicated analyses, and poor practical performance.
  Removing these assumptions has been a long-standing \emph{open problem} in the clustering of bandits literature.
  In this paper, we provide two solutions to this open problem.
  First, following the \emph{i.i.d.} context generation setting in existing studies, we propose two novel algorithms, \UniCLUB and \PhaseUniCLUB, which incorporate enhanced exploration mechanisms to accelerate cluster identification.
  Remarkably, our algorithms require substantially weaker assumptions while achieving regret bounds comparable to prior work.
  Second, inspired by the smoothed analysis framework, we propose a more practical setting that eliminates the requirement for \emph{i.i.d.} context generation used in previous studies, thus enhancing the performance of existing algorithms for online clustering of bandits.
  Our technique can be applied to both graph-based and set-based clustering of bandits frameworks.
  Extensive evaluations on both synthetic and real-world datasets demonstrate that our proposed algorithms consistently outperform existing approaches.
\end{abstract}

\section{Introduction}
\label{sec:introduction}
The stochastic multi-armed bandit (MAB) problem is an online sequential decision-making problem, where at each time step, the learner selects an action (a.k.a. arm) and observes a reward generated from an unknown probability distribution associated with that arm.
The goal of the learner is to maximize cumulative rewards (or equivalently, minimize cumulative regrets) in the long run.
The contextual linear bandit problem~\citep{li-2010-a-contextual,chu-2011-contextual} extends the MAB framework by associating each action with a feature vector and a corresponding unknown linear reward function.

\emph{Online clustering of bandits}, first introduced by \citet{gentile-2014-online-clustering}, generalizes contextual linear bandits by utilizing preference relationships among users.
It adaptively partitions users into clusters and leverages the collaborative effect of similar users to enhance learning efficiency.
This approach has many applications in computational advertising, web page content optimization, and recommendation systems~\citep{li-2018-online-clustering}.
Different from conventional MAB problems that focus solely on regret minimization, online clustering of bandits has two \emph{simultaneous} goals.
Firstly, it infers the underlying cluster structures among users by sequentially recommending arms and receiving user feedback.
Secondly, based on the inferred clusters, it minimizes the cumulative regret along the learning trajectory.
These dual goals significantly influence the algorithm design, as the learner must balance the accurate cluster inference and effective regret minimization.

Most existing studies such as \citet{gentile-2014-online-clustering,li-2018-online-clustering} employ an Upper Confidence Bound (UCB)-based strategy~\citep{abbasi-2011-improved} to balance exploration (for cluster inference) and exploitation (for regret minimization).
While this strategy is intuitive and standard for stochastic linear bandits, the least squares estimator used in the UCB strategy does not directly yield a precise estimation of underlying parameters, leading to insufficient statistical information for cluster inference.
As a result, to ensure correct cluster inference, existing algorithms for clustering of bandits require several strong assumptions on ``data diversity'' for their regret analysis, such as contexts being independently generated from a fixed random process (we refer to it as stochastic context setting) with lower bounded minimum eigenvalue of a covariance matrix and upper bounded variance (See \Cref{sec:key-techniques} for details).
Unfortunately, these assumptions result in impractical settings, overly complicated theoretical analysis, suboptimal regret incurred by cluster inference, and more importantly, poor performance in practice.
Based on these challenges, a natural question arises:
\begin{quote}
  \emph{Can we design new algorithms or propose new settings for the online clustering of bandits that eliminate the restrictive assumptions while achieving similar regrets?}
\end{quote}
In fact, this question reflects an \emph{open problem} discussed in Section A.7 of \citet{gentile-2014-online-clustering}, where the authors questioned ``\emph{whether the i.i.d. and other statistical assumptions they made could be removed}''.
While several efforts have been made to relax these assumptions, previous attempts have not addressed the fundamental issues, and some have resulted in significantly deteriorated regret that grows exponentially with the number of arms, in exchange for milder assumptions (See \Cref{sec:related-work} for details).
This underscores the inherent difficulty of the problem.

In this paper, we clarify the limitations of existing methods for cluster inference and address the aforementioned open problem.
We show that the inherent lack of sufficient exploration in UCB-like strategies impedes the efficient inference of underlying cluster structures.
Therefore, to simultaneously achieve effective cluster inference and regret minimization, additional exploration must be incorporated.
To this end, we propose two approaches, one focusing on algorithmic design and the other on modifications to the problem setup:

\textbf{Algorithmic Design Perspective}: In \Cref{sec:stochastic-context-setting}, we maintain the stochastic context setting as in \citet{gentile-2014-online-clustering} and propose new algorithms with an enhanced exploration mechanism.
This mechanism forces additional exploration beyond the standard UCB approach, leading to more reasonable settings, significantly relaxed assumptions, and comparable cumulative regret.
Intuitively, the additional exploration gathers more information about the underlying clusters, preventing the UCB strategy from exploiting inaccurate cluster estimates.
This technique is quite general and applicable to both graph-based~\citep{gentile-2014-online-clustering} and set-based~\citep{li-2019-improved-algorithm} algorithms.
Furthermore, our technique may hold independent value for broader research on multi-objective MAB problems, as discussed in \Cref{sec:more-related-works}.

\textbf{Problem Setup Perspective}: In \Cref{sec:smoothed-adversarial-context-setting}, we eliminate the need for stochastic context generation by proposing a new setup based on the smoothed analysis framework~\citep{spielman-2004-smoothed}, where the contexts are chosen by a ``smoothed'' adversary.
This setting interpolates between two extremes: the \emph{i.i.d.} context generation used in clustering of bandits~\citep{gentile-2014-online-clustering} and the adversarial context generation used in standard linear bandits~\citep{abbasi-2011-improved}.
We show that with some minor changes, existing algorithms (such as CLUB proposed by \citet{gentile-2014-online-clustering}) can achieve better performance in this setting.

To the best of our knowledge, our work is the first to propose new algorithms and settings to eliminate the restrictive assumptions prevalent in the clustering of bandits literature~\citep{gentile-2014-online-clustering,li-2018-online-clustering,li-2019-improved-algorithm,liu-2022-federated,wang-2023-online-clustering,wang-2023-online-corrupted}, while maintaining comparable cumulative regrets.
In doing so, we address the open problem posed by \citet{gentile-2014-online-clustering}.

\textbf{Main contributions}. Our contributions are highlighted as follows:
\begin{itemize}[leftmargin=*]
  \item Following the stochastic context setting established in \citet{gentile-2014-online-clustering}, we propose two graph-based algorithms, \UniCLUB and \PhaseUniCLUB, based on CLUB~\citep{gentile-2014-online-clustering}. Both algorithms incorporate additional exploration mechanisms besides the conventional UCB strategy. Benefiting from our novel design and analytical techniques, we substantially relax the assumptions required for theoretical analysis. To demonstrate the versatility of our approach, we also present a set-based algorithm, \UniSCLUB, which extends SCLUB~\citep{li-2019-improved-algorithm}.
  \item We demonstrate that both \UniCLUB and \UniSCLUB enjoy an improved regret bound of \(\widetilde{O}\ab(u\frac{d}{\gamma^2\lambda_x} + d\sqrt{mT})\), assuming the minimum gap between clusters \(\gamma\) is known. The first term in this bound improves the state-of-the-art regret in existing literature~\citep{gentile-2014-online-clustering}, and the second term matches the minimax near-optimal regret bound of contextual linear bandits~\citep{abbasi-2011-improved}. Furthermore, we show that when the cluster gap is unknown, \PhaseUniCLUB achieves a regret of \(\widetilde{O}\ab(u\frac{d}{\gamma^5\lambda_x^2} + d\sqrt{mT})\), which fully eliminates the stringent assumptions in \citet{gentile-2014-online-clustering}, albeit with a larger logarithmic term.
  \item Besides the stochastic context setting, we propose a new setting called the smoothed adversarial context setting, where in each round, the context is chosen by an adversary but is then randomly perturbed. This setup is more practical and aligns more closely with the original setting of contextual linear bandits~\citep{abbasi-2011-improved}. We give two associated algorithms \SACLUB, \SASCLUB and prove that they enjoy a regret of \(\widetilde{O}\ab(u\frac{d}{\gamma^2\widetilde{\lambda}_x} + d\sqrt{mT})\) where \(\widetilde{\lambda}_x = O\ab(\frac{1}{\log K})\).
  \item We perform extensive evaluations on both synthetic and real-world datasets. The results demonstrate that our algorithms outperform all the baseline algorithms, validating their effectiveness and practical applicability in various settings.
\end{itemize}

\Cref{tab:comparison} summarizes the comparison between our algorithms and existing studies with different assumptions and cumulative regrets.
Detailed explanations of the diversity conditions and regret analysis are given in \Cref{sec:key-techniques} and \Cref{sec:theoretical-analysis}.

\vspace{-3pt}
\renewcommand{\arraystretch}{1.2} %
\begin{table}[htb]
\centering
\setlength{\extrarowheight}{0pt}
\addtolength{\extrarowheight}{\aboverulesep}
\addtolength{\extrarowheight}{\belowrulesep}
\setlength{\aboverulesep}{0pt}
\setlength{\belowrulesep}{0pt}
\caption{Comparison of our algorithms with existing studies with different design choices. Since all algorithms achieve similar regret bounds, we focus on comparing the regret incurred by clustering.}
\label{tab:comparison}
\resizebox{\linewidth}{!}{%
\begin{threeparttable}
\begin{tabular}{lllll}
\toprule
\textbf{Algorithms}                                                                                                                                                                                                & \begin{tabular}[c]{@{}l@{}}\textbf{Context}\\\textbf{Generation}\end{tabular} & \textbf{Diversity Assumption}                                                                                                                                                                                          & \begin{tabular}[c]{@{}l@{}}\textbf{Regret incurred}\\\textbf{by clustering}\end{tabular}              & \textbf{Constants}                                \\
\midrule
\begin{tabular}[c]{@{}l@{}}CLUB \citep{gentile-2014-online-clustering},\\CLUB-cascade \citep{li-2018-online-clustering},\\SCLUB \citep{li-2019-improved-algorithm},\\FCLUB \citep{liu-2022-federated}\end{tabular} & i.i.d.                                                                        & \begin{tabular}[c]{@{}l@{}}$\lambda_{\text{min}}(\E[\vec{X}\vec{X}^\mathsf{T}])=\lambda_x$,\\$(\vec{z}^\mathsf{T}\vec{X})^2$ is $\sigma^2$-sub-Gaussian,\\$\sigma^2 \leq \frac{\lambda_x^2}{8 \log(4K)}.$\end{tabular} & $\widetilde{O}\ab(u\ab(\frac{d}{\gamma^2\lambda_x} + \frac{1}{\lambda_x^2}))$                         & $\lambda_x=O\ab(\frac{1}{d})$                     \\
\hline
\begin{tabular}[c]{@{}l@{}}RCLUMB, RSCLUMB \citep{wang-2023-online-clustering},\\RCLUB-WCU \citep{wang-2023-online-corrupted},\\FedC$^{3}$UCB-H \citep{yang-2024-federated}\end{tabular}                           & i.i.d.                                                                        & \begin{tabular}[c]{@{}l@{}}$\lambda_{\text{min}}(\E[\vec{X}\vec{X}^\mathsf{T}])=\lambda_x$,\\$(\vec{z}^\mathsf{T}\vec{X})^2$ is $\sigma^2$-sub-Gaussian.\end{tabular}                                                  & $\widetilde{O}\ab(u\ab(\frac{d}{\gamma^2\widetilde{\lambda}_x} + \frac{1}{\widetilde{\lambda}_x^2}))$ & $\widetilde{\lambda}_x=O\ab(\frac{1}{d^{2K+1}})$  \\
\hline
\rowcolor[rgb]{0.871,0.867,0.855} \begin{tabular}[c]{@{}>{\cellcolor[rgb]{0.871,0.867,0.855}}l@{}}\UniCLUB, \UniSCLUB\\(Ours, \Cref{thm:regret-uniclub,thm:regret-unisclub})\end{tabular}                          & i.i.d.                                                                        & \begin{tabular}[c]{@{}>{\cellcolor[rgb]{0.871,0.867,0.855}}l@{}}$\lambda_\text{min}(\E[\vec{X}\vec{X}^\mathsf{T}])=\lambda_x$,\\The parameter $\gamma$ (the minimum\\gap between clusters) is known.\end{tabular}      & $\widetilde{O}\ab(\frac{ud}{\gamma^2\lambda_x})$                                                      & $\lambda_x=O\ab(\frac{1}{d})$                     \\
\hline
\rowcolor[rgb]{0.871,0.867,0.855} \PhaseUniCLUB (Ours, \Cref{thm:regret-gamma-unknown})                                                                                                                            & i.i.d.                                                                        & $\lambda_\text{min}(\E[\vec{X}\vec{X}^\mathsf{T}])=\lambda_x$                                                                                                                                                          & $\widetilde{O}\ab(\frac{ud}{\gamma^5\lambda_x^2})$                                                    & $\lambda_x=O\ab(\frac{1}{d})$                     \\
\hline
\rowcolor[rgb]{0.871,0.867,0.855} \SACLUB, \SASCLUB (Ours, \Cref{thm:regret-smoothed-adversary})                                                                                                                   & Adversarial                                                                   & Gaussian noise perturbation                                                                                                                                                                                            & $\widetilde{O}\ab(\frac{ud}{\gamma^2\widetilde{\lambda}_x})$                                          & $\widetilde{\lambda}_x = O\ab(\frac{1}{\log K})$  \\
\bottomrule
\end{tabular}
\begin{tablenotes}\footnotesize
\item \(K\), \(u\), and \(d\) denote the number of arms, the number of users, and the dimension, respectively. \(\vec{z} \in \RR^d\) represents an arbitrary unit vector.
\item \(\lambda_\text{min}(\cdot)\) denotes the minimum eigenvalue. \(\gamma\) is defined in Assumption~\ref{assumption:well-separatedness}. \(\vec{X}\) and \(\lambda_x\) are defined in Assumption~\ref{assumption:item-regularity}.
\end{tablenotes}
\end{threeparttable}
}
\end{table}

\vspace{-3pt}

The rest of the paper is organized as follows.
In \Cref{sec:related-work}, we review related work, emphasizing the restrictive assumptions in the original setting and discussing prior attempts to eliminate these assumptions.
In \Cref{sec:stochastic-context-setting}, we present our solution for removing these assumptions in the original setting.
In \Cref{sec:smoothed-adversarial-context-setting}, we introduce an alternative approach based on the smoothed analysis framework, which aligns more closely with the original linear bandits setting~\citep{abbasi-2011-improved}.
Finally, in \Cref{sec:theoretical-analysis} and \Cref{sec:performance-evaluation}, we present the theoretical analysis and experiment results.

\section{Related Work}
\label{sec:related-work}
Our work is closely related to the literature of online clustering of bandits.
Since the seminal work by \citet{gentile-2014-online-clustering}, which first formulated the clustering of bandits problem and proposed a graph-based algorithm, there has been a line of follow-up studies.
For example, \citet{li-2016-collaborative} considers the collaborative effects that arise from user-item interactions.
\citet{gentile-2017-on-context-dependent} implements the underlying feedback-sharing mechanism by estimating user neighborhoods in a context-dependent manner.
\citet{li-2018-online-clustering} consider the clustering of bandits problem in the cascading bandits setting with random prefix feedback.
\citet{li-2019-improved-algorithm} propose a set-based algorithm and consider users with non-uniform arrival frequencies.
\citet{ban-2021-local-clustering} introduce local clustering which does not assume that users within the same cluster share exactly the same parameter.
\citet{liu-2022-federated} extend the clustering of bandits problem to the federated setting and consider privacy preservation.
\citet{wang-2023-online-clustering,wang-2023-online-corrupted} investigate clustering of bandits under misspecified and corrupted user models.

However, all of these studies adhere to \citet{gentile-2014-online-clustering}'s original setting and theoretical analysis framework, which imposes several restrictive assumptions on the generation process of contexts, such as (a) the feature vector of each arm is independently sampled from a fixed distribution; (b) the covariance matrix constructed on specific context-action features is full rank with the minimum eigenvalue greater than 0; and (c) the square of contexts projected in a fixed direction is sub-Gaussian with bounded conditional variance.
Constructing a natural context generation distribution that satisfies all these assumptions simultaneously is \emph{highly challenging} (if not impossible), and none of the aforementioned papers provide any concrete examples.
There have been some attempts to relax these assumptions.
For example, \citet{wang-2023-online-clustering,wang-2023-online-corrupted,yang-2024-federated} propose more relaxed assumptions regarding the variance of contexts.
However, these approaches result in a regret that grows exponentially with the number of arms \(K\) (as shown in \Cref{tab:comparison}).
We refer interested readers to more related works about leveraging similar assumptions in \Cref{sec:more-related-works}.

Our smoothed adversarial setting and algorithms \SACLUB, \SASCLUB are inspired by the smoothed analysis framework, introduced by \citet{spielman-2004-smoothed}.
This framework studies algorithms where some instances are chosen by an adversary, but are then perturbed randomly, representing an interpolation between worst-case and average-case analyses.
It was originally proposed to analyze the running time of algorithms.
\citet{kannan-2018-a-smoothed} first introduce the ``smoothed adversary'' setting in multi-armed bandits and study how the regret bound of greedy algorithms behave on smoothed bandit instances.
This setting has been extended to structured linear bandits (i.e., the unknown preference vector has structures such as sparsity, group sparsity, or low rank)~\citep{sivakumar-2020-structured-linear}, linear bandits with knapsacks~\citep{sivakumar-2022-smoothed}, and Bayesian regret~\citep{raghavan-2023-greedy}.
All these studies show that the greedy algorithm almost matches the best possible regret bound.
The core idea is that the inherent diversity in perturbed data (contexts) makes explicit exploration unnecessary.
In contrast to the analysis of greedy algorithms, our work introduces the smoothed analysis framework into the clustering of bandits setting, where the UCB strategy lacks sufficient exploration for identifying unknown user clusters.
We demonstrate that the inherent diversity of contexts in the smoothed adversary setting eliminates the impractical requirement for \emph{i.i.d.} context generation and enhances the cumulative regrets of existing algorithms (e.g., CLUB~\citep{gentile-2014-online-clustering}).

\section{Stochastic Context Setting}
\label{sec:stochastic-context-setting}
In this section, we study the online clustering of bandits problem under the stochastic context setting but with substantially relaxed assumptions.
We begin by introducing the problem setting in \Cref{sec:stochastic-context-problem-setting}, which largely follows the seminal work of \citet{gentile-2014-online-clustering} but without the stringent assumptions.
Next, in \Cref{sec:key-techniques}, we provide the intuition of the key techniques underlying our approach.
Finally, we present our proposed algorithms: \UniCLUB (\Cref{sec:gamma-known}), which assumes a \emph{known} minimum gap between clusters, and \PhaseUniCLUB (\Cref{sec:gamma-unknown}), which does not require this assumption.
We also provide a set-based algorithm \UniSCLUB in \Cref{sec:unisclucb}.

\subsection{Problem Setting}
\label{sec:stochastic-context-problem-setting}

In the following, we use boldface letters for vectors and matrices.
We denote \([M] := \set{1, \dots, M}\) for \(M \in \NN^{+}\).
For any real vector \(\vec{x}, \vec{y}\) and positive semi-definite (PSD) matrix \(\vec{V}\), \(\|\vec{x}\|\) denotes the \(\ell_2\) norm of \(\vec{x}\), \(\inprod{\vec{x}}{\vec{y}}=\vec{x}^\mathsf{T} \vec{y}\) denotes the dot product of vectors, \(\inprod{\vec{x}}{\vec{y}}_{\vec{V}}=\vec{x}^\mathsf{T} \vec{V} \vec{y}\) denotes the weighted inner product, and \(\|\vec{x}\|_{\vec{V}}\) denotes the Mahalanobis norm \(\sqrt{\vec{x}^\mathsf{T} \vec{V} \vec{x}}\).
We use \(\lambda_{\min}(\cdot)\) and \(\lambda_{\max}(\cdot)\) to denote the minimum and maximum eigenvalue.

In the online clustering of bandit problem, there are \(u\) users, denoted by set \([u]=\set{1,2, \dots, u}\).
Each user \(i \in [u]\) is associated with an \emph{unknown} preference feature vector \(\vec{\theta}_i \in \RR^d\) with \(\|\vec{\theta}_i\|_2\leq 1\).
There is an underlying cluster structure among all the users.
Specifically, the users are separated into \(m\) clusters \(\mathcal{I}_1, \mathcal{I}_2, \dots, \mathcal{I}_m\) (\(m \ll u\)), where \(\bigcup_{i \in [m]} \mathcal{I}_i = [u]\) and \(\mathcal{I}_i \cap \mathcal{I}_j = \emptyset\) for \(i \neq j\), such that users lying in the same cluster share the same preference feature vector (i.e., they have similar behavior) and users lying in different clusters have different preference feature vector (i.e., they have different behavior).
Formally, let \(\vec{\theta}^k\)denote the common preference vector for cluster \(\mathcal{I}_k\) and \(j(i) \in [m]\) denote the index of the cluster that user \(i\) belongs to.
In other words, for any user \(i \in \mathcal{I}_k\), we have \(\vec{\theta}_i = \vec{\theta}^k = \vec{\theta}^{j(i)}\).
The underlying partition of users and the number of clusters \(m\) are \emph{unknown} to the learner, and need to be learned during the algorithm.

The learning procedure operates as follows: 
At each round \(t=1,2, \dots, T\), the learner receives a user index
\(i_t \in [u]\) and a finite set of arms \(\mathcal{A}_t \subseteq \mathcal{A}\) where \(|\mathcal{A}_t| = K\).
Each arm \(a \in \mathcal{A}\) is associated with a feature vector \(\vec{x}_a \in \RR^d\), and we denote \(\mathcal{D}_{t} = \set{\vec{x}_a}_{a \in \mathcal{A}_t} \subseteq \RR^d\).
\(\mathcal{D}_t\) is also called \emph{context}.
When we need to emphasize the index of arms, we also denote \(\mathcal{A}_t = \set{a_{t,i}}_{i=1}^K\) and \(\mathcal{D}_t=\set{\vec{x}_{t,i}}_{i=1}^K\).
Then the learner assigns an appropriate cluster \(V_t\) for user \(i_t\) and recommends an arm \(a_t \in \mathcal{A}_t\) based on the aggregated data gathered from cluster \(V_t\).
After receiving the recommended arm, user \(i_t\) sends a random reward \(r_t \in [-1,1]\) back to the learner.
The reward is assumed to have a linear structure: \(r_t = \vec{x}_{a_t}^\mathsf{T} \vec{\theta}_{i_t} + \eta_t\), where \(\eta_t\) is a zero-mean, 1-sub-Gaussian noise term.

Let \(a_t^{*} = \argmax_{a \in \mathcal{A}_t} \vec{x}_a^\mathsf{T}\vec{\theta}_{i_t}\) be the optimal arm with the highest expected reward at time step \(t\).
The goal of the learner is to minimize the expected cumulative regret defined as follows:
\[\E[R(T)] = \E\ab[\sum_{t=1}^{T} \ab(\vec{x}_{a_t^{*}}^\mathsf{T} \vec{\theta}_{i_t} - \vec{x}_{a_t}^\mathsf{T} \vec{\theta}_{i_t})],\]
where the expectation is taken over both the arms (\(a_t\)) and users (\(i_t\)) chosen during the process.

We assume the users and clusters satisfy the assumptions as follows:

\begin{assumption}[User uniformness]\label{assumption:user-uniformness}
  At each time step \(t\), the user \(i_t\) is drawn uniformly from the set of all users \([u]\), independently over the past.
\end{assumption}

\begin{assumption}[Well-separatedness among clusters]\label{assumption:well-separatedness}
  All users in the same cluster \(\mathcal{I}_j\) share the same preference vector \(\vec{\theta}^{j}\).
  For users in different clusters, there is a fixed but \emph{unknown} gap \(\gamma\) between their preference vectors.
  Specifically, for any cluster indices \(i \neq j\),
  \[\ab\|\vec{\theta}^{i} - \vec{\theta}^{j}\|_2 \geq \gamma >0.\]
\end{assumption}
\begin{remark}\label{remark:gamma-unknown}
  The parameter \(\gamma\) being unknown is a critical condition.
  In \Cref{sec:gamma-known}, we begin by considering the scenario where \(\gamma\) is known.
  In \Cref{sec:gamma-unknown}, we remove this additional restriction.
\end{remark}

Following \citet{gentile-2014-online-clustering}, we assume the feature vectors (i.e., contexts) are independently sampled from a fixed distribution, but we \emph{completely} remove the restricted assumptions on the sub-Gaussian distribution and variance of the arm generation process as mentioned in prior works.

\begin{assumption}[Context diversity for stochastic contexts]\label{assumption:item-regularity}
At each time step \(t\), the feature vectors in \(\mathcal{D}_t\) are drawn independently from a fixed distribution \(\vec{X}\) with \(\|\vec{X}\|\leq L\), and \(\E[\vec{X}\vec{X}^\mathsf{T}]\) is of full rank with minimum eigenvalue \(\lambda_x>0\).
\end{assumption}
\begin{remark}
  Intuitively, the minimum eigenvalue indicates how ``diverse'' the distribution \(\vec{X}\) is, depicting how certain the feature vectors span the full \(\RR^{d}\) space.
  Having a lower bound on the minimum eigenvalue means that \(\vec{X}\) has non-zero variance in all directions, which is necessary for the least squares estimator to converge to the true parameter.
  Note that Assumption~\ref{assumption:item-regularity} only maintains the minimum eigenvalue assumption\footnote{This assumption is inevitable, since if the minimum eigenvalue is zero, the covariance matrix is not of full rank, and thus \(\vec{\theta}\) cannot be uniquely determined.} and \emph{completely} removes the additional stringent assumptions as in previous studies~\citep{gentile-2014-online-clustering,li-2018-online-clustering,li-2019-improved-algorithm,wang-2023-online-clustering,wang-2023-online-corrupted,liu-2022-federated,yang-2024-federated}.
\end{remark}

\subsection{Diversity Conditions in Previous Studies and Key Techniques}
\label{sec:key-techniques}
 Before delving into detailed algorithms, we first examine the stringent statistical assumptions in previous studies~\citep{gentile-2014-online-clustering,li-2018-online-clustering,li-2019-improved-algorithm,wang-2023-online-clustering,wang-2023-online-corrupted,liu-2022-federated,yang-2024-federated} and explain their necessity for the theoretical analysis of the existing UCB-based algorithms.
Then we provide insights on how these assumptions can be eliminated.

The key requirement of online clustering of bandits is the precise estimation of the preference vectors \(\vec{\theta}_i\) for each user \(i\), which is essential for correctly identifying the unknown user clusters.
However, the convergence of the least squares estimator relies on sufficiently diverse data.
Intuitively, when data points span a broad range of values and cover the spectrum of possible predictors, the model can better capture the true underlying relationships, leading to more reliable parameter estimates.
Mathematically, diverse data help ensure that user \(i\)'s design matrix \(\vec{S}_{i,t}=\sum_{s\in [t]: i_s = i} \vec{x}_{a_s}\vec{x}_{a_s}^\mathsf{T}\) is well-conditioned, resulting in a more stable matrix inverse, which in turn reduces the variance of the estimated preference vector \(\widehat{\vec{\theta}}_i = \vec{S}_{i,t}^{-1}\sum_{s\in [t]: i_s = i} r_s\vec{x}_{a_s}\).
Therefore, all the previous studies rely on the diverse stochastic context assumption (Assumption~\ref{assumption:item-regularity}), which states that the feature vector of each arm is drawn independently from a fixed distribution \(\vec{X}\) with \(\lambda_{\min}(\E[\vec{X}\vec{X}^\mathsf{T}]) = \lambda_x >0\).
However, diverse contexts do not necessarily lead to a well-conditioned design matrix because \emph{Assumption~\ref{assumption:item-regularity} only guarantees the diversity of arm set \(\mathcal{D}_t\), but not that the arms selected by the UCB strategy are diverse}.
As a result, previous studies impose an additional assumption, requiring that for any fixed unit vector \(\vec{z} \in \RR^{d}\), random variable \((\vec{z}^\mathsf{T} \vec{X})^2\) has sub-Gaussian tails with variance parameter \(\sigma^2 \leq \frac{\lambda_x^2}{8\log(4K)}\).
This assumption, however, contradicts the diverse stochastic context assumption (Assumption~\ref{assumption:item-regularity}), as the bounded variance condition restricts the diversity of \(\vec{X}\).
In fact, it is extremely difficult to construct a natural example of \(\vec{X}\) such that all these assumptions are satisfied simultaneously, and the aforementioned papers also do not provide any.

In summary, the key insight for eliminating the additional variance assumption is to ensure a well-conditioned design matrix \(\vec{S}_{i,t}\), i.e., the selected arms are sufficiently diverse.
To this end,  we introduce an additional pure exploration phase which uniformly selects arms in the arm set \(\mathcal{D}_t\).
In \Cref{lemma:bound-smallest-eigenvalue}, we will show that this explicit exploration guarantees that the minimum eigenvalue of the design matrix \(\vec{S}_{i,t}\) grows linearly with the number of times user \(i\) appears.
At the same time, it is important to limit the amount of explicit exploration so that the cumulative regret remains the order of \(\widetilde{O}(\sqrt{T})\).
To address this balance, \UniCLUB leverages the knowledge of the gap parameter \(\gamma\) to determine a sufficient duration of pure exploration required for accurate cluster estimation while minimizing the cumulative regret.
In contrast, \PhaseUniCLUB employs a phase-based approach to handle the scenario where \(\gamma\) is unknown.

\subsection{\UniCLUB: Algorithm for the Case When \texorpdfstring{\(\gamma\)}{γ} Is Known}
\label{sec:gamma-known}
In this subsection, we assume the parameter \(\gamma\) defined in Assumption~\ref{assumption:well-separatedness} is known and introduce two algorithms: a graph-based algorithm called Uniform Exploration Clustering of Bandits (\UniCLUB, \Cref{algo:clucb}) and a set-based algorithm called Uniform Exploration Set-based Clustering of Bandits (\UniSCLUB, \Cref{algo:sclucb}).
Both algorithms explicitly take \(\gamma\) as an input.
Due to space constraints, we focus on \UniCLUB in the main text, leaving the details of \UniSCLUB in \Cref{sec:unisclucb}.

As shown in \Cref{algo:clucb}, inspired by the CLUB algorithm proposed in \citet{gentile-2014-online-clustering}, \UniCLUB maintains a dynamic undirected graph \(G_t = ([u], E_t)\) representing the current estimated cluster structures of all users.
The main difference is that \UniCLUB includes an additional uniform exploration phase to promote cluster identification.
At the beginning, \(G_t\) is initialized as a complete graph, indicating that all users are considered in a single cluster.
Then at each round \(t\), a user \(i_t \in [u]\) comes to be served with a feasible arm set \(\mathcal{A}_t\) from which the learner has to choose.
The algorithm operates in the following two phases depending on whether the current time step \(t \leq T_0\), and the arm selection strategy differs between these phases.

\IncMargin{1em}
\begin{algorithm}[htb]
  \DontPrintSemicolon
  \SetKwComment{Comment}{$\triangleright$\ }{}
  \SetKwInput{KwInit}{Initialization}
  \KwIn{\(\lambda\), \(\beta\), \(\lambda_x\), \(\delta\), \(L\), \(\gamma\)}
  \KwInit{Let \(G_0=([u], E_0)\) be a complete graph.\\
    Let \(\vec{S}_{i,0} = \vec{0}_{d\times d}, \vec{b}_{i,0} = \vec{0}_{d \times 1}, T_{i,0}=0, \forall i \in [u]\).\\
    Define \(f(T_{i,t}) \triangleq (\sqrt{2\log\ab(u/\delta) + d\log(1+\frac{T_{i,t} L^2}{\lambda d})} + \sqrt{\lambda})/\sqrt{\lambda + T_{i,t}\lambda_x/2}\).\\
    Define \(T_0 \triangleq 16u\log\ab(\frac{u}{\delta}) + 4u\max\ab\{\frac{8L^2}{\lambda_{x}}\log\ab(\frac{ud}{\delta}), \frac{512d}{\gamma^2 \lambda_{x}}\log\ab(\frac{u}{\delta})\}\).
  }

  \SetKwProg{Fn}{Function}{:}{}

  \For{\(t=1, 2, \dots\)}{
    Receive user index \(i_t\) and arm set \(\mathcal{A}_t\)\;
    \uIf{\(t \leq T_0\)}{
      Select \(a_t\) uniformly at random from \(\mathcal{A}_t\)\;\label{line:uniform-select-arm-uniclub}
    }
    \Else{
      Find the connected component \(V_t\) for \(i_t\) in \(G_{t-1}\)\;\label{line:find-scc-uniclub}
      \(\vec{M}_{V_t, t-1} = \sum_{i \in V_t} \vec{S}_{i,t-1}, \ \overline{\vec{M}}_{V_t, t-1} = \lambda \vec{I} + \vec{M}_{V_t, t-1}, \  \vec{b}_{V_t,t-1}=\sum_{i \in V_{t}} \vec{b}_{i,t-1}\)\;
      \(\widehat{\vec{\theta}}_{V_t,t-1} = \overline{\vec{M}}_{V_t, t-1}^{-1}\vec{b}_{V_t,t-1}\)\;\label{line:update-statistics-cluster-uniclub}
      Select arm \(a_t = \argmax_{a \in \mathcal{A}_t} \widehat{\vec{\theta}}_{V_t,t-1}^\mathsf{T} \vec{x}_{a} + \beta \sqrt{\vec{x}_a^\mathsf{T} \overline{\vec{M}}_{V_t,t-1}^{-1}\vec{x}_{a}}\)\;\label{line:select-arm-uniclub}
    }
    Receive reward \(r_t\)\;
    Update statistics for user \(i_t\), others remain unchanged:
    \(\vec{S}_{i_t,t} =\vec{S}_{i_t,t-1} + \vec{x}_{a_t}\vec{x}_{a_t}^\mathsf{T}, \quad \vec{b}_{i_t,t}=\vec{b}_{i_t,t-1} + r_t\vec{x}_{a_t}\)
    \(T_{i_t,t}=T_{i_t,t-1}+1, \quad \widehat{\vec{\theta}}_{i_t,t} = \ab(\lambda \vec{I}+\vec{S}_{i_t, t})^{-1}\vec{b}_{i_t,t}\)\;\label{line:update-statistics-user-uniclub}
    Delete edge \((i_t, \ell) \in E_{t-1}\) if \label{line:delete-edge-uniclub}
    \[\ab\|\widehat{\vec{\theta}}_{i_t,t}-\widehat{\vec{\theta}}_{\ell,t}\| > f(T_{i_t,t}) + f(T_{\ell,t})\]
    and obtain an updated graph \(G_t=([u],E_t)\)\;
  }
  \caption{\UniCLUB: Uniform Exploration Clustering of Bandits} \label{algo:clucb}
\end{algorithm}

  \textbf{Pure exploration phase}. In the first \(T_0\) rounds, the algorithm uniformly select arm \(a_t\) from \(\mathcal{A}_t\) (\Cref{line:uniform-select-arm-uniclub}).
  This arm selection strategy ensures selecting sufficiently \emph{diverse} arms so that the minimum eigenvalue of the design matrix grows linearly in time (\Cref{lemma:bound-smallest-eigenvalue}).
  In \Cref{lemma:clusters-correct-after-T0}, we will show that the phase length \(T_0\) is chosen to guarantee that this phase gathers sufficient statistics to estimate each user's preference vector and correctly infer the underlying user clusters with high probability.

  \textbf{Exploration-exploitation phase}. After \(T_0\), the algorithm constructs the connected component \(V_t\) containing user \(i_t\) in the graph \(G_{t-1}\) (\Cref{line:find-scc-uniclub}), and computes the estimated preference vector \(\widehat{\vec{\theta}}_{V_t,t-1}\) based on historical information associated with \(V_t\) using the least squares estimator with regularization parameter \(\lambda > 0\) (\Cref{line:update-statistics-cluster-uniclub}).
  The algorithm then recommends an arm using the upper confidence bound (UCB) strategy~\citep{abbasi-2011-improved} to balance exploration and exploitation (\Cref{line:select-arm-uniclub}):
  \[a_t = \argmax_{a \in \mathcal{A}_t} \widehat{\vec{\theta}}_{V_t,t-1}^\mathsf{T} \vec{x}_{a} + \beta \sqrt{\vec{x}_a^\mathsf{T} \overline{\vec{M}}_{V_t,t-1}^{-1}\vec{x}_{a}},\]
  where the first term is the estimated reward of arm \(a\) at time \(t\) and the second term is the confidence radius of arm \(a\) at time \(t\) with parameter \(\beta = \sqrt{d \log(1+\frac{TL^2}{d\lambda}) + 2\log(\frac{1}{\delta})} + \sqrt{\lambda}\).

After receiving the feedback \(r_t\) from user \(i_t\), the learner updates the statistics for user \(i_t\) while keeping other users' statistics unchanged.
Note that the estimated preference vector \(\widehat{\vec{\theta}}_{i_t}\) is computed using historical information associated with user \(i_t\) (\Cref{line:update-statistics-user-uniclub}).
Finally, the algorithm updates the inferred clusters by deleting edges in graph \(G_{t-1}\) if it determines that two users belong to different clusters.
Specifically, for every user \(\ell \in [u]\) that has an edge to user \(i_t\), the algorithm checks if the distance between the estimated preference vectors of users \(\ell\) and \(i_t\) exceeds a specific threshold (\Cref{line:delete-edge-uniclub}). If so, the algorithm deletes the edge \(i_t,\ell\) to split them apart and update the graph.

\subsection{\PhaseUniCLUB: Algorithm for the Case When \texorpdfstring{\(\gamma\)}{γ} Is Unknown}
\label{sec:gamma-unknown}
In this subsection, we present a phase-based algorithm \PhaseUniCLUB (\Cref{algo:clucb-unknown-gamma}) to handle the scenario where the parameter \(\gamma\) is unknown.
For convenience, we first define the following notations used in \PhaseUniCLUB:
\begin{align*}
    T^{\text{init}} \triangleq 16u\log\ab(\frac{u}{\delta}) + 4u\cdot\frac{8L^2}{\lambda_{x}}\log\ab(\frac{ud}{\delta}), \quad T^{(s)} \triangleq 4u \cdot \frac{512d}{ 2^{-s} \lambda_{x}}\log\ab(\frac{u}{\delta}).
\end{align*}

\begin{algorithm}[htb]
  \DontPrintSemicolon
  \SetKwComment{Comment}{$\triangleright$\ }{}
  \SetKwInput{KwInit}{Initialization}
  \KwIn{\(\lambda\), \(\beta\), \(\lambda_x\), \(\delta\), \(L\), \(\alpha\)}
  \KwInit{Let \(G_0=([u], E_0)\) be a complete graph.\\
    Let \(\vec{S}_{i,0} = \vec{0}_{d\times d}, \vec{b}_{i,0} = \vec{0}_{d \times 1}, T_{i,0}=0, \forall i \in [u]\).\\
    Define \(f(T_{i,t})=(\sqrt{2\log\ab(u/\delta) + d\log(1+\frac{T_{i,t} L^2}{\lambda d})} + \sqrt{\lambda})/\sqrt{\lambda + T_{i,t}\lambda_x/2}\). \\
  }
  \For{\(t=1,2,\dots, T^{\text{init}}\)}{\label{line:clucb-unknown-T0-first}
      Receive user index \(i_t\) and arm set \(\mathcal{A}_t\)\;
      Select \(a_t\) uniformly at random from \(\mathcal{A}_t\)\; \label{line:clucb-unknown-T0}
      Receive reward \(r_t\)\;
      Update statistics for user \(i_t\): \(\vec{S}_{i_t,t}, \vec{b}_{i_t,t}, T_{i_t,t}, \widehat{\vec{\theta}}_{i_t,t}\)\; \label{line:clucb-unknown-T0-last}
  }
  \For{\(s= 0,1,\dots \)}{
    \For{\(\tau = 1,2, \dots, 2^{\alpha s} \cdot T^{(s)}\) (terminate when \(t>T\))}{
    \(t = t+1\)\; 
    Receive user index \(i_t\) and arm set \(\mathcal{A}_t\)\;
    \uIf{\( \tau \leq T^{(s)} \)}{
        Select \(a_t\) uniformly at random from \(\mathcal{A}_t\)\; \label{line:clucb-unknown-uniform}
    }
    \Else{
        Find all neighbors of user \(i_t\) in \(G_{t-1}\) and include \(i_t\) to form the cluster \(V_t\) \; \label{line:clucb-unknown-cluster-ucb-1}
        \(\vec{M}_{V_t, t-1} = \sum_{i \in V_t} \vec{S}_{i,t-1}, \ \overline{\vec{M}}_{V_t, t-1} = \lambda \vec{I} + \vec{M}_{V_t, t-1}\)\;  \label{line:clucb-unknown-cluster-ucb-2}
        \(\vec{b}_{V_t,t-1}=\sum_{i \in V_{t}} \vec{b}_{i,t-1}, \ 
        \widehat{\vec{\theta}}_{V_t,t-1} = \overline{\vec{M}}_{V_t, t-1}^{-1}\vec{b}_{V_t,t-1}\), \
        \(T_{V_t,t-1} = \sum_{i \in V_{t}} T_{i,t}\)\; \label{line:clucb-unknown-cluster-ucb-3}
        Select arm \(a_t = \argmax_{a \in \mathcal{A}_t} \widehat{\vec{\theta}}_{V_t,t-1}^\mathsf{T} \vec{x}_{a} + \beta \sqrt{\vec{x}_a^\mathsf{T} \overline{\vec{M}}_{V_t,t-1}^{-1}\vec{x}_{a}}\)\;  \label{line:clucb-unknown-cluster-ucb-4}
    }
    Receive reward \(r_t\)\; \label{line:clucb-unknown-reward}
    Update statistics for user \(i_t\): \(\vec{S}_{i_t,t}, \vec{b}_{i_t,t}, T_{i_t,t}, \widehat{\vec{\theta}}_{i_t,t}\)\; \label{line:clucb-unknown-update_user}
    Delete edge \((i_t, \ell) \in E_{t-1}\) if
    \[\ab\|\widehat{\vec{\theta}}_{i_t,t}-\widehat{\vec{\theta}}_{\ell,t}\| > f(T_{i_t,t}) + f(T_{\ell,t})\]
    and obtain the updated graph \(G_t=([u],E_t)\)\; \label{line:clucb-unknown-delete-edge}
    }
    }
  \caption{\PhaseUniCLUB: Phase-based Uniform Exploration Clustering of Bandits} \label{algo:clucb-unknown-gamma}
\end{algorithm}

Similar to \UniCLUB (\Cref{algo:clucb}), \PhaseUniCLUB also maintains a dynamic undirected graph \(G_t = ([u], E_t)\) over all users for clustering purposes.
However, to cope with the unknown \(\gamma\), \PhaseUniCLUB leverages the idea of the \emph{doubling trick}.
Specifically, as depicted in~\Cref{algo:clucb-unknown-gamma}, the algorithm begins with an initial phase of \(T^{\text{init}}\) rounds of uniform exploration (\Crefrange{line:clucb-unknown-T0-first}{line:clucb-unknown-T0-last}).
This ensures that each user appears at least once by the end of the \(T^{\text{init}}\) rounds.
Then, the algorithm proceeds in phases.
Each phase \(s=0, 1,\dots\) is divided into two subphases:
(1) the exploration subphase, consisting of \( T^{(s)}\) rounds, during which the algorithm selects arms uniformly from the arm set \(\mathcal{A}_t\) for the incoming user \(i_t\) (\Cref{line:clucb-unknown-uniform});
(2) the UCB subphase, containing \((2^{\alpha s} -1) T^{(s)}\) rounds, during which the algorithm first identifies the cluster for the incoming user \(i_t\) (\Cref{line:clucb-unknown-cluster-ucb-1}) and then selects arms based on the estimated preference vector of the identified cluster (\Cref{line:clucb-unknown-cluster-ucb-4}).
Throughout the process, after receiving the reward feedback (\Cref{line:clucb-unknown-reward}), the algorithm updates statistics for user \(i_t\) (\Cref{line:clucb-unknown-update_user}) and determines whether to delete any edge between \(i_t\) and its neighbors (\Cref{line:clucb-unknown-delete-edge}).

Unlike \UniCLUB, which has a cut-off time \(T_0\) to terminate uniform exploration, \PhaseUniCLUB lacks access to the cluster gap \(\gamma\) and therefore distributes the uniform exploration across all phases.
During the exploration subphase of each phase \(s\), \PhaseUniCLUB focuses on estimating the users’ preference vectors to a precision level of \(\gamma_s\).
In the subsequent UCB subphase, users are clustered based on the current precision level.
If \(\gamma_{s} > \gamma\), there is a risk of users being incorrectly assigned to clusters they do not belong to.
To mitigate this risk, \PhaseUniCLUB identifies only the immediate neighbors of a user to form a cluster (\Cref{line:clucb-unknown-cluster-ucb-1}), rather than finding the entire connected component as in \UniCLUB.
This localized clustering approach helps to minimize the regret incurred by misclustering.
As shown in \Cref{sec:theoretical-analysis}, with carefully chosen phase lengths, \PhaseUniCLUB achieves a regret bound comparable to \UniCLUB.
The approach of identifying only direct neighbors has also been employed in \citet{wang-2023-online-clustering}, but with a different objective of mitigating misclustering caused by model misspecification.
Moreover, the regret bound in \citet{wang-2023-online-clustering} scales linearly with \(T\), whereas \PhaseUniCLUB grows only \(\widetilde{O}(\sqrt{T})\).

\section{Smoothed Adversarial Context Setting}
\label{sec:smoothed-adversarial-context-setting}
Although \UniCLUB and \PhaseUniCLUB offer significant theoretical improvements, the stochastic context setting necessitates an \emph{i.i.d.} context generation process (Assumption~\ref{assumption:item-regularity}), which might be impractical in real-world applications.
To overcome these limitations, based on the intuition in \Cref{sec:key-techniques}, we propose the smoothed adversarial context setting to eliminate the need for explicit pure exploration.
This setting interpolates between the two extremes: the \emph{i.i.d.} context generation in \citet{gentile-2014-online-clustering} and the adversarial context generation in \citet{abbasi-2011-improved}.
The intrinsic diversity of contexts makes explicit exploration unnecessary, thereby ensuring a well-conditioned design matrix (\Cref{lemma:bound-smallest-eigenvalue-smoothed-adversary}).
This approach allows existing algorithms in previous studies, such as CLUB~\citep{gentile-2014-online-clustering} and SCLUB~\citep{li-2019-improved-algorithm}, which consistently employ the UCB strategy, to perform more effectively.

\subsection{Problem Setting}
\label{sec:problem-formulation-smoothed-adversary}
In the smoothed adversarial setting, we retain Assumptions~\ref{assumption:user-uniformness} and \ref{assumption:well-separatedness}, while replacing Assumption~\ref{assumption:item-regularity} with Assumptions~\ref{assumption:item-regularity2}.
As detailed below, Assumption~\ref{assumption:item-regularity2} allows feature vectors (i.e., contexts) to be arbitrarily chosen by an adversary, but with some random perturbation to ensure the resulting contexts remain sufficiently diverse.
This approach maintains enough data diversity to support effective learning while avoiding the need for explicit pure exploration.

\begin{assumption}[Context diversity for adversarial contexts]\label{assumption:item-regularity2}
  At each time step \(t\), the feature vector \(\vec{x}_{a} \in \mathcal{D}_t\) for each arm \(a \in \mathcal{A}_t\) is drawn by a ``smoothed'' adversary, meaning that the adversary first chooses an arbitrary vector \(\vec{\mu}_{a} \in \RR^d\) with \(\|\vec{\mu}_{a}\|\leq1\), then samples a noise vector \(\vec{\varepsilon}_{a} \in \RR^d\) from a truncated multivariate Gaussian distribution where each dimension is truncated within \([-R, R]\), i.e., \(\vec{\varepsilon}_{a} \sim \mathcal{N}(0, \sigma^2 \vec{I})\) conditioned on \(|(\vec{\varepsilon}_{a})_j| \leq R, \forall j \in [d]\).
  And the feature vector \(\vec{x}_{a} = \vec{\mu}_{a}+\vec{\varepsilon}_{a}\).
\end{assumption}

\begin{remark}
  The truncation in Assumption~\ref{assumption:item-regularity2} is used to guarantee that the length of each feature vector is bounded, which is a standard requirement in the linear bandits literature.
  In fact, if we are only concerned with high-probability regret, we can also use a Gaussian distribution (without truncation) and show that the length of each feature vector is bounded with high probability.
  Note that Assumption~\ref{assumption:item-regularity2} is more similar to the original linear bandit setting~\citep{abbasi-2011-improved}, except that we require each arm to be perturbed by Gaussian noise.
  It remains an open problem whether a fully adversarial setting can be achieved.
\end{remark}

\subsection{Algorithms for Smoothed Adversarial Context Setting}
Our proposed algorithms for the smoothed adversarial context setting, \SACLUB and \SASCLUB, are essentially CLUB~\citep{gentile-2014-online-clustering} and SCLUB~\citep{li-2019-improved-algorithm} with \(\lambda_x\) replaced by \(\widetilde{\lambda}_x\) and \(L\) replaced by \(1+\sqrt{d}R\) in the edge deletion threshold.
Due to space constraints, we omit the full details of \SACLUB and \SASCLUB here, and the complete proofs are provided in \Cref{sec:proof-of-advclub}.

\section{Theoretical Analysis}
\label{sec:theoretical-analysis}
In this section, we present the theoretical results of our algorithms, with detailed proofs provided in Appendices \ref{sec:proof-of-uniclub}, \ref{sec:unisclucb}, \ref{sec:proof-of-phaseuniclub}, and \ref{sec:proof-of-advclub}.
For clarity, we ignore the constants but they are fleshed out in the proofs.
Note that \(\lambda_x\) appears in the denominator of the regret expressions.
This is due to the assumption of bounded contexts (\(\|\vec{X}\|_2\) is bounded) in the stochastic context setting, resulting in \(\lambda_x=O(1/d)\), and therefore it is important to track the dependency of our regret bounds on \(\lambda_x\).

\begin{restatable}[Regret of \UniCLUB]{theorem}{restateregret}\label{thm:regret-uniclub}
  Under the stochastic context setting (Assumptions~\ref{assumption:user-uniformness}, \ref{assumption:well-separatedness}, \ref{assumption:item-regularity}) and assuming the cluster gap \(\gamma\) is known, the expected regret of the \UniCLUB (\Cref{algo:clucb}) satisfies:
  \begin{align*}
    \E[R(T)] = O\ab(\frac{ud}{\gamma^2\lambda_x}\log(T) + d\sqrt{mT}\log(T)).
  \end{align*}
\end{restatable}

\begin{restatable}[Regret of \UniSCLUB]{theorem}{restateregretunisclub}\label{thm:regret-unisclub}
  Under the stochastic context setting (Assumptions~\ref{assumption:user-uniformness}, \ref{assumption:well-separatedness}, \ref{assumption:item-regularity}) and assuming the cluster gap \(\gamma\) is known, the expected regret of the \UniSCLUB (\Cref{algo:sclucb}) satisfies:
  \begin{align*}
    \E[R(T)] = O\ab(\frac{ud}{\gamma^2\lambda_x}\log(T) + d\sqrt{mT}\log(T)).
  \end{align*}
\end{restatable}

\begin{remark}
The regret bounds in \Cref{thm:regret-uniclub,thm:regret-unisclub} comprise two components: the first term corresponds to the cost of cluster identification, and the second term aligns with the minimax near-optimal regret of linear contextual bandits.
Notably, thanks to the additional knowledge of parameter \(\gamma\), our algorithms \UniCLUB and \UniSCLUB enhance the first term by \(\widetilde{O}(u/\lambda_x^2)\), offering a significant improvement over existing studies, as detailed in \Cref{tab:comparison}.
\end{remark}

\begin{restatable}[Regret of \PhaseUniCLUB]{theorem}{restateregretgammaunknown}\label{thm:regret-gamma-unknown}
  Under the stochastic context setting (Assumptions~\ref{assumption:user-uniformness}, \ref{assumption:well-separatedness}, \ref{assumption:item-regularity}), the expected regret of algorithm \PhaseUniCLUB (\Cref{algo:clucb-unknown-gamma}) satisfies:
  \begin{align*}
    \E[R(T)] = O\ab(\frac{ud}{\gamma^5\lambda_x^2}\log(T) + \ab(\frac{ud}{\lambda_x} \log(T))^{\frac{2}{3}} T^{\frac{1}{3}} + d\sqrt{mT}\log(T)).
  \end{align*}
\end{restatable}

\begin{remark}
  The regret bound in \Cref{thm:regret-gamma-unknown} consists of three components: the first term reflects the cost of cluster identification, the second term arises from misclustering, and the third term matches the minimax near-optimal regret of linear contextual bandits.
  Compared to the regret bounds in \Cref{thm:regret-uniclub,thm:regret-unisclub}, the first term becomes larger but grows only logarithmically with \(T\).
  The additional second term scales as \(\widetilde{O}(T^{1/3})\), which is asymptotically smaller than the third term.
  Consequently, the overall regret remains dominated by the minimax near-optimal regret bound of order \(\widetilde{O}(\sqrt{T})\).
\end{remark}

\begin{restatable}[Regret of \SACLUB and \SASCLUB]{theorem}{restateregretsmoothedadversary}\label{thm:regret-smoothed-adversary}
  Under the smoothed adversarial context setting (Assumptions~\ref{assumption:user-uniformness}, \ref{assumption:well-separatedness}, \ref{assumption:item-regularity2}), the expected regrets of \SACLUB and \SASCLUB both satisfy:
  \begin{align*}
    \E[R(T)] = O\ab(\frac{ud}{\gamma^2\widetilde{\lambda}_x} \log(T) + d\sqrt{mT}\log(T)),
  \end{align*}
  where \(\widetilde{\lambda}_x = c_1 \frac{\sigma^2}{\log K}\) for some constant \(c_1\).
\end{restatable}

\begin{remark}
For the smoothed adversarial context setting, we prove that the intrinsic diversity of contexts guarantees a lower bound on the minimum eigenvalue of \(\E\ab[\vec{x}_{a_t}\vec{x}_{a_t}^\mathsf{T}]\) (\Cref{lemma:bound-smallest-eigenvalue-smoothed-adversary}).
This allows us to apply techniques similar to those used in the stochastic context setting to get this result.
\end{remark}

\section{Performance Evaluation}
\label{sec:performance-evaluation}
In this section, we present the evaluation results of our algorithms.
We focus on the stochastic context setting in the main paper since the smoothed adversarial setting serves mainly for theoretical analysis, and algorithms \SACLUB/\SASCLUB are minor modifications of CLUB~\citep{gentile-2014-online-clustering}/SCLUB~\citep{li-2019-improved-algorithm}.
Nonetheless, we provide detailed evaluations of the smoothed adversarial setting and an ablation study on different arm set sizes and user numbers in \Cref{sec:ablation-study}.

\subsection{Experiment Setup}
\label{sec:experiment-setup}
We compare our algorithms against the following state-of-the-art algorithms for clustering of bandits:
(1) LinUCB-One: LinUCB~\citep{li-2010-a-contextual} with a single preference vector shared across all users.
(2) LinUCB-Ind: LinUCB~\citep{li-2010-a-contextual} with separate preference vectors estimated for each user.
(3) CLUB~\citep{gentile-2014-online-clustering}: A graph-based algorithm that consistently employs the UCB strategy.
(4) SCLUB~\citep{li-2019-improved-algorithm}: A set-based algorithm with improved practical performance.
In addition to these clustering-based baselines, we also evaluate our approach against two graph-based algorithms from a related but distinct setting in \Cref{sec:ablation-study}: GOB.Lin~\citep{cesa-2013-gang}, which incorporates user similarities using Laplacian regularization, and GraphUCB~\citep{yang-2020-laplacian}, which utilizes the random-walk Laplacian matrix to encode user relationships.
Note that neither GOB.Lin nor GraphUCB explicitly performs clustering.
All the experiments were conducted on a device equipped with a 3.60 GHz Intel Xeon W-2223 CPU and 32GB RAM.
Each experiment was repeated over 5 random seeds, and the results are reported with confidence intervals calculated by dividing the standard deviation by the square root of the number of seeds.

\subsection{Datasets Generation and Preprocessing}
\label{sec:dataset-generation-preprocessing}
  In our experiments, we employ one synthetic dataset and three real-world datasets, MovieLens-25M~\citep{harper-2015-movielens}, Last.fm~\citep{cantador-2011-second-workshop}, and Yelp~\citep{yelp-dataset}.
  We generate the synthetic dataset and preprocess the real-world datasets following the same method in previous studies~\citep{zhang-2020-conversational,li-2024-fedconpe,dai-2024-conversational,dai-2024-online}.

  To generate the synthetic dataset, we set the dimension \(d=50\), the number of users \(u=200\), and the total number of arms \(|\mathcal{A}|=5,000\).
  The feature vector \(\vec{x}_a \in \RR^d\) of each arm \(a \in \mathcal{A}\) and the preference vector \(\vec{\theta}_{i} \in \RR^d\) for each user \(i \in [u]\) are generated by independently sampling each dimension from a uniform distribution \(\mathcal{U}(-1,1)\) and then normalizing to unit length.

  For the real-world datasets, MovieLens-25M, Last.fm, and Yelp, we regard movies\slash artists\slash businesses as arms.
  We extract a subset of \(|\mathcal{A}|=5,000\) arms with the most quantity of user-assigned ratings/tags, and a subset of \(u=200\) users who assign the most quantity of ratings/tags.
  Using the data extracted above, we create a \emph{feedback matrix} \(\vec{R}\) of size \({u \times |\mathcal{A}|}\), where each element \(\vec{R}_{i,j}\) represents the user \(i\)'s feedback to arm \(j\).
  We assume that the user's feedback is binary.
  For the MovieLens and Yelp datasets, a user's feedback for a movie/business is 1 if the user's rating is higher than 3.
  For the Last.fm dataset, a user's feedback for an artist is 1 if the user assigns a tag to the artist.
  We generate feature vectors and preference vectors by decomposing \(\vec{R}\) using Singular Value Decomposition (SVD) as \(\vec{R}=\vec{\Theta} \vec{S} \vec{A}^\mathsf{T}\), where \(\vec{\Theta} = \set{\vec{\theta}_i}_{i \in [u]}\), and \(\vec{A} = \set{\vec{x}_a}_{a \in \mathcal{A}}\).
  Then the top \(d=50\) dimensions of these vectors associated with the highest singular values in \(\vec{S}\) are extracted.

\subsection{Experiment Results}
  In the experiment, to incorporate user clustering, we randomly select \(50\) users and partition them into 10 clusters.
  For each cluster \(j\), we calculate the mean preference vector across all users within that cluster to serve as \(\vec{\theta}^{j}\).
  Note that the number of clusters is unknown to the algorithms.
  At each round \(t\), we uniformly draw a user \(i_t\) from the 50 users, and randomly select 100 arms from \(\mathcal{A}\) to form the arm set \(\mathcal{A}_t\). The results are presented in~\Cref{fig:regret-stochastic-context}.

  As shown in \Cref{fig:regret-stochastic-context}, the algorithms that employ user clustering significantly outperform LinUCB-Ind and LinUCB-One, which do not consider the similarity among users or cluster users.
  More importantly, our graph-based algorithm \UniCLUB consistently outperforms the graph-based baseline CLUB, and our set-based algorithm \UniSCLUB is better than the set-based baseline SCLUB across all four datasets.
  It is important to note that the modest advantage of \UniCLUB/\UniSCLUB over CLUB/SCLUB is both expected and reasonable, given the logarithmic improvement in the regret upper bound.
  The results demonstrate the effect of uniform exploration and the robustness of our proposed algorithms across various datasets. More evaluation under the smoothed adversarial context setting and ablation study can be found in \Cref{sec:ablation-study}.

  \begin{figure}[htb]
    \centering
    \includegraphics[width=\linewidth]{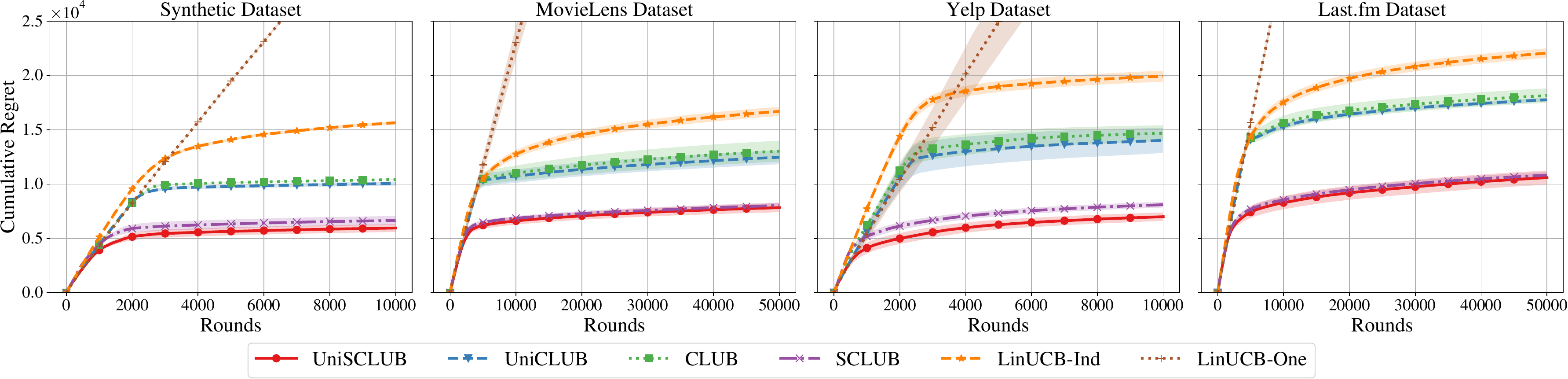}
    \caption{\label{fig:regret-stochastic-context} Comparison of cumulative regrets in the stochastic context setting.}
  \end{figure}
  \vspace{-3pt}

\section{Conclusion}
\label{sec:conclusion}
In this paper, we addressed a long-standing open problem in the online clustering of bandits literature.
We proposed two new algorithms, \UniCLUB and \PhaseUniCLUB, which incorporate explicit exploration to enhance the identification of user clusters.
Notably, our algorithms require significantly weaker assumptions while achieving cumulative regrets comparable to previous studies.
Furthermore, we introduced the smoothed adversarial context setting, which interpolates between \emph{i.i.d.} and fully adversarial context generation.
We showed that with minor modifications, existing algorithms achieve improved performance in this new setting.
Finally, we conducted extensive evaluations to validate the effectiveness of our methods.

\bibliography{reference}

\begin{thebibliography}{36}
\providecommand{\natexlab}[1]{#1}
\providecommand{\url}[1]{\texttt{#1}}
\expandafter\ifx\csname urlstyle\endcsname\relax
  \providecommand{\doi}[1]{doi: #1}\else
  \providecommand{\doi}{doi: \begingroup \urlstyle{rm}\Url}\fi

\bibitem[Abbasi-Yadkori et~al.(2011)Abbasi-Yadkori, P\'{a}l, and
  Szepesv\'{a}ri]{abbasi-2011-improved}
Yasin Abbasi-Yadkori, D\'{a}vid P\'{a}l, and Csaba Szepesv\'{a}ri.
\newblock Improved algorithms for linear stochastic bandits.
\newblock In \emph{Proceedings of the 24th International Conference on Neural
  Information Processing Systems}, NIPS'11, pp.\  2312–2320, 2011.

\bibitem[Amani et~al.(2019)Amani, Alizadeh, and
  Thrampoulidis]{amani-2019-linear-stochastic}
Sanae Amani, Mahnoosh Alizadeh, and Christos Thrampoulidis.
\newblock Linear stochastic bandits under safety constraints.
\newblock In \emph{Advances in Neural Information Processing Systems},
  volume~32. Curran Associates, Inc., 2019.

\bibitem[Ban \& He(2021)Ban and He]{ban-2021-local-clustering}
Yikun Ban and Jingrui He.
\newblock Local clustering in contextual multi-armed bandits.
\newblock In \emph{Proceedings of the Web Conference 2021}, WWW '21, pp.\
  2335–2346, 2021.
\newblock ISBN 9781450383127.

\bibitem[Bastani et~al.(2021)Bastani, Bayati, and
  Khosravi]{bastani-2021-mostly-exploration}
Hamsa Bastani, Mohsen Bayati, and Khashayar Khosravi.
\newblock Mostly exploration-free algorithms for contextual bandits.
\newblock \emph{Management Science}, 67\penalty0 (3):\penalty0 1329–1349,
  March 2021.
\newblock ISSN 1526-5501.

\bibitem[Cantador et~al.(2011)Cantador, Brusilovsky, and
  Kuflik]{cantador-2011-second-workshop}
Ivan Cantador, Peter Brusilovsky, and Tsvi Kuflik.
\newblock Second workshop on information heterogeneity and fusion in
  recommender systems (hetrec2011).
\newblock In \emph{Proceedings of the Fifth ACM Conference on Recommender
  Systems}, RecSys '11, pp.\  387–388, 2011.

\bibitem[Cesa-Bianchi et~al.(2013)Cesa-Bianchi, Gentile, and
  Zappella]{cesa-2013-gang}
Nicol\`{o} Cesa-Bianchi, Claudio Gentile, and Giovanni Zappella.
\newblock A gang of bandits.
\newblock In \emph{Advances in Neural Information Processing Systems},
  volume~26, 2013.

\bibitem[Chatterji et~al.(2020)Chatterji, Muthukumar, and
  Bartlett]{chatterji-2020-simultaneously-optimal}
Niladri Chatterji, Vidya Muthukumar, and Peter Bartlett.
\newblock Osom: A simultaneously optimal algorithm for multi-armed and linear
  contextual bandits.
\newblock In \emph{Proceedings of the Twenty Third International Conference on
  Artificial Intelligence and Statistics}, volume 108 of \emph{Proceedings of
  Machine Learning Research}, pp.\  1844--1854. PMLR, 26--28 Aug 2020.

\bibitem[Chu et~al.(2011)Chu, Li, Reyzin, and Schapire]{chu-2011-contextual}
Wei Chu, Lihong Li, Lev Reyzin, and Robert Schapire.
\newblock Contextual bandits with linear payoff functions.
\newblock In \emph{Proceedings of the Fourteenth International Conference on
  Artificial Intelligence and Statistics}, pp.\  208--214, 2011.

\bibitem[Dai et~al.(2024{\natexlab{a}})Dai, Wang, Xie, Liu, and
  Lui]{dai-2024-conversational}
Xiangxiang Dai, Zhiyong Wang, Jize Xie, Xutong Liu, and John~C.S. Lui.
\newblock Conversational recommendation with online learning and clustering on
  misspecified users.
\newblock \emph{IEEE Transactions on Knowledge and Data Engineering}, pp.\
  1--14, 2024{\natexlab{a}}.

\bibitem[Dai et~al.(2024{\natexlab{b}})Dai, Wang, Xie, Yu, and
  Lui]{dai-2024-online}
Xiangxiang Dai, Zhiyong Wang, Jize Xie, Tong Yu, and John~CS Lui.
\newblock Online learning and detecting corrupted users for conversational
  recommendation systems.
\newblock \emph{IEEE Transactions on Knowledge and Data Engineering},
  2024{\natexlab{b}}.

\bibitem[Gentile et~al.(2014)Gentile, Li, and
  Zappella]{gentile-2014-online-clustering}
Claudio Gentile, Shuai Li, and Giovanni Zappella.
\newblock Online clustering of bandits.
\newblock In \emph{Proceedings of the 31st International Conference on Machine
  Learning}, volume~32 of \emph{Proceedings of Machine Learning Research}, pp.\
   757--765. PMLR, 22--24 Jun 2014.

\bibitem[Gentile et~al.(2017)Gentile, Li, Kar, Karatzoglou, Zappella, and
  Etrue]{gentile-2017-on-context-dependent}
Claudio Gentile, Shuai Li, Purushottam Kar, Alexandros Karatzoglou, Giovanni
  Zappella, and Evans Etrue.
\newblock On context-dependent clustering of bandits.
\newblock In \emph{Proceedings of the 34th International Conference on Machine
  Learning}, volume~70 of \emph{Proceedings of Machine Learning Research}, pp.\
   1253--1262. PMLR, 06--11 Aug 2017.

\bibitem[Ghosh \& Sankararaman(2022)Ghosh and
  Sankararaman]{ghosh-2022-breaking-the-barrier}
Avishek Ghosh and Abishek Sankararaman.
\newblock Breaking the $\sqrt{T}$ barrier: Instance-independent logarithmic
  regret in stochastic contextual linear bandits.
\newblock In \emph{Proceedings of the 39th International Conference on Machine
  Learning}, volume 162 of \emph{Proceedings of Machine Learning Research},
  pp.\  7531--7549. PMLR, 17--23 Jul 2022.

\bibitem[Ghosh et~al.(2021{\natexlab{a}})Ghosh, Sankararaman, and
  Kannan]{ghosh-2021-problem-complexity}
Avishek Ghosh, Abishek Sankararaman, and Ramchandran Kannan.
\newblock Problem-complexity adaptive model selection for stochastic linear
  bandits.
\newblock In \emph{Proceedings of The 24th International Conference on
  Artificial Intelligence and Statistics}, volume 130 of \emph{Proceedings of
  Machine Learning Research}, pp.\  1396--1404. PMLR, 13--15 Apr
  2021{\natexlab{a}}.

\bibitem[Ghosh et~al.(2021{\natexlab{b}})Ghosh, Sankararaman, and
  Ramchandran]{ghosh-2021-collaborative}
Avishek Ghosh, Abishek Sankararaman, and Kannan Ramchandran.
\newblock Collaborative learning and personalization in multi-agent stochastic
  linear bandits.
\newblock \emph{stat}, 1050:\penalty0 15, 2021{\natexlab{b}}.

\bibitem[Hao et~al.(2020)Hao, Lattimore, and
  Szepesvari]{hao-2020-adaptive-exploration}
Botao Hao, Tor Lattimore, and Csaba Szepesvari.
\newblock Adaptive exploration in linear contextual bandit.
\newblock In \emph{Proceedings of the Twenty Third International Conference on
  Artificial Intelligence and Statistics}, volume 108 of \emph{Proceedings of
  Machine Learning Research}, pp.\  3536--3545. PMLR, 26--28 Aug 2020.

\bibitem[Harper \& Konstan(2015)Harper and Konstan]{harper-2015-movielens}
F.~Maxwell Harper and Joseph~A. Konstan.
\newblock The movielens datasets: History and context.
\newblock \emph{ACM Trans. Interact. Intell. Syst.}, 5\penalty0 (4), dec 2015.

\bibitem[Kannan et~al.(2018)Kannan, Morgenstern, Roth, Waggoner, and
  Wu]{kannan-2018-a-smoothed}
Sampath Kannan, Jamie Morgenstern, Aaron Roth, Bo~Waggoner, and Zhiwei~Steven
  Wu.
\newblock A smoothed analysis of the greedy algorithm for the linear contextual
  bandit problem.
\newblock In \emph{Proceedings of the 32nd International Conference on Neural
  Information Processing Systems}, NIPS'18, pp.\  2231–2241, 2018.

\bibitem[Li et~al.(2010)Li, Chu, Langford, and Schapire]{li-2010-a-contextual}
Lihong Li, Wei Chu, John Langford, and Robert~E. Schapire.
\newblock A contextual-bandit approach to personalized news article
  recommendation.
\newblock In \emph{Proceedings of the 19th International Conference on World
  Wide Web}, WWW '10, pp.\  661–670, 2010.

\bibitem[Li \& Zhang(2018)Li and Zhang]{li-2018-online-clustering}
Shuai Li and Shengyu Zhang.
\newblock Online clustering of contextual cascading bandits.
\newblock In \emph{Proceedings of the Thirty-Second AAAI Conference on
  Artificial Intelligence and Thirtieth Innovative Applications of Artificial
  Intelligence Conference and Eighth AAAI Symposium on Educational Advances in
  Artificial Intelligence}, AAAI'18/IAAI'18/EAAI'18, 2018.

\bibitem[Li et~al.(2016)Li, Karatzoglou, and Gentile]{li-2016-collaborative}
Shuai Li, Alexandros Karatzoglou, and Claudio Gentile.
\newblock Collaborative filtering bandits.
\newblock In \emph{Proceedings of the 39th International ACM SIGIR Conference
  on Research and Development in Information Retrieval}, SIGIR '16, pp.\
  539–548, 2016.
\newblock ISBN 9781450340694.

\bibitem[Li et~al.(2019)Li, Chen, Li, and Leung]{li-2019-improved-algorithm}
Shuai Li, Wei Chen, Shuai Li, and Kwong-Sak Leung.
\newblock Improved algorithm on online clustering of bandits.
\newblock In \emph{Proceedings of the Twenty-Eighth International Joint
  Conference on Artificial Intelligence, {IJCAI-19}}, pp.\  2923--2929.
  International Joint Conferences on Artificial Intelligence Organization, 7
  2019.

\bibitem[Li et~al.(2024)Li, Liu, and Lui]{li-2024-fedconpe}
Zhuohua Li, Maoli Liu, and John~C.S. Lui.
\newblock Fedconpe: Efficient federated conversational bandits with
  heterogeneous clients.
\newblock In \emph{Proceedings of the Thirty-Third International Joint
  Conference on Artificial Intelligence, {IJCAI-24}}. International Joint
  Conferences on Artificial Intelligence Organization, 2024.

\bibitem[Liu et~al.(2022)Liu, Zhao, Yu, Li, and Lui]{liu-2022-federated}
Xutong Liu, Haoru Zhao, Tong Yu, Shuai Li, and John Lui.
\newblock Federated online clustering of bandits.
\newblock In \emph{The 38th Conference on Uncertainty in Artificial
  Intelligence}, 2022.

\bibitem[Raghavan et~al.(2023)Raghavan, Slivkins, Vaughan, and
  Wu]{raghavan-2023-greedy}
Manish Raghavan, Aleksandrs Slivkins, Jennifer~Wortman Vaughan, and
  Zhiwei~Steven Wu.
\newblock Greedy algorithm almost dominates in smoothed contextual bandits.
\newblock \emph{SIAM Journal on Computing}, 52\penalty0 (2):\penalty0 487--524,
  2023.

\bibitem[Sivakumar et~al.(2020)Sivakumar, Wu, and
  Banerjee]{sivakumar-2020-structured-linear}
Vidyashankar Sivakumar, Steven Wu, and Arindam Banerjee.
\newblock Structured linear contextual bandits: A sharp and geometric smoothed
  analysis.
\newblock In \emph{Proceedings of the 37th International Conference on Machine
  Learning}, volume 119 of \emph{Proceedings of Machine Learning Research},
  pp.\  9026--9035. PMLR, 13--18 Jul 2020.

\bibitem[Sivakumar et~al.(2022)Sivakumar, Zuo, and
  Banerjee]{sivakumar-2022-smoothed}
Vidyashankar Sivakumar, Shiliang Zuo, and Arindam Banerjee.
\newblock Smoothed adversarial linear contextual bandits with knapsacks.
\newblock In Kamalika Chaudhuri, Stefanie Jegelka, Le~Song, Csaba Szepesvari,
  Gang Niu, and Sivan Sabato (eds.), \emph{Proceedings of the 39th
  International Conference on Machine Learning}, volume 162 of
  \emph{Proceedings of Machine Learning Research}, pp.\  20253--20277. PMLR,
  17--23 Jul 2022.

\bibitem[Spielman \& Teng(2004)Spielman and Teng]{spielman-2004-smoothed}
Daniel~A. Spielman and Shang-Hua Teng.
\newblock Smoothed analysis of algorithms: Why the simplex algorithm usually
  takes polynomial time.
\newblock \emph{J. ACM}, 51\penalty0 (3):\penalty0 385–463, may 2004.
\newblock ISSN 0004-5411.

\bibitem[Tropp(2011)]{tropp-2011-user-friendly}
Joel~A. Tropp.
\newblock User-friendly tail bounds for sums of random matrices.
\newblock \emph{Foundations of Computational Mathematics}, 12\penalty0
  (4):\penalty0 389–434, August 2011.
\newblock ISSN 1615-3383.

\bibitem[Wang et~al.(2023{\natexlab{a}})Wang, Xie, Liu, Li, and
  Lui]{wang-2023-online-clustering}
Zhiyong Wang, Jize Xie, Xutong Liu, Shuai Li, and John~C.S. Lui.
\newblock Online clustering of bandits with misspecified user models.
\newblock In \emph{Thirty-seventh Conference on Neural Information Processing
  Systems}, 2023{\natexlab{a}}.

\bibitem[Wang et~al.(2023{\natexlab{b}})Wang, Xie, Yu, Li, and
  Lui]{wang-2023-online-corrupted}
Zhiyong Wang, Jize Xie, Tong Yu, Shuai Li, and John~C.S. Lui.
\newblock Online corrupted user detection and regret minimization.
\newblock In \emph{Thirty-seventh Conference on Neural Information Processing
  Systems}, 2023{\natexlab{b}}.

\bibitem[Wu et~al.(2020)Wu, Yang, and Shen]{wu-2020-stochastic-linear}
Weiqiang Wu, Jing Yang, and Cong Shen.
\newblock Stochastic linear contextual bandits with diverse contexts.
\newblock In \emph{Proceedings of the Twenty Third International Conference on
  Artificial Intelligence and Statistics}, volume 108 of \emph{Proceedings of
  Machine Learning Research}, pp.\  2392--2401. PMLR, 26--28 Aug 2020.

\bibitem[Yang et~al.(2024)Yang, Liu, Wang, Xie, Lui, Lian, and
  Chen]{yang-2024-federated}
Hantao Yang, Xutong Liu, Zhiyong Wang, Hong Xie, John C.~S. Lui, Defu Lian, and
  Enhong Chen.
\newblock Federated contextual cascading bandits with asynchronous
  communication and heterogeneous users.
\newblock \emph{Proceedings of the AAAI Conference on Artificial Intelligence},
  38\penalty0 (18):\penalty0 20596--20603, Mar. 2024.

\bibitem[Yang et~al.(2020)Yang, Toni, and Dong]{yang-2020-laplacian}
Kaige Yang, Laura Toni, and Xiaowen Dong.
\newblock Laplacian-regularized graph bandits: Algorithms and theoretical
  analysis.
\newblock In \emph{Proceedings of the Twenty Third International Conference on
  Artificial Intelligence and Statistics}, volume 108 of \emph{Proceedings of
  Machine Learning Research}, pp.\  3133--3143. PMLR, 26--28 Aug 2020.

\bibitem[Yelp(2023)]{yelp-dataset}
Yelp.
\newblock {Y}elp {D}ataset --- yelp.com.
\newblock \url{https://www.yelp.com/dataset}, 2023.
\newblock [Accessed 21-05-2024].

\bibitem[Zhang et~al.(2020)Zhang, Xie, Li, and
  C.S.~Lui]{zhang-2020-conversational}
Xiaoying Zhang, Hong Xie, Hang Li, and John C.S.~Lui.
\newblock Conversational contextual bandit: Algorithm and application.
\newblock In \emph{Proceedings of The Web Conference 2020}, WWW '20, pp.\
  662–672, 2020.

\end{thebibliography}
\bibliographystyle{iclr2025_conference}

\appendix
\clearpage
\appendix
\allowdisplaybreaks %

\section{More Discussions on Related Work}
\label{sec:more-related-works}
Leveraging data diversity (i.e., conditions refer to the minimum eigenvalue of a design matrix) in stochastic linear contextual bandits has two lines of research.
The first line involves using additional ``diversity conditions'' to improve cumulative regrets.
For example, \citet{bastani-2021-mostly-exploration} introduce a condition for the disjoint-parameter case, and prove that a non-explorative greedy algorithm achieves \(O(\log T)\) problem-dependent regret on a 2-arm bandit instance.
\citet{hao-2020-adaptive-exploration} give a condition and prove a constant problem-dependent regret for LinUCB in the shared-parameter case.
\citet{wu-2020-stochastic-linear} show that under some diversity conditions, LinUCB achieves constant expected regret in the disjoint-parameter case.
\citet{ghosh-2022-breaking-the-barrier} use a condition similar to \citet{gentile-2014-online-clustering} and achieve a problem-independent logarithmic regret for linear contextual bandits.
The second line of research focuses on achieving concurrent statistical inference and regret minimization (i.e., multi-objective MAB).
This involves performing additional tasks on top of regret minimization, such as clustering~\citep{gentile-2014-online-clustering}, model selection~\citep{chatterji-2020-simultaneously-optimal,ghosh-2021-problem-complexity}, personalization~\citep{ghosh-2021-collaborative}, and exploration under safety constraints~\citep{amani-2019-linear-stochastic}.

\section{Theoretical Analysis of \UniCLUB}
\label{sec:proof-of-uniclub}

\begin{lemma}\label{lemma:bound-theta-precision}
  With probability at least \(1-\delta\) for any \(\delta \in (0,1)\), \(\forall t \in [T]\) and \(\forall i \in [u]\),
  \[\ab\|\widehat{\vec{\theta}}_{i,t} - \vec{\theta}^{j(i)}\|_2 \leq \frac{\sqrt{2 \log\ab(\frac{u}{\delta}) + d \log\ab(1+\frac{T_{i,t}L^2}{\lambda d})} + \sqrt{\lambda}}{\sqrt{\lambda + \lambda_{\min}(\vec{S}_{i,t})}}.\]
\end{lemma}

\begin{proof}
  Fix a user \(i \in [u]\), for all \(t \in [T]\), we have
  \begin{align*}
    &\widehat{\vec{\theta}}_{i,t} - \vec{\theta}^{j(i)} = \ab(\lambda \vec{I} + \sum_{\tau \in [t]: i_{\tau}=i} \vec{x}_{a_{\tau}}\vec{x}_{a_{\tau}}^\mathsf{T})^{-1} \ab(\sum_{\tau \in [t]: i_{\tau}=i} \vec{x}_{a_{\tau}} (\vec{x}_{a_{\tau}}^\mathsf{T} \vec{\theta}^{j(i)} + \eta_\tau)) - \vec{\theta}^{j(i)}\\
    =& \ab(\lambda \vec{I} + \sum_{\tau \in [t]: i_{\tau}=i} \vec{x}_{a_{\tau}}\vec{x}_{a_{\tau}}^\mathsf{T})^{-1} \ab[\ab(\lambda \vec{I} + \sum_{\tau \in [t]: i_{\tau}=i}\vec{x}_{a_{\tau}}\vec{x}_{a_{\tau}}^\mathsf{T}) \vec{\theta}^{j(i)} + \sum_{\tau \in [t]: i_{\tau}=i}\vec{x}_{a_{\tau}}\eta_{\tau} - \lambda \vec{\theta}^{j(i)}] - \vec{\theta}^{j(i)}\\
    =& \vec{\theta}^{j(i)} +(\lambda \vec{I} + \vec{S}_{i,t})^{-1}\sum_{\tau \in [t]: i_{\tau}=i} \vec{x}_{a_{\tau}} \eta_{\tau} -\lambda (\lambda \vec{I} + \vec{S}_{i,t})^{-1} \vec{\theta}^{j(i)} -\vec{\theta}^{j(i)}\\
    =& \overline{\vec{S}}_{i,t}^{-1}\sum_{\tau \in [t]: i_{\tau}=i} \vec{x}_{a_{\tau}} \eta_{\tau} -\lambda \overline{\vec{S}}_{i,t}^{-1} \vec{\theta}^{j(i)},
  \end{align*}
  where we denote \(\overline{\vec{S}}_{i,t} \triangleq \lambda \vec{I} + \vec{S}_{i,t} = \lambda \vec{I} + \sum_{\tau \in [t]: i_{\tau}=i} \vec{x}_{a_{\tau}}\vec{x}_{a_{\tau}}^\mathsf{T}\).

  For any vector \(\vec{x} \in \RR^d\),
  \begin{align*}
    \vec{x}^\mathsf{T}\ab(\widehat{\vec{\theta}}_{i,t} - \vec{\theta}^{j(i)}) &= \vec{x}^\mathsf{T}\overline{\vec{S}}_{i,t}^{-1}\sum_{\tau \in [t]: i_{\tau}=i} \vec{x}_{a_{\tau}} \eta_{\tau} -\lambda \vec{x}^\mathsf{T}\overline{\vec{S}}_{i,t}^{-1} \vec{\theta}^{j(i)}\\
    &= \inprod{\vec{x}^\mathsf{T}}{\sum_{\tau \in [t]: i_{\tau}=i} \vec{x}_{a_{\tau}} \eta_{\tau}}_{\overline{\vec{S}}_{i,t}^{-1}} - \lambda \inprod{\vec{x}}{\vec{\theta}^{j(i)}}_{\overline{\vec{S}}_{i,t}^{-1}}.
  \end{align*}
  Therefore, by the Cauchy–Schwarz inequality,
  \begin{align}\label{eq:reward-distance}
    \ab| \vec{x}^\mathsf{T}\ab(\widehat{\vec{\theta}}_{i,t} - \vec{\theta}^{j(i)})| \leq \|\vec{x}\|_{\overline{\vec{S}}_{i,t}^{-1}} \ab(\ab\|\sum_{\tau \in [t]: i_{\tau}=i} \vec{x}_{a_{\tau}} \eta_{\tau}\|_{\overline{\vec{S}}_{i,t}^{-1}} + \lambda\|\vec{\theta}^{j(i)}\|_{\overline{\vec{S}}_{i,t}^{-1}})
  \end{align}
  Next, we bound the two terms in the parenthesis.
  For the first term, consider the \(\sigma\)-algebra \(\mathcal{F}_t=\sigma(i_1,\boldsymbol{x}_{a_1},\eta_1, \dots, i_{t},\boldsymbol{x}_{a_t}, \eta_t, i_{t+1},\boldsymbol{x}_{a_{t+1}})\), where \(i_t\) and \(\vec{x}_{a_t}\) are \(\mathcal{F}_{t-1}\)-measurable, and \(\eta_t\) is \(\mathcal{F}_t\)-measurable.
  Let \(\set{\mathcal{F}_t}_{t=1}^{\infty}\) be a filtration.
  By applying Theorem 1 of \cite{abbasi-2011-improved} and using the union bound over all users \(i \in [u]\), we have that for all \(i \in [u]\) and for all \(t \in [T]\), with probability \(\geq 1-\delta\),
  \begin{align*}
    \ab\|\sum_{\tau \in [t]: i_{\tau}=i} \vec{x}_{a_{\tau}} \eta_{\tau}\|_{\overline{\vec{S}}_{i,t}^{-1}} &\leq \sqrt{2 \log \ab(\frac{u\det(\overline{\vec{S}}_{i,t})^{\frac{1}{2}} \det(\lambda \vec{I})^{-\frac{1}{2}}}{\delta})}\\
    &\leq \sqrt{2\log\ab(\frac{u}{\delta}) + d\log\ab(1+\frac{T_{i,t} L^2}{\lambda d})},
  \end{align*}
  where we use the standard derivation due to \Cref{lemma:det-trace-inequality}: \(\det(\overline{\vec{S}}_{i,t}) \leq \ab(\frac{\tr(\overline{\vec{S}}_{i,t})}{d})^d \leq \ab(\lambda + \frac{T_{i,t} L^2}{d})^d\) and the fact that \(\det(\lambda \vec{I}) = \lambda^d\).

  For the second term, by the property of the Rayleigh quotient, for any invertible PSD matrix \(\vec{V}\) and non-zero vector \(\vec{w}\), we have
  \[\frac{\|\vec{w}\|^2_{\vec{V^{-1}}}}{\|\vec{w}\|^2_2} = \frac{\vec{w}^\mathsf{T}\vec{V^{-1}}\vec{w}}{\vec{w}^\mathsf{T}\vec{w}} \leq \lambda_{\max}(\vec{V}^{-1}) = \frac{1}{\lambda_{\min}(\vec{V})}.\]
  Therefore, by the fact that \(\lambda_{\min}(\overline{\vec{S}}_{i,t}) \geq \lambda\), and the assumption that \(\|\vec{\theta}^{j(i)}\|_2 \leq 1\), we have
  \begin{align*}
    \lambda\|\vec{\theta}^{j(i)}\|_{\overline{\vec{S}}_{i,t}^{-1}} \leq \sqrt{\lambda}\|\vec{\theta}^{j(i)}\|_{2} \leq \sqrt{\lambda}.\numberthis\label{eq:sqrt-lambda}
  \end{align*}

  Plugging in \(\vec{x} = \overline{\vec{S}}_{i,t}\ab(\widehat{\vec{\theta}}_{i,t} - \vec{\theta}^{j(i)})\) to \Cref{eq:reward-distance}, we have
  \begin{align*}
    &\ab| \vec{x}^\mathsf{T}\ab(\widehat{\vec{\theta}}_{i,t} - \vec{\theta}^{j(i)})| = \ab\|\widehat{\vec{\theta}}_{i,t} - \vec{\theta}^{j(i)}\|^2_{\overline{\vec{S}}_{i,t}}\\
    \leq& \ab\|\overline{\vec{S}}_{i,t}\ab(\widehat{\vec{\theta}}_{i,t} - \vec{\theta}^{j(i)})\|_{\overline{\vec{S}}_{i,t}^{-1}} \ab(\sqrt{2\log\ab(\frac{u}{\delta}) + d\log\ab(1+\frac{T_{i,t} L^2}{\lambda d})} + \sqrt{\lambda})\\
    =& \ab\|\widehat{\vec{\theta}}_{i,t} - \vec{\theta}^{j(i)}\|_{\overline{\vec{S}}_{i,t}} \ab(\sqrt{2\log\ab(\frac{u}{\delta}) + d\log\ab(1+\frac{T_{i,t} L^2}{\lambda d})} + \sqrt{\lambda}).
  \end{align*}
  Again, by the property of the Rayleigh quotient, for any PSD matrix \(\vec{V}\) and non-zero vector \(\vec{w}\),
  \[\lambda_{\min}(\vec{V}) \leq \frac{\vec{w}^\mathsf{T}\vec{V}\vec{w}}{\vec{w}^\mathsf{T}\vec{w}} = \frac{\|\vec{w}\|^2_{\vec{V}}}{\|\vec{w}\|^2_2} \implies \|\vec{w}\|_2 \leq \frac{\|\vec{w}\|_{\vec{V}}}{\sqrt{\lambda_{\min}(\vec{V})}}.\]
  Therefore, dividing \(\ab\|\widehat{\vec{\theta}}_{i,t} - \vec{\theta}^{j(i)}\|_{\overline{\vec{S}}_{i,t}}\) on both sides and applying the above inequality, we get
  \begin{align*}
    \ab\|\widehat{\vec{\theta}}_{i,t} - \vec{\theta}^{j(i)}\|_2  &\leq \frac{\sqrt{2\log\ab(\frac{u}{\delta}) + d\log\ab(1+\frac{T_{i,t} L^2}{\lambda d})} + \sqrt{\lambda}}{\sqrt{\lambda_{\min}(\overline{\vec{S}}_{i,t})}}\\
    &\leq \frac{\sqrt{2\log\ab(\frac{u}{\delta}) + d\log\ab(1+\frac{T_{i,t} L^2}{\lambda d})} + \sqrt{\lambda}}{\sqrt{\lambda + \lambda_{\min}(\vec{S}_{i,t})}},
  \end{align*}
  where in the last inequality we use \(\lambda_{\min}(\lambda\vec{I}+\vec{S}_{i,t}) \geq \lambda_{\min}(\lambda\vec{I}) + \lambda_{\min}(\vec{S}_{i,t})\), due to Weyl's inequality.
\end{proof}

\begin{lemma}\label{lemma:bound-smallest-eigenvalue}
  In the uniform exploration phase of \Cref{algo:clucb}, with probability at least \(1-\delta\) for any \(\delta \in (0,1)\), if \(T_{i,t} \geq \frac{8L^2}{\lambda_{x}} \log\ab(\frac{ud}{\delta})\) for all users \(i \in [u]\), we have
  \[\lambda_{\min}\ab(\vec{S}_{i,t}) \geq \frac{\lambda_{x}T_{i,t}}{2}, \forall i \in [u].\]
\end{lemma}

\begin{proof}
  To apply the matrix Chernoff bound (\Cref{lemma:matrix-chernoff}), we first verify the required two conditions for the self-adjoint matrices \(\vec{x}_{a_{\tau}}\vec{x}_{a_{\tau}}^\mathsf{T}\) for any \(\tau \in [t]\).
  First, due to the generation process of feature vectors, \(\vec{x}_{a_{\tau}}\vec{x}_{a_{\tau}}^\mathsf{T}\) is independent and obviously positive semi-definite.
  Second, by the Courant-Fischer theorem, we have
  \begin{align*}
    \lambda_{\max}(\vec{x}_{a_{\tau}}\vec{x}_{a_{\tau}}^\mathsf{T}) &= \max_{\vec{w}: \|\vec{w}\|=1} \vec{w}^\mathsf{T}\vec{x}_{a_{\tau}}\vec{x}_{a_{\tau}}^\mathsf{T} \vec{w}
    = \max_{\vec{w}: \|\vec{w}\|=1} (\vec{w}^\mathsf{T}\vec{x}_{a_{\tau}})^2 \leq \max_{\vec{w}: \|\vec{w}\|=1}\|\vec{w}\|^2\|\vec{x}_{a_{\tau}}\|^2 \leq L^2.
  \end{align*}
  According to \Cref{algo:clucb}, in the uniform exploration phase, for all \(\tau \in [t]\), \(a_{\tau}\) is uniformly selected in \(\mathcal{A}_{\tau}\).
  Also, by Assumption~\ref{assumption:item-regularity}, all the feature vectors in \(\mathcal{A}_{\tau}\) are independently sampled from \(\vec{X}\), therefore by \Cref{lemma:uniform-select-maintain-distribution}, \(\vec{x}_{a_{\tau}}\) follows the same distribution as \(\vec{X}\).

  So for a fixed user \(i \in [u]\), we can compute
  \begin{align*}
    \mu_{\min} = \lambda_{\min}\ab(\sum_{\tau \in [t]: i_{\tau}=i} \E[\vec{x}_{a_{\tau}}\vec{x}_{a_{\tau}}^\mathsf{T}]) = \lambda_{\min}\ab(T_{i,t} \E[\vec{X}\vec{X}^\mathsf{T}]) = T_{i,t} \lambda_{x},
  \end{align*}
  where the last equality is due to the minimum eigenvalue in Assumption~\ref{assumption:item-regularity}.
  Now applying \Cref{lemma:matrix-chernoff}, we get for any \(\varepsilon \in (0,1)\),
  \begin{align*}
    \Pr\left[\lambda_{\min}(\vec{S}_{i,t}) \leq (1-\varepsilon)T_{i,t}\lambda_x\right] \leq d \ab[\frac{e^{-\varepsilon}}{(1-\varepsilon)^{1-\varepsilon}}]^{T_{i,t}\lambda_x/L^2}.
  \end{align*}
  Choosing \(\varepsilon=\frac{1}{2}\), we get
  \begin{align*}
    \Pr\left[\lambda_{\min}(\vec{S}_{i,t}) \leq \frac{T_{i,t}\lambda_x}{2}\right] \leq d \ab(\sqrt{2}e^{-\frac{1}{2}})^{T_{i,t}\lambda_x/L^2}.
  \end{align*}
  Letting the RHS be \(\frac{\delta}{u}\), we get \(T_{i,t} = \frac{L^2\log(\frac{ud}{\delta})}{\lambda_x(\frac{1}{2}-\log(2))}\).
  Therefore, for any fixed user \(i \in [u]\), \(\lambda_{\min}(\vec{S}_{i,t}) \geq \frac{T_{i,t}\lambda_x}{2}\) holds with probability at least \(1-\frac{\delta}{u}\) when \(T_{i,t} \geq \frac{8L^2}{\lambda_{x}}\log\ab(\frac{ud}{\delta})\).
  The proof follows by a union bound over all users \(i \in [u]\).
\end{proof}

\begin{lemma}\label{lemma:clusters-correct-after-T0}
  With probability at least \(1-3\delta\), \Cref{algo:clucb} can cluster all the users correctly after
  \begin{align}\label{eq:T0}
    T_0 \triangleq 16u\log\ab(\frac{u}{\delta}) + 4u\max\ab\{\frac{8L^2}{\lambda_{x}}\log\ab(\frac{ud}{\delta}), \frac{512d}{\gamma^2 \lambda_{x}}\log\ab(\frac{u}{\delta})\}.
  \end{align}
\end{lemma}

\begin{proof}
  Combining \Cref{lemma:bound-theta-precision} and \Cref{lemma:bound-smallest-eigenvalue}, we have with probability at least \(1-2\delta\), when \(T_{i,t} \geq \frac{8L^2}{\lambda_{x}}\log\ab(\frac{ud}{\delta})\) for all users \(i \in [u]\),
  \begin{align*}
    \ab\|\widehat{\vec{\theta}}_{i,t} - \vec{\theta}^{j(i)}\|_2 &\leq \frac{\sqrt{2\log\ab(\frac{u}{\delta}) + d\log\ab(1+\frac{T_{i,t} L^2}{\lambda d})} + \sqrt{\lambda}}{\sqrt{\lambda + \lambda_{\min}(\vec{S}_{i,t})}}\\
    &\leq \frac{\sqrt{2\log\ab(\frac{u}{\delta}) + d\log\ab(1+\frac{T_{i,t} L^2}{\lambda d})} + \sqrt{\lambda}}{\sqrt{\lambda + T_{i,t}\lambda_x/2}} \triangleq f(T_{i,t}), \forall i \in [u].
  \end{align*}
  Next, we find a sufficient time step \(T_{i,t}\) such that the following holds:
  \begin{align}\label{eq:seperate-condition}
    f(T_{i,t}) \triangleq \frac{\sqrt{2\log\ab(\frac{u}{\delta}) + d\log\ab(1+\frac{T_{i,t} L^2}{\lambda d})} + \sqrt{\lambda}}{\sqrt{\lambda + T_{i,t}\lambda_x/2}} \leq \frac{\gamma}{4}.
  \end{align}
  We assume \(\lambda \leq 2\log\ab(\frac{u}{\delta}) + d\log\ab(1+\frac{T_{i,t} L^2}{\lambda d})\), which typically holds. Then a sufficient condition for \Cref{eq:seperate-condition} is
  \begin{align*}
    \frac{2\log\ab(\frac{u}{\delta}) + d\log\ab(1+\frac{T_{i,t} L^2}{\lambda d})}{\lambda + T_{i,t}\lambda_x/2} \leq \frac{\gamma^2}{64}.
  \end{align*}
  To make the above equation to hold, it suffices to let
  \begin{align*}
    \frac{2\log\ab(\frac{u}{\delta})}{T_{i,t}\lambda_x/2} \leq \frac{\gamma^2}{128} \quad \text{and} \quad
    \frac{d\log\ab(1+\frac{T_{i,t} L^2}{\lambda d})}{T_{i,t}\lambda_x/2} \leq \frac{\gamma^2}{128}.
  \end{align*}
  The first inequality holds when \(T_{i,t} \geq \frac{512 \log(\frac{u}{\delta})}{\gamma^2\lambda_x}\).
  For the second inequality, by some basic arithmetic (Lemma~\ref{lemma:basic-arithmetic}), a sufficient condition is \(T_{i,t} \geq \frac{512d}{\gamma^2 \lambda_{x}} \log\ab(\frac{256L^2}{\gamma^2 \lambda \lambda_x})\).
  By choosing \(\delta\) such that \(\frac{u}{\delta} \geq \frac{256L^2}{\gamma^2 \lambda \lambda_x}\), we get a sufficient condition for \Cref{eq:seperate-condition}: \(T_{i,t}\geq \frac{512d}{\gamma^2 \lambda_{x}}\log\ab(\frac{u}{\delta})\).

  In summary, in the uniform exploration phase of \Cref{algo:clucb}, with probability at least \(1-2\delta\) for some \(\delta>0\), when
  \begin{align}\label{eq:minimum-sample-times-each-user}
    T_{i,t} \geq \max\ab\{\frac{8L^2}{\lambda_{x}}\log\ab(\frac{ud}{\delta}), \frac{512d}{\gamma^2 \lambda_{x}}\log\ab(\frac{u}{\delta})\}, \forall i \in [u]
  \end{align}
  we have \(\ab\|\widehat{\vec{\theta}}_{i,t} - \vec{\theta}^{j(i)}\|_2 \leq f(T_{i,t}) \leq \frac{\gamma}{4}\) for all users \(i \in [u]\).

  Because of Assumption~\ref{assumption:user-uniformness}, users arrive uniformly, so by Lemma~\ref{lemma:bound-sum-of-bernoulli} and a union bound over all users \(i \in [u]\), \Cref{eq:minimum-sample-times-each-user} holds for all \(i \in [u]\) with probability at least \(1-\delta\) when
  \begin{align*}
  t\geq T_0\triangleq 16u\log\ab(\frac{u}{\delta}) + 4u\max\ab\{\frac{8L^2}{\lambda_{x}}\log\ab(\frac{ud}{\delta}), \frac{512d}{\gamma^2 \lambda_{x}}\log\ab(\frac{u}{\delta})\}.
  \end{align*}
  Therefore, with probability \(1-3\delta\), we have \(\ab\|\widehat{\vec{\theta}}_{i,t} - \vec{\theta}^{j(i)}\|_2 \leq f(T_{i,t}) \leq \frac{\gamma}{4}, \forall i \in [u]\) when \(t\geq T_0\).

  Next, we show that under this condition, the algorithm will cluster all the users correctly.
  To guarantee this, we need to verify two aspects: (1) if users \(k,\ell\) are in the same cluster, then the algorithm will not delete edge \((k,\ell)\); (2) if users \(k,\ell\) are not in the same cluster, then the algorithm will delete edge \((k,\ell)\).
  We show the contrapositive of (1): if edge \((k, \ell)\) is deleted, then users \(k,\ell\) are not in the same cluster. Due to the triangle inequality and the deletion rule in \Cref{algo:clucb}, we have
  \begin{align*}
    \ab\|\vec{\theta}^{j(k)} - \vec{\theta}^{j(\ell)}\|_2
    &\geq \ab\|\widehat{\vec{\theta}}_{k,t} - \widehat{\vec{\theta}}_{\ell,t}\|_2 - \ab\|\widehat{\vec{\theta}}_{k,t} - \vec{\theta}^{j(k)}\|_2 - \ab\|\widehat{\vec{\theta}}_{\ell,t} - \vec{\theta}^{j(\ell)}\|_2\\
    &\geq \ab\|\widehat{\vec{\theta}}_{k,t} - \widehat{\vec{\theta}}_{\ell,t}\|_2 - f(T_{k,t}) - f(T_{\ell,t}) > 0.
  \end{align*}
  So by Assumption~\ref{assumption:well-separatedness}, \(\ab\|\vec{\theta}^{j(k)} - \vec{\theta}^{j(\ell)}\|_2 >0\) implies that users \(k,\ell\) are not in the same cluster.
  For (2), we show that if \(\ab\|\vec{\theta}^{j(k)} - \vec{\theta}^{j(\ell)}\|_2 \geq \gamma\), the algorithm will delete edge \((k, \ell)\). By the triangle inequality,
  \begin{align*}
    \ab\|\widehat{\vec{\theta}}_{k,t} - \widehat{\vec{\theta}}_{\ell,t}\|_2
    &\geq \ab\|\vec{\theta}^{j(k)} - \vec{\theta}^{j(\ell)}\|_2 - \ab\|\widehat{\vec{\theta}}_{k,t} - \vec{\theta}^{j(k)}\|_2 - \ab\|\widehat{\vec{\theta}}_{\ell,t} - \vec{\theta}^{j(\ell)}\|_2\\
    &\geq \gamma - \frac{\gamma}{4} -\frac{\gamma}{4} =\frac{\gamma}{2} \geq f(T_{k,t}) + f(T_{\ell,t}),
  \end{align*}
  which triggers \Cref{algo:clucb} to delete edge \((k,\ell)\).
\end{proof}

\begin{lemma}\label{lemma:C-a-t}
  With probability at least \(1-4\delta\), for all \(t>T_0\), we have
  \begin{align*}
    \ab|\vec{x}_a^\mathsf{T} \ab(\widehat{\vec{\theta}}_{V_t,t-1} - \vec{\theta}^{j(i_t)})| \leq \beta \|\vec{x}_a\|_{\overline{\vec{M}}_{V_t,t-1}^{-1}} \triangleq C_{a,t},
  \end{align*}
  where \(\beta = \sqrt{d \log(1+\frac{TL^2}{d\lambda}) + 2\log(\frac{1}{\delta})} + \sqrt{\lambda}\).
\end{lemma}
\begin{proof}
  Assume that after \(T_0\), the underlying clusters are identified correctly, meaning that \(V_t\) is the true cluster that contains user \(i_t\), i.e., \(V_t= \mathcal{I}_{j(i_t)}\) then we have
  \begin{align*}
    &\widehat{\vec{\theta}}_{V_t,t-1} - \vec{\theta}^{j(i_t)} = \ab(\lambda \vec{I} + \sum_{\tau \in [t-1]: i_{\tau} \in V_t} \vec{x}_{a_{\tau}}\vec{x}_{a_{\tau}}^\mathsf{T})^{-1} \ab(\sum_{\tau \in [t-1]: i_{\tau} \in V_t} r_{\tau} \vec{x}_{a_{\tau}}) - \vec{\theta}^{j(i_t)}\\
    =& \ab(\lambda \vec{I} + \sum_{\tau \in [t-1]: i_{\tau} \in V_t} \vec{x}_{a_{\tau}}\vec{x}_{a_{\tau}}^\mathsf{T})^{-1} \ab(\sum_{\tau \in [t-1]: i_{\tau} \in V_t} \vec{x}_{a_{\tau}} \ab(\vec{x}_{a_{\tau}}^\mathsf{T} \vec{\theta}_{i_t} + \eta_{\tau})) - \vec{\theta}^{j(i_t)}\numberthis\label{eq:introduce-theta-it}\\
    =& \overline{\vec{M}}_{V_t,t-1}^{-1} \ab[\ab(\lambda \vec{I} + \sum_{\tau \in [t-1]: i_{\tau} \in V_t} \vec{x}_{a_{\tau}}\vec{x}_{a_{\tau}}^\mathsf{T})\vec{\theta}_{i_t} - \lambda \vec{\theta}_{i_t} + \sum_{\tau \in [t-1]: i_{\tau} \in V_t} \vec{x}_{a_{\tau}} \eta_{\tau}] - \vec{\theta}^{j(i_t)}\\
    =& -\lambda \overline{\vec{M}}_{V_t,t-1}^{-1} \vec{\theta}_{i_t} + \overline{\vec{M}}_{V_t,t-1}^{-1}\sum_{\tau \in [t-1]: i_{\tau} \in V_t} \vec{x}_{a_{\tau}} \eta_{\tau},
  \end{align*}
  where we denote \(\overline{\vec{M}}_{V_t,t-1} = \lambda \vec{I} + \vec{M}_{V_t,t-1}\). \Cref{eq:introduce-theta-it} is because \(V_t\) is the true cluster that contains \(i_t\), thus by Assumption~\ref{assumption:well-separatedness}, \(\vec{\theta}_{i_{\tau}} = \vec{\theta}_{i_t},  \forall i_{\tau} \in V_t\).
  Therefore, we have
  \begin{align*}
    &\ab|\vec{x}_a^\mathsf{T}\ab(\widehat{\vec{\theta}}_{V_t,t-1} - \vec{\theta}^{j(i_t)})|
    \leq \lambda \ab|\vec{x}_a^\mathsf{T} \overline{\vec{M}}_{V_t,t-1}^{-1} \vec{\theta}^{j(i_t)}| + \ab|\vec{x}_a^\mathsf{T} \overline{\vec{M}}_{V_t,t-1}^{-1}\sum_{\tau \in [t-1]: i_{\tau} \in V_t} \vec{x}_{a_{\tau}} \eta_{\tau}|\\
    =& \lambda \ab|\inprod{\vec{x}_a}{\vec{\theta}^{j(i_t)}}_{\overline{\vec{M}}_{V_t,t-1}^{-1}}| + \ab|\inprod{\vec{x}_a}{ \sum_{\tau \in [t-1]: i_{\tau} \in V_t} \vec{x}_{a_{\tau}} \eta_{\tau}}_{\overline{\vec{M}}_{V_t,t-1}^{-1}} |\\
    \leq& \lambda \ab\|\vec{x}_a\|_{\overline{\vec{M}}_{V_t,t-1}^{-1}} \ab\|\vec{\theta}^{j(i_t)}\|_{\overline{\vec{M}}_{V_t,t-1}^{-1}} + \ab\|\vec{x}_a\|_{\overline{\vec{M}}_{V_t,t-1}^{-1}} \ab\|\sum_{\tau \in [t-1]: i_{\tau} \in V_t} \vec{x}_{a_{\tau}} \eta_{\tau}\|_{\overline{\vec{M}}_{V_t,t-1}^{-1}} \numberthis\label{eq:c-s-inequality}\\
    \leq& \ab\|\vec{x}_a\|_{\overline{\vec{M}}_{V_t,t-1}^{-1}} \ab(\sqrt{\lambda} + \ab\|\sum_{\tau \in [t-1]: i_{\tau} \in V_t} \vec{x}_{a_{\tau}} \eta_{\tau}\|_{\overline{\vec{M}}_{V_t,t-1}^{-1}}), \numberthis\label{eq:simplify-to-sqrt-lambda}
  \end{align*}
  where \Cref{eq:c-s-inequality} is by the Cauchy–Schwarz inequality and \Cref{eq:simplify-to-sqrt-lambda} is derived by \Cref{eq:sqrt-lambda}.
  Following Theorem 1 in \cite{abbasi-2011-improved}, with probability at least \(1-\delta\), for a fixed user \(i \in [u]\), we have
  \begin{align*}
    \ab\|\sum_{\tau \in [t-1]: i_{\tau} \in V_t} \vec{x}_{a_{\tau}} \eta_{\tau}\|_{\overline{\vec{M}}_{V_t,t-1}^{-1}}
    &\leq \sqrt{2 \log\ab(\frac{\det(\overline{\vec{M}}_{V_t,t-1})^{\frac{1}{2}}\det(\lambda \vec{I})^{-\frac{1}{2}}}{\delta})}\\
    &= \sqrt{2\log\ab(\frac{1}{\delta}) + \log\ab(\frac{\det(\overline{\vec{M}}_{V_t,t-1})}{\det(\lambda \vec{I})})}\\
    &\leq \sqrt{2\log\ab(\frac{1}{\delta}) + d\log\ab(1+\frac{TL^2}{\lambda d})},\numberthis\label{eq:use-det-trace-ineq}
  \end{align*}
  where in \Cref{eq:use-det-trace-ineq}, we use \(\det(\lambda \vec{I}) = \lambda^d\) and by \Cref{lemma:det-trace-inequality}, we get
  \[\det(\overline{\vec{M}}_{V_t,t-1}) \leq \ab(\frac{\tr(\overline{\vec{M}}_{V_t,t-1})}{d})^d = \ab(\frac{\tr(\lambda \vec{I}) + \tr\ab(\sum_{\substack{\tau \in [t-1]\\ i_{\tau} \in V_{t}}} \vec{x}_{a_{\tau}}\vec{x}_{a_{\tau}}^\mathsf{T})}{d})^d \leq \ab(\frac{\lambda d + TL^2}{d})^{d}.\]

  Since the clusters are correctly identified after \(T_0\) with probability at least \(1-3\delta\) (\Cref{lemma:clusters-correct-after-T0}), and \Cref{eq:use-det-trace-ineq} holds with probability at least \(1-\delta\). By plugging \Cref{eq:use-det-trace-ineq} into \Cref{eq:simplify-to-sqrt-lambda} and collecting high probability events, we conclude the proof.
\end{proof}

\begin{lemma}\label{lemma:sum-x-sqaure}
  Under the high probability event in \Cref{lemma:clusters-correct-after-T0}, for any cluster \(j \in [m]\), we have
  \[\sum_{\substack{t=T_0+1\\i_t \in \mathcal{I}_j}}^{T} \min\set{1,\ab\|\vec{x}_{a_t}\|^2_{\overline{\vec{M}}_{\mathcal{I}_j,t-1}^{-1}}} \leq 2d \log\ab(1+\frac{TL^2}{\lambda d + 4L^2}),\]
  where \(\overline{\vec{M}}_{\mathcal{I}_j,t-1}^{-1} \triangleq \lambda \vec{I} + \sum_{\tau \in [t-1]: i_{\tau} \in \mathcal{I}_j} \vec{x}_{a_{\tau}}\vec{x}_{a_{\tau}}^\mathsf{T} = \lambda \vec{I} + \sum_{\tau \in [t-1]}\1\set{i_{\tau} \in \mathcal{I}_j} \vec{x}_{a_{\tau}}\vec{x}_{a_{\tau}}^\mathsf{T}\).
\end{lemma}
\begin{proof}
  The proof mainly follows Lemma 9 in \cite{abbasi-2011-improved}, but we improve the bound by leveraging \Cref{lemma:bound-smallest-eigenvalue}.

  Consider the covariance matrix of cluster \(\mathcal{I}_j\) after time \(T\), we have
  \begin{align*}
    &\det\ab(\overline{\vec{M}}_{\mathcal{I}_j,T})\\
    =& \det\ab(\overline{\vec{M}}_{\mathcal{I}_j,T-1} + \1\set{i_{T} \in \mathcal{I}_j}\vec{x}_{a_{T}}\vec{x}_{a_{T}}^\mathsf{T})\\
    =& \det\ab(\overline{\vec{M}}_{\mathcal{I}_j,T-1}^{\frac{1}{2}} \ab(\vec{I}+\overline{\vec{M}}_{\mathcal{I}_j,T-1}^{-\frac{1}{2}}\1\set{i_{T} \in \mathcal{I}_j}\vec{x}_{a_{T}}\vec{x}_{a_{T}}^\mathsf{T} \overline{\vec{M}}_{\mathcal{I}_j,T-1}^{-\frac{1}{2}})\overline{\vec{M}}_{\mathcal{I}_j,T-1}^{\frac{1}{2}})\\
    =& \det\ab(\overline{\vec{M}}_{\mathcal{I}_j,T-1})\det\ab(\vec{I}+\1\set{i_{T} \in \mathcal{I}_j}\overline{\vec{M}}_{\mathcal{I}_j,T-1}^{-\frac{1}{2}}\vec{x}_{a_{T}}\ab(\overline{\vec{M}}_{\mathcal{I}_j,T-1}^{-\frac{1}{2}}\vec{x}_{a_{T}})^\mathsf{T})\\
    =& \det\ab(\overline{\vec{M}}_{\mathcal{I}_j,T-1})\ab(1+\1\set{i_{T} \in \mathcal{I}_j}\ab\|\overline{\vec{M}}_{\mathcal{I}_j,T-1}^{-\frac{1}{2}}\vec{x}_{a_{T}}\|^2) \numberthis\label{eq:property-of-determint}\\
    =& \det\ab(\overline{\vec{M}}_{\mathcal{I}_j,T-1})\ab(1+\1\set{i_{T} \in \mathcal{I}_j}\ab\|\vec{x}_{a_{T}}\|^2_{\overline{\vec{M}}_{\mathcal{I}_j,T-1}^{-1}})\\
    =& \det(\overline{\vec{M}}_{\mathcal{I}_j,T_0}) \prod_{t=T_0+1}^T \ab(1+\ab\|\vec{x}_{a_t}\|^2_{\overline{\vec{M}}_{\mathcal{I}_j,t-1}^{-1}}),\numberthis\label{eq:telescope}
  \end{align*}
  where Equation~\ref{eq:property-of-determint} is due to the property that all the eigenvalues of a matrix of the form \(\vec{I} + \vec{x}\vec{x}^\mathsf{T}\) are one except one eigenvalue, which is \(1 + \|\vec{x}\|^2\). And Equation~\ref{eq:telescope} is by telescoping.

  Since we assume the high probability event in \Cref{lemma:clusters-correct-after-T0} happens, when \(t>T_0\), we have
  \[T_{i,t} \geq \max\ab\{\frac{8L^2}{\lambda_{x}}\log\ab(\frac{ud}{\delta}), \frac{512d}{\gamma^2 \lambda_{x}}\log\ab(\frac{u}{\delta})\} \geq \frac{8L^2}{\lambda_{x}}\log\ab(\frac{ud}{\delta}), \forall i \in [u].\]
  Note that \(\overline{\vec{M}}_{\mathcal{I}_j,T_0} = \lambda \vec{I} + \sum_{\tau=1: i_{\tau} \in \mathcal{I}_{j}}^{T_0} \vec{x}_{a_{\tau}}\vec{x}_{a_{\tau}}^\mathsf{T} = \lambda \vec{I} + \sum_{i \in \mathcal{I}_{j}} S_{i,T_0}\). So by \Cref{lemma:bound-smallest-eigenvalue}, we have
  \begin{align*}
    \tr\ab(\overline{\vec{M}}_{\mathcal{I}_j,T_0})
    &= \tr\ab(\lambda \vec{I} + \sum_{i \in \mathcal{I}_{j}} S_{i,T_0})\\
    &\geq \lambda d + \sum_{i \in \mathcal{I}_{j}} \frac{\lambda_x}{2} \frac{8L^2}{\lambda_{x}}\log\ab(\frac{ud}{\delta}) \geq \lambda d + 4L^2 \log\ab(\frac{ud}{\delta})\numberthis\label{eq:trace}
  \end{align*}
  Therefore, taking logarithms on both sides of \Cref{eq:telescope}, we have
  \begin{align*}
    \log\ab(\det\ab(\overline{\vec{M}}_{\mathcal{I}_j,T}))
    = \log(\det\ab(\overline{\vec{M}}_{\mathcal{I}_j,T_0})) + \sum_{t=T_0+1}^{T} \log\ab(1+\ab\|\vec{x}_{a_t}\|^2_{\overline{\vec{M}}_{\mathcal{I}_j,t-1}^{-1}}).\numberthis\label{eq:take-log-on-both-sides}
  \end{align*}
  Since \(x \leq 2 \log(1 + x)\) for \(x \in [0,1]\), by Equation~\ref{eq:take-log-on-both-sides}, we get
  \begin{align*}
    \sum_{t=T_0+1}^{T} \min\set{1,\ab\|\vec{x}_{a_t}\|^2_{\overline{\vec{M}}_{\mathcal{I}_j,t-1}^{-1}}}
    &\leq 2 \sum_{t=T_0+1}^{T} \log\ab(1+\ab\|\vec{x}_{a_t}\|^2_{\overline{\vec{M}}_{\mathcal{I}_j,t-1}^{-1}})\\
    &= 2\ab[\log\ab(\det\ab(\overline{\vec{M}}_{\mathcal{I}_j,T})) - \log\ab(\det\ab(\overline{\vec{M}}_{\mathcal{I}_j,T_0}))]\numberthis\label{eq:use-result-take-log-on-both-sides}\\
    &\leq 2\ab[d\log\ab(\frac{\tr(\overline{\vec{M}}_{\mathcal{I}_j,T})}{d}) - d\log\ab(\frac{\tr\ab(\overline{\vec{M}}_{\mathcal{I}_j,T_0})}{d})]\numberthis\label{eq:use-result-determinant-trace}\\
    &\leq 2\ab[d\log\ab(\frac{\lambda d+TL^2}{d}) - d\log\ab(\frac{\lambda d + 4L^2 \log\ab(\frac{ud}{\delta})}{d})]\numberthis\label{eq:use-trace-lower-bound}\\
    &\leq 2d \log\ab(1+\frac{TL^2}{\lambda d + 4L^2}).
  \end{align*}
  where \Cref{eq:use-result-take-log-on-both-sides} uses \Cref{eq:take-log-on-both-sides}. \Cref{eq:use-result-determinant-trace} uses the determinant-trace inequality \Cref{lemma:det-trace-inequality}. \Cref{eq:use-trace-lower-bound} uses \Cref{eq:trace}.
\end{proof}

\restateregret*
\begin{proof}
  Define events:
  \begin{align*}
    \mathcal{E}_1 &= \set{\text{all users are correctly clustered after }T_0}.\\
    \mathcal{E}_2 &= \set{\ab|\vec{x}_a^\mathsf{T} \ab(\widehat{\vec{\theta}}_{V_t,t-1} - \vec{\theta}^{j(i_t)})| \leq \beta \|\vec{x}_a\|_{\overline{\vec{M}}_{V_t,t-1}^{-1}}, \forall t \geq T_0}.
  \end{align*}
  Let \(\mathcal{E} = \mathcal{E}_1 \cap \mathcal{E}_{2}\). By Lemma~\ref{lemma:clusters-correct-after-T0} and \ref{lemma:C-a-t}, \(\mathcal{E}\) happens with probability at least \(1-4\delta\). Then let \(\delta=\frac{1}{T}\) and by the law of total expectations:
  \begin{align*}
    \E[R(T)] &= \E[R(T) \mid \mathcal{E}]\Pr\left[\mathcal{E}\right] + \E[R(T) \mid \mathcal{E}^{c}]\Pr\left[\mathcal{E}^{c}\right]\\
    &\leq\E[R(T) \mid \mathcal{E}] \times 1 + T \times \frac{4}{T}. \numberthis\label{eq:expected-regret}
  \end{align*}
  It remains to bound \(\E[R(T) \mid \mathcal{E}]\), i.e., the expected cumulative regret conditioned on event \(\mathcal{E}\).

  Assume \(\mathcal{E}\) happens, by the bounded reward assumption, we can upper bound the regret in the first \(T_0\) rounds by \(T_0\). After \(T_0\), the instantaneous regret at round \(t\) satisfies:
  \begin{align*}
    R_t &= \vec{x}_{a_t^{*}}^\mathsf{T} \vec{\theta}^{j(i_t)} - \vec{x}_{a_t}^\mathsf{T} \vec{\theta}^{j(i_t)}\\
        &= \vec{x}_{a_t^{*}}^\mathsf{T} \ab(\vec{\theta}^{j(i_t)}- \widehat{\theta}_{V_t,t-1}) + \ab(\vec{x}_{a_t^{*}}^\mathsf{T} \widehat{\theta}_{V_t,t-1} + C_{a_t^{*},t}) - \ab(\vec{x}_{a_t}^\mathsf{T} \widehat{\theta}_{V_t,t-1} + C_{a_t,t})\\
        &\quad +\vec{x}_{a_t}^\mathsf{T} \ab(\vec{\theta}^{j(i_t)}- \widehat{\theta}_{V_t,t-1}) -C_{a_t^{*},t} + C_{a_t,t}\\
        &\leq C_{a_t^{*},t} + C_{a_t,t} - C_{a_t^{*},t} + C_{a_t,t} = 2C_{a_t,t},
  \end{align*}
  where the inequality is due to Lemma~\ref{lemma:C-a-t} and the UCB arm selection strategy of our algorithm.

  Therefore, we can bound the cumulative regret conditioned on event \(\mathcal{E}\) as follows:
  \begin{align*}
    R(T) &= \sum_{t=1}^{T} R_t = \sum_{t=1}^{T_0} R_t + \sum_{t=T_0+1}^{T} R_t\\
    &\leq T_0 + \sum_{t=T_0+1}^{T} \min\set{2, R_t} \numberthis\label{eq:Rt-leq-2}\\
    &\leq T_0 + \sum_{j=1}^{m} \sum_{\substack{T_0<t\leq T\\i_t \in \mathcal{I}_j}} \min\set{2,2\beta\ab\|\vec{x}_{a_t}\|_{\overline{\vec{M}}_{\mathcal{I}_t,t-1}^{-1}}} \numberthis\label{eq:use-Rt-leq-2C}\\
    &\leq T_0 + 2\beta \sum_{j=1}^{m} \sum_{\substack{T_0<t\leq T\\i_t \in \mathcal{I}_j}} \min\set{1,\ab\|\vec{x}_{a_t}\|_{\overline{\vec{M}}_{\mathcal{I}_t,t-1}^{-1}}}\\
    &\leq T_0 + 2\beta\sum_{j=1}^{m} \sqrt{T_{\mathcal{I}_j}\sum_{\substack{T_0<t\leq T\\i_t \in \mathcal{I}_j}} \min\set{1,\ab\|\vec{x}_{a_t}\|^2_{\overline{\vec{M}}_{\mathcal{I}_t,t-1}^{-1}}}} \numberthis\label{eq:use-c-s}\\
    &\leq T_0 + 2\beta \sum_{j=1}^{m} \sqrt{T_{\mathcal{I}_j}} \sqrt{2d \log\ab(1+\frac{TL^2}{\lambda d + 4L^2})} \numberthis\label{eq:use-lemma-sum-x-square}\\
    &\leq T_0 + 2\beta \sqrt{m T}\sqrt{2d \log\ab(1+\frac{TL^2}{\lambda d+4L^2})} \numberthis\label{eq:use-c-s-and-sum-T}\\
    &\leq T_0 + 2\ab(\sqrt{d \log\ab(1+\frac{TL^2}{d\lambda}) + 2\log\ab(\frac{1}{\delta})} + \sqrt{\lambda})\sqrt{2dmT \log\ab(1+\frac{TL^2}{\lambda d+4L^2})} \numberthis\label{eq:cumulative-regret},
  \end{align*}
  where \Cref{eq:Rt-leq-2} uses \(R_t\leq 2\). \Cref{eq:use-Rt-leq-2C} uses \Cref{lemma:C-a-t}. \Cref{eq:use-c-s} uses the Cauchy-Schwarz inequality, and we denote the number of times cluster \(j\) is selected as \(T_{\mathcal{I}_j}\). \Cref{eq:use-lemma-sum-x-square} uses \Cref{lemma:sum-x-sqaure}. And \Cref{eq:use-c-s-and-sum-T} is due to the Cauchy-Schwarz inequality and \(\sum_{j=1}^{m} T_{\mathcal{I}_j} = T\).

  Let \(\delta = \frac{1}{T}\) and plug \Cref{eq:cumulative-regret,eq:T0} into \Cref{eq:expected-regret}, we have
  \begin{align*}
    \E\ab[R(T)] &\leq 4+T_0 + 2\ab(\sqrt{d \log\ab(1+\frac{TL^2}{d\lambda}) + 2\log(T)} + \sqrt{\lambda})\sqrt{2dmT \log\ab(1+\frac{TL^2}{\lambda d+4L^2})}\\
    &\leq 4 + 16u\log\ab(uT) + 4u\max\ab\{\frac{8L^2}{\lambda_{x}}\log\ab(udT), \frac{512d}{\gamma^2 \lambda_{x}}\log\ab(uT)\}\\
    &\quad+ 2\ab(\sqrt{d \log\ab(1+\frac{TL^2}{d\lambda}) + 2\log(T)} + \sqrt{\lambda})\sqrt{2dmT \log\ab(1+\frac{TL^2}{\lambda d+4L^2})}\\
    &= O\ab(\frac{ud}{\gamma^2 \lambda_{x}} \log(T) + d\sqrt{mT}\log(T)).
  \end{align*}
\end{proof}

\section{The Set-based Algorithm \UniSCLUB}
\label{sec:unisclucb}
\subsection{Details of the \UniSCLUB Algorithm}
In this section, we introduce a set-based algorithm named Uniform Exploration Set-based Clustering of Bandits (\UniSCLUB), which is inspired by SCLUB~\cite{li-2019-improved-algorithm}. Instead of using a graph structure to maintain the clustering information, 
SCLUB uses a set structure for the same purpose. 
\UniSCLUB inherits the set structure from SCLUB, but incorporates the uniform exploration to enhance its performance.
The set structure not only supports the split operations which are similar to those in the graph structure, but also enables the merging of two clusters when the algorithm identifies that their estimated preference vectors are closely aligned. 
By allowing for both split and merge operations, \UniSCLUB can adapt to the underlying clusters more flexibly and expedite the overall clustering process.

\begin{algorithm}[htb]
  \DontPrintSemicolon
  \SetKwComment{Comment}{$\triangleright$\ }{}
  \SetKwInput{KwInit}{Initialization}
  \KwIn{\(\lambda\), \(\beta\), \(\lambda_x\), \(\delta\), \(L\), \(\gamma\)}
  \KwInit{ Initialize the cluster indexes by \(\mathcal{J} = \{1\}\). \\
    Let \(\vec{M}^{1} = \vec{0}_{d\times d}, \overline{\vec{M}}^{1} = \lambda \vec{I},  \vec{b}^{0} = \vec{0}_{d \times 1}, T^{0}=0, \mathcal{C}^{1} = [u] \).\\
    Let \(\vec{S}_{i,0} = \vec{0}_{d\times d}, \vec{b}_{i,0} = \vec{0}_{d \times 1}, T_{i,0}=0, \forall i \in [u]\).\\
    Define \(f(T_{i,t})=(\sqrt{2\log\ab(u/\delta) + d\log(1+\frac{T_{i,t} L^2}{\lambda d})} + \sqrt{\lambda})/\sqrt{\lambda + T_{i,t}\lambda_x/2}\).\\
    Define \(T_0 \triangleq 16u\log\ab(\frac{u}{\delta}) + 4u\max\ab\{\frac{8L^2}{\lambda_{x}}\log\ab(\frac{ud}{\delta}), \frac{512d}{\gamma^2 \lambda_{x}}\log\ab(\frac{u}{\delta})\}\).
  }

  \SetKwProg{Fn}{Function}{:}{}
  
  \SetKwFunction{Split}{Split}
  \SetKwFunction{Merge}{Merge}

  \For{\(s = 1,2, \dots \)\label{algo:sclub-forloop}}{
    Mark every user unchecked for each cluster.\; \label{algo:sclub-revert}
    \For{\(t = 2^{s-1}, \dots, 2^s-1\) (terminate when \(t>T\))}{
    Receive user index \(i_t\) and arm set \(\mathcal{A}_t\)\; \label{algo:sclub-userindex}
    \uIf{\(t>2T_0\)}{
    Find the cluster \(j\in\mathcal{J}\) satisfying \(i_t \in C^{j}\)\; \label{algo:sclub-clusterindex}
    Select arm \(a_t = \argmax_{a \in \mathcal{A}_t} ({\widehat{\vec{\theta}}^{j}})^\mathsf{T} \vec{x}_{a} + \beta \sqrt{\vec{x}_a^\mathsf{T} (\overline{\vec{M}}^{j})^{-1}\vec{x}_{a}}\) \; \label{algo:sclub-clusterselect}
    }
    \Else{
      Select \(a_t\) uniformly at random from \(\mathcal{A}_t\)\; \label{algo:sclub-uniformselect}
    }
    Receive reward \(r_t\)\; \label{algo:sclub-reward}
    Update statistics for user \(i_t\), others remain unchanged:
    \(\vec{S}_{i_t,t} =\vec{S}_{i_t,t-1} + \vec{x}_{a_t}\vec{x}_{a_t}^\mathsf{T}, \ \ 
    \vec{b}_{i_t,t}=\vec{b}_{i_t,t-1} + r_t\vec{x}_{a_t}\)
    \(T_{i_t,t}=T_{i_t,t-1}+1, \ \ 
    \widehat{\vec{\theta}}_{i_t,t} = \ab(\lambda \vec{I}+\vec{S}_{i_t, t})^{-1}\vec{b}_{i_t,t}\) \; \label{algo:sclub-userupdate}
    Update statistics for cluster \(j\), others remain unchanged: 
    \(\vec{M}^{j} = \vec{M}^{j} + \vec{x}_{a_t}\vec{x}_{a_t}^\mathsf{T}, \ \
    \vec{b}^{j} = \vec{b}^{j} + r_t\vec{x}_{a_t}\)
    \( T^{j} = T^{j} + 1, \ \ 
    \overline{\vec{M}}^{j} = {\vec{M}}^{j} + \lambda \vec{I}, \ \ 
    \widehat{\vec{\theta}}^{j} = (\overline{\vec{M}}^{j})^{-1}\vec{b}^{j} \)\; \label{algo:sclub-clusterupdate}
    Run \Split \; \label{algo:sclub-split}
    Mark user \(i_t\) as checked \; \label{algo:sclub-mark}
    Run \Merge \; \label{algo:sclub-merge}
    }
  }
  \caption{\UniSCLUB: Uniform Exploration Set-based Clustering of Bandits} \label{algo:sclucb}
\end{algorithm}

\begin{algorithm}[htb]
  \DontPrintSemicolon
  \SetKwProg{Fn}{Function}{:}{}
  \If{ \( \exists i^{\prime}\in\mathcal{C}^{j}\) s.t. \({\ab\|\widehat{\vec{\theta}}_{i_t,t} - \widehat{\vec{\theta}}_{i^{\prime},t} \|} > f(T_{i_t,t}) + f(T_{i^{\prime},t})\)}{
  Split user \(i_t\) from the cluster \(j\):
  \(\vec{M}^{j} = \vec{M}^{j} - \vec{S}_{i_t,t}, \ \ 
  \vec{b}^{j} = \vec{b}^{j} - \vec{b}_{i_t,t}, \ \ 
  T^{j} = T^{j} - T_{i_t,t}, \ \ 
  \mathcal{C}^{j} = \mathcal{C}^{j} \setminus \{i_t\}\)
  \(\overline{\vec{M}}^{j} = {\vec{M}}^{j} + \lambda \vec{I}, \ \ 
  \widehat{\vec{\theta}}^{j} = (\overline{\vec{M}}^{j})^{-1}\vec{b}^{j}\)\;
  Generate a new cluster \(j^{\prime}\) containing only user \(i_t\):
  \(\vec{M}^{j^{\prime}} = \vec{S}_{i_t,t}, \ \ 
  \vec{b}^{j^{\prime}} = \vec{b}_{i_t,t}, \ \ 
  T^{j^{\prime}} = T_{i_t,t}, \ \ 
  \mathcal{C}^{j^{\prime}} = \{i_t\}\)
  \(\overline{\vec{M}}^{j^{\prime}} = {\vec{M}}^{j^{\prime}} + \lambda \vec{I}, \ \ 
  \widehat{\vec{\theta}}^{j^{\prime}} = \widehat{\vec{\theta}}_{i_t,t}\)\;
  \( \mathcal{J} =  \mathcal{J} \cup \{j^{\prime}\}\)
  }
  \caption{Split} \label{algo:split}
\end{algorithm}

\begin{algorithm}[htb]
  \DontPrintSemicolon
  \SetKwProg{Fn}{Function}{:}{}
  \For{any two checked clusters $j_1, j_2 \in \mathcal{J}$}{
      \If{\(\ab\|\widehat{\vec{\theta}}^{j_1} - \widehat{\vec{\theta}}^{j_2} \| < f(T^{j_1}) + f({T}^{j_2})\)}{
      Merge clusters \(j_1\) and \(j_2\):
      \(\vec{M}^{j_1} = \vec{M}^{j_1} + \vec{M}^{j_2}, \ \ 
      \vec{b}^{j_1} = \vec{b}^{j_1} + \vec{b}^{j_2}, \ \ 
      T^{j_1} = T^{j_1} + T^{j_2}, \ \ 
      \mathcal{C}^{j_1} = \mathcal{C}^{j_1} \cup \mathcal{C}^{j_2}\)
      \(\overline{\vec{M}}^{j_1} = {\vec{M}}^{j_1} + \lambda \vec{I}, \ \ 
      \widehat{\vec{\theta}}^{j_1} = (\overline{\vec{M}}^{j_1})^{-1}\vec{b}^{j_1}\)\;
      \( \mathcal{J} =  \mathcal{J} \setminus \{j_2\}\)
      }
  }
  \caption{Merge} \label{algo:merge}
\end{algorithm}

The details of \UniSCLUB are shown in~\Cref{algo:sclucb}. 
The algorithm maintains information at two levels. 
At the cluster level, a cluster index \(\mathcal{J}\) contains the indices of currently existing clusters, and for each cluster \(j \in \mathcal{J}\), the algorithm maintains the set of users \(\mathcal{C}^j\) in this cluster and other corresponding information such as the estimated preference vector \(\widehat{\vec{\theta}}^j\). 
Initially, there is only a single cluster containing all users.
At the user level, for each user \(i\), the algorithm maintains the estimated preference vector \(\widehat{\vec{\theta}}_{i,t}\) at round \(t\) and other corresponding information.
Additionally, all users and clusters are associated with a ``checked'' or ``unchecked'' status to indicate the estimation accuracy of their preference vectors. 
\UniSCLUB proceeds in phases (Line~\ref{algo:sclub-forloop}) and each phase \(s \in \mathcal{N_{+}}\) consists of \(2^{s-1}\) rounds.
At the beginning of each phase, all users revert to the ``unchecked'' status (Line~\ref{algo:sclub-revert}).
When a user first appears in a phase, it will be marked as ``checked'' (Line~\ref{algo:sclub-mark}).
A cluster is marked as "checked" once all its users are checked.
The algorithm will only consider merging checked clusters, so as to avoid premature merging because of inaccurate preference vector estimation.
At round \(t\), a user \(i_t\) comes with a set \(\mathcal{A}_t\) of items (Line~\ref{algo:sclub-userindex}).
If \(t \leq 2T_0\) (with \(T_0\) defined in~\Cref{eq:T0-split}), \UniSCLUB uniformly selects the arm \(a_t\) from \(\mathcal{A}_t\) (Line~\ref{algo:sclub-uniformselect}).
Otherwise, it determines the cluster \(j\) to which the user \(i_t\) belongs and selects the item \(a_t\) based on the cluster information (\Cref{algo:sclub-clusterindex,algo:sclub-clusterselect}).
The algorithm updates the corresponding information of both the user and the cluster after receiving the feedback of the selected item (\Cref{algo:sclub-reward,algo:sclub-userupdate,algo:sclub-clusterupdate}).
Then the algorithm determines whether any split or merge operations are necessary (\Cref{algo:sclub-split,algo:sclub-merge}).
If the estimated preference vector of any user within the cluster \(j\) diverges from that of user \(i_t\), the algorithm will split \(i_t\) from the cluster (\Cref{algo:split}).
If the estimated preference vectors of two checked clusters are closely aligned, a merge operation will be performed (\Cref{algo:merge}) .

Since no clustering information is utilized during the uniform exploration period, \UniSCLUB can only update the user-level information.
Then the algorithm clusters all users at round \(2T_0\), and continues updating both user and cluster information subsequently. 
This implementation makes \UniSCLUB more efficient and robust compared to SCLUB.

\subsection{Theoretical Analysis of \UniSCLUB}
\label{sec:proof-unisclub}
\restateregretunisclub*
\begin{proof}
    From \Cref{lemma:split-and-merge-correct-after-T0}, \Cref{algo:sclucb} will have correct clusers after \(2T_0\) with probability \(\geq 1-3\delta\).

     Based on \Cref{lemma:split-and-merge-correct-after-T0} and with \(T_0\) replaced by \(2T_0\), we can derive the counterparts of \Cref{lemma:C-a-t} and Lemma~\ref{lemma:sum-x-sqaure}. 
     Then, similar to the proof of Theorem~\ref{thm:regret-uniclub}, we have
    \begin{align*}
        \E\ab[R(T)] &\leq 4+2T_0 + 2\ab(\sqrt{d \log\ab(1+\frac{TL^2}{d\lambda}) + 2\log(T)} + \sqrt{\lambda})\sqrt{2dmT \log\ab(1+\frac{TL^2}{\lambda d+4L^2})}\\
        &\leq 4 + 32u\log\ab(uT) + 8u\max\ab\{\frac{8L^2}{\lambda_{x}}\log\ab(udT), \frac{512d}{\gamma^2 \lambda_{x}}\log\ab(uT)\}\\
        &\quad+ 2\ab(\sqrt{d \log\ab(1+\frac{TL^2}{d\lambda}) + 2\log(T)} + \sqrt{\lambda})\sqrt{2dmT \log\ab(1+\frac{TL^2}{\lambda d+4L^2})}\\
        &= O\ab(\frac{ud}{\gamma^2 \lambda_{x}} \log(T) + d\sqrt{mT}\log(T)).
   \end{align*}
\end{proof}

\begin{lemma}\label{lemma:split-and-merge-correct-after-T0}
  With probability at least \(1-3\delta\), \Cref{algo:sclucb} can cluster all the users correctly after \(2T_0\), where 
  \begin{align}\label{eq:T0-split}
    T_0 \triangleq 16u\log\ab(\frac{u}{\delta}) + 4u\max\ab\{\frac{8L^2}{\lambda_{x}}\log\ab(\frac{ud}{\delta}), \frac{512d}{\gamma^2 \lambda_{x}}\log\ab(\frac{u}{\delta})\}.
  \end{align}
\end{lemma}
\begin{proof} 
    From Lemma~\ref{lemma:clusters-correct-after-T0}, with probability at least \(1-3\delta\), we have \(\ab\|\widehat{\vec{\theta}}_{i,t} - \vec{\theta}^{j(i)}\|_2 \leq f(T_{i,t}) \leq \frac{\gamma}{4}, \forall i \in [u]\),  when \(t\geq T_0\).

    We now show that under this condition, \Cref{algo:sclucb} will split well, i.e, the current clusters are subsets of true clusters. To guarantee this, at round \(t\), for the user \(i_t\) and the corresponding cluster \(j\), we need to verify that:
    (1) if the current cluster \(j\) is a subset of the ground-truth cluster of user \(i_t\), then user \(i_t\) will not be split from the cluster \(j\).
    (2) if the current cluster \(j\) contains users that are not in the same ground-truth cluster as user \(i_t\), then user \(i_t\) will be split from the cluster \(j\).
        
    To prove (1), we show the contrapositive of (1): if user \(i_t\) is split from the cluster \(j\), the cluster \(j\) contains users such that \(i_t\) and these users are from different ground-truth clusters.
    
    If user \(i_t\) is split from the cluster \(j\) by \Cref{algo:split}, i.e., there exists some user \(i^{\prime} \in \mathcal{C}^j\) such that \({\ab\|\widehat{\vec{\theta}}_{i_t,t} - \widehat{\vec{\theta}}_{i^{\prime},t} \|} > f(T_{i_t,t}) + f(T_{i^{\prime},t})\), 
    due to the triangle inequality, we have
    \begin{align*}
        \|{\vec{\theta}}^{j(i_t)} - \vec{\theta}^{j(i^{\prime})}\|_2 &\geq \|\widehat{\vec{\theta}}_{i_t,t} - \widehat{\vec{\theta}}_{i^{\prime},t}\|_2 - \|\widehat{\vec{\theta}}_{i_t,t} - {\vec{\theta}}^{j(i_t)}\|_2 - \|\widehat{\vec{\theta}}_{i^{\prime},t} -{\vec{\theta}}^{j(i^{\prime})}\|_2 \\
        & \geq \|\widehat{\vec{\theta}}_{i_t,t} - \widehat{\vec{\theta}}_{i^{\prime},t}\|_2 - f(T_{i_t,t}) - f(T_{i^{\prime},t}) > 0.
    \end{align*}

    By Assumption~\ref{assumption:item-regularity},  \(\ab\|\vec{\theta}^{j(i_t)} - \vec{\theta}^{j(i^{\prime})}\|_2 >0\) implies that users \(i_t,i^{\prime}\) are not in the same true cluster.

    To prove (2), we show that if there exists some user \(i^{\prime}\in \mathcal{C}^{j}\) and \(\ab\|\vec{\theta}^{j(i_t)} - \vec{\theta}^{j(i^{\prime})}\|_2 >\gamma\), we have
    \begin{align*}
    \ab\|\widehat{\vec{\theta}}_{i_t,t} - \widehat{\vec{\theta}}_{i^{\prime},t}\|_2
    &\geq \ab\|\vec{\theta}^{j(i_t)} - \vec{\theta}^{j(i^{\prime})}\|_2 - \ab\|\widehat{\vec{\theta}}_{i_t,t} - \vec{\theta}^{j(i_t)}\|_2 - \ab\|\widehat{\vec{\theta}}_{i^{\prime},t} - \vec{\theta}^{j(i^{\prime})}\|_2\\
    &> \gamma - \frac{\gamma}{4} -\frac{\gamma}{4} =\frac{\gamma}{2} \geq f(T_{i_t,t}) + f(T_{i^{\prime},t}),
    \end{align*}
    which satisfies the condition of splitting. Thus, user \(i_t\) will be split out from the current cluster \(j\).
    
    Therefore, when \(t > T_0\), each existing cluster at \(t\) will not contain users from different true clusters.

    Now we show that \Cref{algo:sclucb} will merge well so that only correct clusters remain. For two checked clusters \(j_1\) and \(j_2\),  we need to verify that: 
    (1) if \(j_1\) and \(j_2\) are not merged, they are from different true clusters.
    (2) if \(j_1\) and \(j_2\) are merged, they are from the same true cluster.
    
    We only consider the case where the existing clusters are subsets of ground-truth clusters because of splitting. For convenience, we denote the true preference vectors of clusters \(j_1\) and \(j_2\) by \(\vec{\theta}^{j_1}\) and \(\vec{\theta}^{j_2}\), which is reasonable since each of they only contains users with the same preference vector.
    
    To prove (1), if \(j_1\) and \(j_2\) are not merged, i.e., \(\ab\|\widehat{\vec{\theta}}^{j_1} - \widehat{\vec{\theta}}^{j_2} \| \geq f(T^{j_1}) + f({T}^{j_2})\), by the triangle inequality, we have 
    \begin{align*}
        \|\vec{\theta}^{j_1} - \vec{\theta}^{j_2}\|_2 &\geq \|\widehat{\vec{\theta}}^{j_1} - \widehat{\vec{\theta}}^{j_2}\|_2 - \|\widehat{\vec{\theta}}^{j_1} - \vec{\theta}^{j_1}\|_2 - \|\widehat{\vec{\theta}}^{j_2} - \vec{\theta}^{j_2}\|_2 \\
        &\geq \|\widehat{\vec{\theta}}^{j_1} - \widehat{\vec{\theta}}^{j_2}\|_2 - f(T^{j_1}) - f(T^{j_2}) > 0,
    \end{align*}
    which implies that \(j_1\) and \(j_2\) are from two different true clusters.

    To prove (2), we show a contraction: if \(j_1\) and \(j_2\) are merged, but they are from different true clusters, i.e., \(\|\vec{\theta}^{j_1} - \vec{\theta}^{j_2}\|_2 > \gamma \), we have
    \begin{align*}
        \|\widehat{\vec{\theta}}^{j_1} - \widehat{\vec{\theta}}^{j_2}\|_2 &\geq \|\vec{\theta}^{j_1} - \vec{\theta}^{j_2}\|_2 - \|\widehat{\vec{\theta}}^{j_1} - \vec{\theta}^{j_1}\|_2 - \|\widehat{\vec{\theta}}^{j_2}- \vec{\theta}^{j_2} \|_2 \\
        &\geq \gamma - f(T^{j_1}) - f(T^{j_2}) \geq \frac{\gamma}{2} \geq (f(T^{j_1}) + f(T^{j_2})),
    \end{align*}
    because that \(f(T^{j_1}) + f(T^{j_2}) \leq \frac{\gamma}{2}\).
    However, this contradicts the condition of merging. Therefore, if \(j_1\) and \(j_2\) are merged, they are from the same true cluster.

    We double the time \(T_0\) to \(2T_0\) to ensure that all users in each cluster can be checked and provide sufficient time for the split and merge operations.
    Therefore, after \(2T_0\), all the clusters are the ground-truth clusters with probability at least \(1-3\delta\).
\end{proof}

\section{Theoretical Analysis of \PhaseUniCLUB}
\label{sec:proof-of-phaseuniclub}
To bound the cumulative regret of \PhaseUniCLUB, we present the following lemmas.

\begin{lemma}\label{lemma:clucb-unknown-phase-prob}
Denote \(\gamma_s \triangleq 2^{-\frac{s}{2}}\) and  \(C_p = \frac{2048 ud \log\ab(\frac{u}{\delta})}{\lambda_x} \). 
With probability at least  \(1-\frac{3}{\alpha+1}\log_2{\ab(\frac{2^{\alpha+1}T}{C_p}+1)}\delta\), after the exploration subphase in any phase \(s=0, 1, \dots\), for all users \(i\in[u]\), we have
\begin{align*}
        \ab\|\widehat{\vec{\theta}}_{i,t} - \vec{\theta}^{j(i)}\|_2 &\leq \frac{\gamma_{s}}{4}.
    \end{align*}
\end{lemma}

\begin{proof}
    According to~\Cref{lemma:clusters-correct-after-T0}, with probability \(\ge 1- 3\delta\), after the exploration subphase in phase \(s=0, 1, \dots\), for any user \(i\in[u]\), we have 
    \begin{align*}
        \ab\|\widehat{\vec{\theta}}_{i,t} - \vec{\theta}^{j(i)}\|_2 &\leq \frac{\gamma_{s}}{4}.
    \end{align*}

     Denote the number of phases for \PhaseUniCLUB by \(N_p\).
     Since  phase \(s\) contains \(2^{\alpha s} T^{(s)}\) rounds, \(N_{p}\) is the maximum integer satisfying 
     \begin{align*}
        T^{\text{init}} + \sum_{s=0}^{N_{p}-1} 2^{\alpha s} \cdot 4u \cdot \frac{512d}{ 2^{-s} \lambda_{x}}\log\ab(\frac{u}{\delta}) \leq T.
    \end{align*}
    By some calculations, we have
    \begin{align*}
        N_p \leq \frac{1}{\alpha+1} \log_2{\ab(2^{\alpha+1}\cdot\frac{\lambda_{x}(T-T^{\text{init}})}{2048ud\log{\frac{u}{\delta}}}+1)} \coloneqq  \frac{1}{\alpha+1} \log_2{\ab(2^{\alpha+1}\frac{T-T^{\text{init}}}{C_p}+1)},
    \end{align*}
    where we denote \(C_p = \frac{2048 ud \log\ab(\frac{u}{\delta})}{\lambda_x} \).

    Applying a union bound over all phases, we have that with probability at least \(1-\frac{3}{\alpha+1}\log_2{\ab(\frac{2^{\alpha+1}T}{C_p}+1)}\delta\), after the exploration subphase in any phase \(s=0, 1, \dots\), for all users \(i\in[u]\), \(\|\widehat{\vec{\theta}}_{i,t} - \vec{\theta}^{j(i)}\|_2 \leq \frac{\gamma_{s}}{4}\) holds.
\end{proof}

\begin{lemma}\label{lemma:clucb-unknown-phase-cluster} 
With probability at least  \(1-\frac{3}{\alpha+1}\log_2{\ab(\frac{2^{\alpha+1}T}{C_p}+1)}\delta\), in any phase \(s\), for two users  \(i_1,i_2\in V_t\), we have
\begin{align*}
        \ab\| \vec{\theta}^{j(i_1)} -  \vec{\theta}^{j(i_2)}\|_2 &\leq {\gamma_{s}}.
    \end{align*}
\end{lemma}

\begin{proof}
 In phase \(s\), for any two users  \(i_1,i_2\in V_t\), it satisfies \(\ab\|\widehat{\vec{\theta}}_{i_1,t}-\widehat{\vec{\theta}}_{i_2,t}\| \leq f(T_{i_1,t}) + f(T_{i_2,t})\).
 By the triangle inequality and \Cref{lemma:clucb-unknown-phase-prob}, with probability at least \(1-\frac{3}{\alpha+1}\log_2{\ab(\frac{2^{\alpha+1}T}{C_p}+1)}\delta\),
 \begin{align*}
     \ab\| \vec{\theta}^{j(i_1)} -  \vec{\theta}^{j(i_2)}\|_2 
     &\leq  \ab\| \vec{\theta}^{j(i_1)} -  \widehat{\vec{\theta}}_{i_1,t}\|_2 +  \ab\| \widehat{\vec{\theta}}_{i_2,t} - \vec{\theta}^{j(i_2)}\|_2 + 
      \ab\|\widehat{\vec{\theta}}_{i_1,t} - \widehat{\vec{\theta}}_{i_2,t}\|_2 \\
     &\leq \frac{\gamma_{s}}{4} + \frac{\gamma_{s}}{4} + f(T_{i_1,t}) + f(T_{i_2,t}) \leq \gamma_s.
 \end{align*}
\end{proof}

\begin{lemma}\label{lemma:clucb-unknown-phase-error}
   With probability at least   \(1-\frac{4}{\alpha+1}\log_2{\ab(\frac{2^{\alpha+1}T}{C_p}+1)}\delta\),  in any phase \(s\), we have 
    \begin{align*}
        \ab|\vec{x}_a^\mathsf{T} \ab(\widehat{\vec{\theta}}_{V_t,t-1} - \vec{\theta}^{j(i_t)})| \leq C_{a,t} + \frac{2L^2\gamma_s}{\lambda_x} \1\{\gamma_s>\gamma\},
    \end{align*}
    where \(C_{a,t} \triangleq \beta \|\vec{x}_a\|_{\overline{\vec{M}}_{V_t,t-1}^{-1}}\) and  \(\beta = \sqrt{d \log(1+\frac{TL^2}{d\lambda}) + 2\log(\frac{1}{\delta})} + \sqrt{\lambda}\).
\end{lemma}

\begin{proof}
     In phase \(s\), the confidence radius for users' preference vectors is \(\frac{\gamma_{s}}{4}\). 
     If \(\gamma_{s} > \gamma\), user \(i_t\) may be mistakenly grouped into the wrong cluster, i.e., \(V_t\neq V_{j(i_t)}\).
     Accounting for this error, we analyze the following two cases:

    \textbf{Case 1}: user \(i_t\) is correctly clustered, i.e.,  \(V_t = V_{j(i_t)}\). 
    Similar to~\Cref{lemma:C-a-t}, and taking a union bound of \(N_p < \frac{1}{\alpha+1} \log_2{\ab(\frac{2^{\alpha+1}T}{C_p}+1)} \) phases, 
    \begin{align*}
        \ab|\vec{x}_a^\mathsf{T} \ab(\widehat{\vec{\theta}}_{V_t,t-1} - \vec{\theta}^{j(i_t)})| \leq \beta \|\vec{x}_a\|_{\overline{\vec{M}}_{V_t,t-1}^{-1}},
    \end{align*}
    with probability \(1-\frac{4}{\alpha+1}\log_2{\ab(\frac{2^{\alpha+1}T}{C_p}+1)}\delta\).

    \textbf{Case 2}: user \(i_t\) is mistakenly clustered, i.e., \(V_t\neq V_{j(i_t)}\), which happens when \(\gamma_{s} > \gamma\). Then,
      \begin{align*}
    &\widehat{\vec{\theta}}_{V_t,t-1} - \vec{\theta}^{j(i_t)} = \ab(\lambda I + \sum_{\tau \in [t-1]: i_{\tau} \in V_t} \vec{x}_{a_{\tau}}\vec{x}_{a_{\tau}}^\mathsf{T})^{-1} \ab(\sum_{\tau \in [t-1]: i_{\tau} \in V_t} r_{\tau} \vec{x}_{a_{\tau}}) - \vec{\theta}^{j(i_t)}\\
    =& \ab(\lambda I + \sum_{\tau \in [t-1]: i_{\tau} \in V_t} \vec{x}_{a_{\tau}}\vec{x}_{a_{\tau}}^\mathsf{T})^{-1} \ab(\sum_{\tau \in [t-1]: i_{\tau} \in V_t} \vec{x}_{a_{\tau}} \ab(\vec{x}_{a_{\tau}}^\mathsf{T} \vec{\theta}_{i_\tau} + \eta_{\tau})) - \vec{\theta}^{j(i_t)}\\
     =& \overline{\vec{M}}_{V_t,t-1}^{-1} \ab(\sum_{\tau \in [t-1]: i_{\tau} \in V_t} \vec{x}_{a_{\tau}} \vec{x}_{a_{\tau}}^\mathsf{T}(\vec{\theta}_{i_\tau} - \vec{\theta}_{i_t}) + \sum_{\tau \in [t-1]: i_{\tau} \in V_t} \vec{x}_{a_{\tau}}\vec{x}_{a_{\tau}}^\mathsf{T} \vec{\theta}_{i_t} +\sum_{\tau \in [t-1]: i_{\tau} \in V_t} \vec{x}_{a_{\tau}}\eta_{\tau}) - \vec{\theta}^{j(i_t)}\\
     =& \overline{\vec{M}}_{V_t,t-1}^{-1} \ab(\ab(\lambda \vec{I} + \sum_{\tau \in [t-1]: i_{\tau} \in V_t} \vec{x}_{a_{\tau}}\vec{x}_{a_{\tau}}^\mathsf{T} )\vec{\theta}_{i_t} - \lambda \vec{\theta}_{i_t} +\sum_{\tau \in [t-1]: i_{\tau} \in V_t} \vec{x}_{a_{\tau}}\eta_{\tau}) - \vec{\theta}^{j(i_t)} \\ 
     & +\overline{\vec{M}}_{V_t,t-1}^{-1} \sum_{\tau \in [t-1]: i_{\tau} \in V_t} \vec{x}_{a_{\tau}} \vec{x}_{a_{\tau}}^\mathsf{T}(\vec{\theta}_{i_\tau} - \vec{\theta}_{i_t})\\
     =& -\lambda \overline{\vec{M}}_{V_t,t-1}^{-1} \vec{\theta}^{j(i_t)} + \overline{\vec{M}}_{V_t,t-1}^{-1} \sum_{\tau \in [t-1]: i_{\tau} \in V_t} \vec{x}_{a_{\tau}}\eta_{\tau} + \overline{\vec{M}}_{V_t,t-1}^{-1} \sum_{\tau \in [t-1]: i_{\tau} \in V_t} \vec{x}_{a_{\tau}} \vec{x}_{a_{\tau}}^\mathsf{T}(\vec{\theta}_{i_\tau} - \vec{\theta}_{i_t}),
  \end{align*}
  where we denote \(\overline{\vec{M}}_{V_t,t-1} = \lambda I + \vec{M}_{V_t,t-1}\).
  
  Then from~\Cref{lemma:C-a-t}, the first two terms can be bounded by \( \beta \|\vec{x}_a\|_{\overline{\vec{M}}_{V_t,t-1}^{-1}} \) with high probability, and we have
  \begin{align}\label{eq:mistakenly-cluster}
     & \ab|\vec{x}_a^\mathsf{T}\ab(\widehat{\vec{\theta}}_{V_t,t-1} - \vec{\theta}^{j(i_t)})|
    \leq \beta \|\vec{x}_a\|_{\overline{\vec{M}}_{V_t,t-1}^{-1}}  + \ab|\vec{x}_a^\mathsf{T} \overline{\vec{M}}_{V_t,t-1}^{-1} \sum_{\tau \in [t-1]: i_{\tau} \in V_t} \vec{x}_{a_{\tau}} \vec{x}_{a_{\tau}}^\mathsf{T}(\vec{\theta}_{i_\tau} - \vec{\theta}_{i_t})|.
  \end{align}

    Now we bound the second term in~\Cref{eq:mistakenly-cluster}.
    \begin{align*}
        & \ab|\vec{x}_a^\mathsf{T} \overline{\vec{M}}_{V_t,t-1}^{-1} \sum_{\tau \in [t-1]: i_{\tau} \in V_t} \vec{x}_{a_{\tau}} \vec{x}_{a_{\tau}}^\mathsf{T}(\vec{\theta}_{i_\tau} - \vec{\theta}_{i_t})| \\
        \leq &\ab\|\vec{x}_a\|_2\ab\|\overline{\vec{M}}_{V_t,t-1}^{-1}\|_2\ab\| \sum_{\tau \in [t-1]: i_{\tau} \in V_t} \vec{x}_{a_{\tau}} \vec{x}_{a_{\tau}}^\mathsf{T}(\vec{\theta}_{i_\tau} - \vec{\theta}_{i_t})\|_2 \numberthis\label{eq:use-mmatrix-norm-1} \\
        \leq & L \ab\|\overline{\vec{M}}_{V_t,t-1}^{-1}\|_2\sum_{\tau \in [t-1]: i_{\tau} \in V_t} \ab\| \vec{x}_{a_{\tau}} \vec{x}_{a_{\tau}}^\mathsf{T}\|_2\|\vec{\theta}_{i_\tau} - \vec{\theta}_{i_t}\|_2  \numberthis\label{eq:use-mmatrix-norm-2}\\
        \leq & L \gamma_s \lambda_{\max}(\overline{\vec{M}}_{V_t,t-1}^{-1}) \sum_{\tau \in [t-1]: i_{\tau} \in V_t} \ab\| \vec{x}_{a_{\tau}} \vec{x}_{a_{\tau}}^\mathsf{T}\|_2  \numberthis\label{eq:use-mmatrix-norm-3}\\
        \leq & L \gamma_s  \frac{\sum_{\tau \in [t-1]: i_{\tau} \in V_t} L}{\lambda_{\min}(\overline{\vec{M}}_{V_t,t-1})} \numberthis\label{eq:use-mmatrix-norm-4}\\
        \leq & L \gamma_s  \frac{T_{V_t,t-1}L}{\lambda_{x}T_{V_t,t-1}/2 + \lambda} \numberthis\label{eq:use-mmatrix-norm-5}\\
        \leq & \frac{2L^2\gamma_s}{\lambda_x},
    \end{align*}
    where \(\|\overline{\vec{M}}_{V_t,t-1}^{-1}\|_2\) and \(\ab\| \vec{x}_{a_{\tau}} \vec{x}_{a_{\tau}}^\mathsf{T}\|_2 \) are the spectral norm of matrix \(\overline{\vec{M}}_{V_t,t-1}^{-1}\) and \(\vec{x}_{a_{\tau}} \vec{x}_{a_{\tau}}^\mathsf{T}\).
    \Cref{eq:use-mmatrix-norm-1} is because of the Cauchy–Schwarz inequality and the induced matrix norm inequality.
    \Cref{eq:use-mmatrix-norm-2} is due to \(\|\vec{x}_a\|_2 \leq L\) and the induced matrix norm inequality.
    \Cref{eq:use-mmatrix-norm-3} follows~\Cref{lemma:clucb-unknown-phase-cluster} and that \(\overline{\vec{M}}_{V_t,t-1}^{-1}\) is PSD.
    \Cref{eq:use-mmatrix-norm-4} is because of \(\ab\| \vec{x}_{a_{\tau}} \vec{x}_{a_{\tau}}^\mathsf{T}\|_2 = L\) and \Cref{eq:use-mmatrix-norm-5} follows~\Cref{lemma:bound-smallest-eigenvalue}.

    Therefore, we have
    \begin{align*}
        & \ab|\vec{x}_a^\mathsf{T}\ab(\widehat{\vec{\theta}}_{V_t,t-1} - \vec{\theta}^{j(i_t)})|
    \leq \beta \|\vec{x}_a\|_{\overline{\vec{M}}_{V_t,t-1}^{-1}}  + \frac{2L^2\gamma_s}{\lambda_x}\1\{\gamma_s>\gamma\}.
    \end{align*}

    Combining the two cases, we finish the proof of~\Cref{lemma:clucb-unknown-phase-error}.
\end{proof}

Now we can prove the regret upper bound of \PhaseUniCLUB.

\restateregretgammaunknown*

\begin{proof}
    For any round \(t\) in the UCB subphase of any phase \(s\),  with probability at least \(1-\frac{4}{\alpha+1}\log_2{\ab(\frac{2^{\alpha+1}T}{C_p}+1)}\delta\), the instantaneous regret at round \(t\) satisfies:
    \begin{align*}
       R_t &= \vec{x}_{a_t^{*}}^\mathsf{T} \vec{\theta}^{j(i_t)} - \vec{x}_{a_t}^\mathsf{T} \vec{\theta}^{j(i_t)}\\
        &= \vec{x}_{a_t^{*}}^\mathsf{T} \ab(\vec{\theta}^{j(i_t)}- \widehat{\theta}_{V_t,t-1}) + \ab(\vec{x}_{a_t^{*}}^\mathsf{T} \widehat{\theta}_{V_t,t-1} + C_{a_t^{*},t}) - \ab(\vec{x}_{a_t}^\mathsf{T} \widehat{\theta}_{V_t,t-1} + C_{a_t,t})\\
        &\quad +\vec{x}_{a_t}^\mathsf{T} \ab(\vec{\theta}^{j(i_t)}- \widehat{\theta}_{V_t,t-1}) -C_{a_t^{*},t} + C_{a_t,t}\\
        &\leq C_{a_t^{*},t} + \frac{2L^2\gamma_s}{\lambda_x}\1\{\gamma_s>\gamma\}  + C_{a_t,t}  + \frac{2L^2\gamma_s}{\lambda_x}\1\{\gamma_s>\gamma\} - C_{a_t^{*},t} + C_{a_t,t} \\ 
        &= 2C_{a_t,t} + \frac{4L^2\gamma_s}{\lambda_x}\1\{\gamma_s>\gamma\},
    \end{align*}
    where the inequality is due to Lemma~\ref{lemma:clucb-unknown-phase-error} and the UCB arm selection strategy of our algorithm. 

    Denote the regret in phase \(s\) by \(R^{(s)}\). Note that the round index \(t\) in phase \(s\) ranges from \(T^{\text{init}}+ \sum_{p=0}^{s-1}2^{\alpha p}T^{(p)} + T^{(s)}+1\) to \(T^{\text{init}}+ \sum_{p=0}^{s}2^{\alpha p}T^{(p)}\).
    With the same probability, for any phase \(s\), we have
    \begin{align*}
    R^{(s)} &\leq  T^{(s)} + \sum_{t=T^{\text{init}}+ \sum_{p=0}^{s-1}2^{\alpha p}T^{(p)} + T^{(s)}+1} ^{T^{\text{init}}+ \sum_{p=0}^{s}2^{\alpha p}T^{(p)}} R_t \\
    & \leq T^{(s)} + \frac{4L^2\gamma_s}{\lambda_x}\1\{\gamma_s>\gamma\} \cdot (2^{\alpha s}-1)T^{(s)} + \sum_{t=T^{\text{init}}+ \sum_{p=0}^{s-1}2^{\alpha p}T^{(p)} + T^{(s)}+1} ^{T^{\text{init}}+ \sum_{p=0}^{s}2^{\alpha p}T^{(p)}} 2C_{a_t,t}.
    \end{align*}

    Summing up the regret over all phases, we obtain
    \begin{align*}
        R(T) &\leq T^{\text{init}} + \sum_{s=0}^{N_p-1} R^{(s)}\\
        &\leq  T^{\text{init}} + \sum_{s=0}^{N_p-1} \ab( T^{(s)} +\frac{4L^2\gamma_s}{\lambda_x}\1\{\gamma_s>\gamma\} \cdot (2^{\alpha s}-1)T^{(s)}) + \sum_{t = T^{\text{init}} +1 }^{T} 2C_{a_t,t}\\
        &\leq T^{\text{init}} + \sum_{s=0}^{N_p-1} \ab( T^{(s)} +\frac{4L^2\gamma_s}{\lambda_x}\1\{\gamma_s>\gamma\} \cdot (2^{\alpha s}-1)T^{(s)}) +  2\beta\sqrt{2mTd \log\ab(1+\frac{TL^2}{\lambda d + 4L^2})}, \numberthis\label{eq:use-lemma5} \\
    \end{align*}
    where~\Cref{eq:use-lemma5} follows from~\Cref{lemma:sum-x-sqaure} and the proof of~\Cref{thm:regret-uniclub}.

    Now we bound the second term in~\Cref{eq:use-lemma5}.
    \begin{align*}
         \sum_{s=0}^{N_p-1} \ab( T^{(s)} +\frac{4L^2\gamma_s}{\lambda_x} \cdot (2^{\alpha s}-1)T^{(s)}) & \leq \sum_{s=0}^{N_p-1} T^{(s)} + \sum_{s=0}^{N_p-1} \frac{4L^2\gamma_s}{\lambda_x}\1\{\gamma_s>\gamma\} \cdot 2^{\alpha s}T^{(s)}  \numberthis\label{eq:sum-all-phases}.
    \end{align*}

    For the first term in \Cref{eq:sum-all-phases}, since \(\frac{1}{\alpha+1} \log_2{\ab(\frac{2^{\alpha+1}(T-T^{\text{init}})}{C_p}+1)}\), we have
    \begin{align*}
        \sum_{s=0}^{N_p-1} T^{(s)} = \sum_{s=0}^{N_p-1} 4u \cdot \frac{512d}{ 2^{-s} \lambda_{x}}\log\ab(\frac{u}{\delta}) & \leq \ab(\frac{2^{\alpha+1}(T-T^{\text{init}})}{C_p}+1)^{\frac{1}{\alpha+1}}\frac{2048ud\log(\frac{u}{\delta})}{\lambda_x} \\
        & \leq 2\ab(\frac{T-T^{\text{init}}}{C_p}+1)^{\frac{1}{\alpha+1}}\frac{2048ud\log(\frac{u}{\delta})}{\lambda_x} \\
        & \leq 2\ab(\ab(\frac{T}{C_p})^{\frac{1}{\alpha+1}} + 1)\frac{2048ud\log(uT)}{\lambda_x}. \numberthis\label{eq:alpha+1-cal} \\
    \end{align*}
    where \Cref{eq:alpha+1-cal} follows from that \((x+1)^{\frac{1}{\alpha+1}}\leq x^{\frac{1}{\alpha+1}} +1\) when \(\alpha>0\) and \(x>0\).
    
    For the second term in \Cref{eq:sum-all-phases}, note that regret due to misclustering arises only in the phases where \(\gamma_s > \gamma\) holds, i.e., \(s < \log_2{\frac{1}{\gamma^2}}\).
    Consequently, we have
    \begin{align*}
        \sum_{s=0}^{N_p-1} \frac{4L^2\gamma_s}{\lambda_x}\1\{\gamma_s>\gamma\} \cdot 2^{\alpha s}T^{(s)} & = \sum_{s=0}^{\lfloor\log_2{\frac{1}{\gamma^2}}\rfloor} \frac{4L^2\gamma_s}{\lambda_x}\cdot 2^{\alpha s}T^{(s)} \\
        & = \sum_{s=0}^{\lfloor\log_2{\frac{1}{\gamma^2}}\rfloor} \frac{4L^2 2^{-\frac{s}{2}} }{\lambda_x} \cdot 2^{\alpha s} \cdot  4u \cdot \frac{512d}{ 2^{-s} \lambda_{x}}\log\ab(\frac{u}{\delta})   \numberthis\label{eq:gamma-s-definition} \\
        & \leq\frac{2}{\gamma^{2\alpha+1}}\frac{4L^2}{\lambda_{x}} \frac{2048ud\log(\frac{u}{\delta})}{\lambda_x}.  \numberthis\label{eq:alpha-cal}
    \end{align*}
    where~\Cref{eq:gamma-s-definition} is due to \(\gamma_s = 2^{-\frac{s}{2}}\), and \Cref{eq:alpha-cal} follows from \(\frac{2^{\alpha+\frac{1}{2}}}{2^{\alpha+\frac{1}{2}}-1} < 2\) when \(\alpha>0\).

    Therefore, with probability at least   \(1-\frac{4}{\alpha+1}\log_2{\ab(\frac{2^{\alpha+1}T}{C_p}+1)}\delta\), we have
     \begin{align*}
     R(T) &\leq   T^{\text{init}} + 2\beta\sqrt{2mTd \log\ab(1+\frac{TL^2}{\lambda d + 4L^2})} +\ab(2\ab(\frac{T}{C_p})^{\frac{1}{\alpha+1}} + 2+ \frac{8L^2}{\gamma^{2\alpha +1}\lambda_x})\frac{2048ud\log(\frac{u}{\delta})}{\lambda_x}. \\
     \end{align*}

    Recall that \( T^{\text{init}} = 16u\log\ab(\frac{u}{\delta}) + 4u\cdot\frac{8L^2}{\lambda_{x}}\log\ab(\frac{ud}{\delta})\) and  \(C_p = \frac{2048ud\log(\frac{u}{\delta})}{\lambda_x} \). 
     Let \(\delta=\frac{1}{T}\) and by the law of total expectations, we have
    \begin{align*}
        \E[R(T)] &\leq  T^{\text{init}} + 2\beta\sqrt{2mTd \log\ab(1+\frac{TL^2}{\lambda d + 4L^2})} + \ab(2\ab(\frac{T}{C_p})^{\frac{1}{\alpha+1}} + 2+ \frac{8L^2}{\gamma^{2\alpha +1}\lambda_x})\frac{2048ud\log(uT)}{\lambda_x}\\
         & \quad +  \frac{4}{\alpha+1}\log_2{\ab(\frac{2^{\alpha+1}T}{C_p}+1)} \delta \cdot T \\
        & \leq 16u\log\ab(uT) + 4u\cdot\frac{8L^2}{\lambda_{x}}\log\ab(udT) + 2\beta\sqrt{2mTd \log\ab(1+\frac{TL^2}{\lambda d + 4L^2})} \\
        & \quad + 2T^{\frac{1}{\alpha+1}}\ab(\frac{2048ud\log(uT)}{\lambda_x})^{\frac{\alpha}{\alpha+1}}+ \frac{8L^2\cdot 2048ud\log(uT)}{\gamma^{2\alpha+1}\lambda_x^2} \\
        & \quad + \frac{4096 ud\log(uT)}{\lambda_x} + \frac{4}{\alpha+1} \log_2{\ab(2^{\alpha+1}\frac{2048 udT}{\lambda_x}+1)}\\
        & = \mathcal{O}\ab(\frac{ud}{\gamma^{2\alpha+1}\lambda_x^2}\log(T) + T^{\frac{1}{\alpha+1}}{\ab(\frac{ud}{\lambda_x}\log(T))}^{\frac{\alpha}{\alpha+1}} + d\sqrt{mT}\log(T)). \numberthis\label{eq:regret-with-alpha}
    \end{align*}

When \(\alpha > 1\), the second term in \Cref{eq:regret-with-alpha} exhibits a dependence on \(T\) that grows at a rate slower than \(\sqrt{T}\).
Thus, let \(\alpha=2\) and we have
\begin{align*}
      \E[R(T)] \leq \mathcal{O}\ab(\frac{ud}{\gamma^{5}\lambda_x^2}\log(T) + T^{\frac{1}{3}}{\ab(\frac{ud}{\lambda_x}\log(T))}^{\frac{2}{3}} + d\sqrt{mT}\log(T)).
\end{align*}
\end{proof}

\section{Theoretical Analysis of \SACLUB and \SASCLUB}
\label{sec:proof-of-advclub}

\begin{lemma}\label{lemma:bound-expected-smallest-eigenvalue-smoothed-adversary}
  Under the smoothed adversary setting, \SACLUB and \SASCLUB have the following lower bound on the expected minimum eigenvalue of \(\vec{x}_{a_t}\vec{x}_{a_t}^\mathsf{T}\):
  \[\lambda_{\min}\ab(\E\ab[\vec{x}_{a_t}\vec{x}_{a_t}^\mathsf{T}]) \geq c_1 \frac{\sigma^2}{\log K} \triangleq \widetilde{\lambda}_x,\]
  where \(c_1\) is some constant.
\end{lemma}

\begin{proof}
  Fix a time \(t\). Let \(\vec{Q}\) be a unitary matrix that rotates \(\widehat{\vec{\theta}}_{i_t,t}\) to align it with the x-axis, retaining its magnitude but zeroing out all components other than the first component, i.e., \(\vec{Q}\widehat{\vec{\theta}}_{i_t,t} = (\|\widehat{\vec{\theta}}_{i_t,t}\|, 0, 0, \dots, 0)\).
  Such \(\vec{Q}\) always exists because it just rotates the space.
  According to the UCB arm selection strategy, \(\vec{x}_{a_t} = \argmax_{a \in \mathcal{A}_t}\ab(\widehat{\vec{\theta}}_{i_t,t}^\mathsf{T} \vec{x}_{a} + C_{a,t})\), where \(C_{a,t} = \beta \|\vec{x}_a\|_{\overline{\vec{M}}_{V_t,t-1}^{-1}}\).
  Denote the \(i\)-th arm in \(\mathcal{A}_t\) as \(a_{t,i}\), we have
  \begin{align*}
    \lambda_{\min}\ab(\E\ab[\vec{x}_{a_t}\vec{x}_{a_t}^\mathsf{T}])
    =& \lambda_{\min}\ab(\E\ab[\vec{x}\vec{x}^\mathsf{T} \Bigm| \vec{x} = \argmax_{a \in \mathcal{A}_t}\ab(\widehat{\vec{\theta}}_{i_t,t}^\mathsf{T} \vec{x}_{a} + C_{a,t})])\\
    =& \min_{\vec{w}:\|\vec{w}\|=1} \vec{w}^\mathsf{T} \E\ab[\vec{x}\vec{x}^\mathsf{T} \Bigm| \vec{x} = \argmax_{a \in \mathcal{A}_t}\ab(\widehat{\vec{\theta}}_{i_t,t}^\mathsf{T} \vec{x}_{a} + C_{a,t})] \vec{w}\\
    =& \min_{\vec{w}:\|\vec{w}\|=1} \E\ab[(\vec{w}^\mathsf{T}\vec{x})^2 \Bigm| \vec{x} = \argmax_{a \in \mathcal{A}_t}\ab(\widehat{\vec{\theta}}_{i_t,t}^\mathsf{T} \vec{x}_{a} + C_{a,t})]\\
    \geq& \min_{\vec{w}:\|\vec{w}\|=1} \Var\ab[\vec{w}^\mathsf{T}\vec{x} \Bigm| \vec{x} = \argmax_{a \in \mathcal{A}_t}\ab(\widehat{\vec{\theta}}_{i_t,t}^\mathsf{T} \vec{x}_{a} + C_{a,t})]\\
    =& \min_{\vec{w}:\|\vec{w}\|=1} \Var\ab[(\vec{Q}\vec{w})^\mathsf{T}\vec{Q}\vec{x} \Bigm| \vec{x} = \argmax_{a \in \mathcal{A}_t}\ab((\vec{Q}\widehat{\vec{\theta}}_{i_t,t})^\mathsf{T} \vec{Q}\vec{x}_{a} + C_{a,t})]\numberthis\label{eq:unitary-property}\\
    =& \min_{\vec{w}:\|\vec{w}\|=1} \Var\ab[\vec{w}^\mathsf{T}\vec{Q}\vec{x} \Bigm| \vec{x} = \argmax_{a \in \mathcal{A}_t}\ab(\ab\|\widehat{\vec{\theta}}_{i_t,t}\| (\vec{Q}\vec{x}_{a})_{1} + C_{a,t})]\numberthis\label{eq:apply-Q}\\
    =& \min_{\vec{w}:\|\vec{w}\|=1} \Var\ab[\vec{w}^\mathsf{T}\vec{Q}\vec{\varepsilon} \Bigm| \vec{\varepsilon} = \argmax_{\vec{\varepsilon}_{t,i}: i \in [K]}\ab((\vec{Q}\vec{\mu}_{t,i} + \vec{Q}\vec{\varepsilon}_{t,i})_{1} + \frac{C_{a_{t,i},t}}{\|\widehat{\vec{\theta}}_{i_t,t}\|})]\numberthis\label{eq:decompose-arm-to-mu-and-epsilon}\\
    =& \min_{\vec{w}:\|\vec{w}\|=1} \Var\ab[\vec{w}^\mathsf{T}\vec{\varepsilon} \Bigm| \vec{\varepsilon} = \argmax_{\vec{\varepsilon}_{t,i}: i \in [K]}\ab((\vec{Q}\vec{\mu}_{t,i} + \vec{\varepsilon}_{t,i})_{1} + \frac{C_{a_{t,i},t}}{\|\widehat{\vec{\theta}}_{i_t,t}\|})]\numberthis\label{eq:remove-Q-by-rotation-invariance}\\
  \end{align*}
  where \Cref{eq:unitary-property} uses the property of unitary matrices: \(\vec{Q}^\mathsf{T}\vec{Q}=\vec{I}\).
  \Cref{eq:apply-Q} applies matrix \(\vec{Q}\) so only the first component is non-zero and we use the fact that minimizing over \(\vec{Q}\vec{w}\) is equivalent to over \(\vec{w}\).
  \Cref{eq:decompose-arm-to-mu-and-epsilon} follows because each arm \(\vec{x} = \vec{\mu} + \vec{\varepsilon}\) and adding a constant a to a random variable does not change its variance.
  \Cref{eq:remove-Q-by-rotation-invariance} is due to the rotation invariance of symmetrically truncated Gaussian distributions.

  Since \(\vec{\varepsilon}_{t,i} \sim \mathcal{N}(0, \sigma^2\vec{I})\) conditioned on \(|(\vec{\varepsilon}_{t,i})_j| \leq R, \forall j \in [d]\), by the property of (truncated) multivariate Gaussian distributions, the components of \(\vec{\varepsilon}_{t,i}\) can be equivalently regarded as \(d\) independent samples from a (truncated) univariate Gaussian distribution, i.e., \((\vec{\varepsilon}_{t,i})_j \sim \mathcal{N}(0,\sigma^2)\) conditioned on \(|(\vec{\varepsilon}_{t,i})_j| \leq R, \forall j \in [d]\).
  Therefore, we have
  \begin{align*}
    \Var\left[\vec{w}^\mathsf{T}\vec{\varepsilon}\right]
    = \Var\left[\sum_{i=1}^{d} \vec{w}_i \vec{\varepsilon}_i\right]
    = \sum_{i=1}^{d} \vec{w}_i^2\Var\left[ \vec{\varepsilon}_i\right],
  \end{align*}
  where the exchanging of variance and summation is due to the independence of \(\vec{\varepsilon}_i\). Therefore, let \(p_{t,i}= (\vec{Q}\vec{\mu}_{t,i})_{1} + \frac{C_{a_{t,i},t}}{\|\widehat{\vec{\theta}}_{i_t,t}\|}\), we can write
  \begin{align*}
    & \min_{\vec{w}:\|\vec{w}\|=1} \Var\ab[\vec{w}^\mathsf{T}\vec{\varepsilon} \Bigm| \vec{\varepsilon} = \argmax_{\vec{\varepsilon}_{t,i}: i \in [K]}\ab((\vec{\varepsilon}_{t,i})_{1} + p_{t,i})]\\
    =&\min_{\vec{w}:\|\vec{w}\|=1} \sum_{j=1}^{d} \vec{w}_j^2\Var\ab[\vec{\varepsilon}_j \Bigm| \vec{\varepsilon} = \argmax_{\vec{\varepsilon}_{t,i}: i \in [K]}\ab((\vec{\varepsilon}_{t,i})_{1} + p_{t,i})]\\
    =&\min_{\vec{w}:\|\vec{w}\|=1} \left\{\vec{w}_1^2\Var\ab[\vec{\varepsilon}_1 \Bigm| \vec{\varepsilon} = \argmax_{\vec{\varepsilon}_{t,i}: i \in [K]}\ab((\vec{\varepsilon}_{t,i})_{1} + p_{t,i})]\right.\\
      & \left.+ \sum_{j=2}^{d} \vec{w}_j^2 \Var\ab[\vec{\varepsilon}_j \Bigm| \vec{\varepsilon} = \argmax_{\vec{\varepsilon}_{t,i}: i \in [K]}\ab((\vec{\varepsilon}_{t,i})_{1} + p_{t,i})]\right\}\\
    =&\min_{\vec{w}:\|\vec{w}\|=1} \left\{\vec{w}_1^2\Var\ab[\vec{\varepsilon}_1 \Bigm| \vec{\varepsilon} = \argmax_{\vec{\varepsilon}_{t,i}: i \in [K]}\ab((\vec{\varepsilon}_{t,i})_{1} + p_{t,i})] + \sum_{j=2}^{d} \vec{w}_j^2 \Var\ab[\vec{\varepsilon}_j]\right\}\\
    =&\min_{\vec{w}:\|\vec{w}\|=1} \left\{\vec{w}_1^2\Var\ab[\vec{\varepsilon}_1 \Bigm| \vec{\varepsilon} = \argmax_{\vec{\varepsilon}_{t,i}: i \in [K]}\ab((\vec{\varepsilon}_{t,i})_{1} + p_{t,i})] + (1-\vec{w}_1^2) \sigma^2\right\}\\
    =&\min \left\{\Var\ab[\vec{\varepsilon}_1 \Bigm| \vec{\varepsilon} = \argmax_{\vec{\varepsilon}_{t,i}: i \in [K]}\ab((\vec{\varepsilon}_{t,i})_{1} + p_{t,i})], \sigma^2\right\} \geq c_1 \frac{\sigma^2}{\log K},
  \end{align*}
  where in the last inequality, we use Lemma 15 and Lemma 14 in \citet{sivakumar-2020-structured-linear} and get
  \[\Var\ab[\vec{\varepsilon}_1 \Bigm| \vec{\varepsilon} = \argmax_{\vec{\varepsilon}_{t,i}: i \in [K]}\ab((\vec{\varepsilon}_{t,i})_{1} + p_{t,i})] \geq \Var\ab[\vec{\varepsilon}_1 \Bigm| \vec{\varepsilon} = \argmax_{\vec{\varepsilon}_{t,i}: i \in [K]} (\vec{\varepsilon}_{t,i})_{1}] \geq c_1\frac{\sigma^2}{\log K}.\]
\end{proof}

\begin{lemma}\label{lemma:bound-smallest-eigenvalue-smoothed-adversary}
  Under the smoothed adversarial setting, for algorithms \SACLUB and \SASCLUB, with probability at least \(1-\delta\) for any \(\delta \in (0,1)\), if \(T_{i,t} \geq \frac{8(1+\sqrt{d}R)^2\log(K)}{c_1\sigma^2}\log\ab(\frac{ud}{\delta}) = \frac{8(1+\sqrt{d}R)^2}{\widetilde{\lambda}_{x}}\log\ab(\frac{ud}{\delta})\) for all users \(i \in [u]\), we have
  \[\lambda_{\min}\ab(\vec{S}_{i,t}) \geq \frac{c_1\sigma^2T_{i,t}}{2\log K} = \frac{\widetilde{\lambda}_{x} T_{i,t}}{2}, \forall i \in [u],\]
  where \(c_1\) is some constant and \(\widetilde{\lambda}_x \triangleq \frac{c_1\sigma^2}{\log K}\).
\end{lemma}
\begin{proof}
The proof follows the same techniques as \Cref{lemma:bound-smallest-eigenvalue}.
The main difference lies in two aspects:
(1) in the smoothed adversary setting, the length of feature vectors is bounded by \(1+\sqrt{d}R\) instead of \(L\).
Therefore, \(\lambda_{\max}(\vec{x}_{a_{\tau}}\vec{x}_{a_{\tau}}^\mathsf{T}) \leq (1+\sqrt{d}R)^2\).
(2) the computation of \(\mu_{\min}\) needs to refer to \Cref{lemma:bound-expected-smallest-eigenvalue-smoothed-adversary}.

Specifically, to compute \(\mu_{\min}\), by the super-additivity of the minimum eigenvalue (due to Weyl's inequality) and \Cref{lemma:bound-expected-smallest-eigenvalue-smoothed-adversary}, we have
  \begin{align*}
    \mu_{\min}
    = \lambda_{\min}\ab(\sum_{\tau \in [t]: i_{\tau}=i} \E[\vec{x}_{a_{\tau}}\vec{x}_{a_{\tau}}^\mathsf{T}])
    \geq \sum_{\tau \in [t]: i_{\tau}=i} \lambda_{\min}\ab(\E[\vec{x}_{a_{\tau}}\vec{x}_{a_{\tau}}^\mathsf{T}])
    \geq T_{i,t} c_1 \frac{\sigma^2}{\log K}.
  \end{align*}
  On the other hand, we have
  \[\mu_{\max} = \lambda_{\max}\ab(\sum_{\tau \in [t]: i_{\tau}=i} \E[\vec{x}_{a_{\tau}}\vec{x}_{a_{\tau}}^\mathsf{T}])\]
  So by \Cref{lemma:matrix-chernoff}, we have for any \(\varepsilon \in (0,1)\),
  \begin{align*}
    \Pr\left[\lambda_{\min}(\vec{S}_{i,t}) \leq (1-\varepsilon)T_{i,t}c_1 \frac{\sigma^2}{\log K}\right]
    &\leq \Pr\left[\lambda_{\min}(\vec{S}_{i,t}) \leq (1-\varepsilon)\mu_{\min}\right]\\
    &\leq d \ab[\frac{e^{-\varepsilon}}{(1-\varepsilon)^{1-\varepsilon}}]^{\mu_{\min}/(1+\sqrt{d}R)^2}\\
    &\leq d \ab[\frac{e^{-\varepsilon}}{(1-\varepsilon)^{1-\varepsilon}}]^{\frac{T_{i,t}c_1 \sigma^2}{\log(K)(1+\sqrt{d}R)^2}},
  \end{align*}
  where the last inequality is because \(e^{-x}\) is decreasing.
  Then the proof follows by the same derivation as \Cref{lemma:bound-smallest-eigenvalue}, except that \(\lambda_x\) is replaced by \(\widetilde{\lambda}_x\triangleq \frac{c_1\sigma^2}{\log K}\).
\end{proof}

\restateregretsmoothedadversary*
\begin{proof}
  For algorithm \SACLUB, similar to the proof of \Cref{thm:regret-uniclub} under the stochastic context setting, we need to derive the counterparts of \Cref{lemma:clusters-correct-after-T0}, \Cref{lemma:C-a-t}, and \Cref{lemma:sum-x-sqaure} under the smoothed adversary setting according to \Cref{lemma:bound-expected-smallest-eigenvalue-smoothed-adversary} and \Cref{lemma:bound-smallest-eigenvalue-smoothed-adversary}.
  The proofs use almost the same techniques with \(\lambda_x\) replaced by \(\widetilde{\lambda}_x\).
  Similarly, the proof for algorithm \SASCLUB requires the counterparts of Lemmas used in the proof of \Cref{thm:regret-unisclub}, thus we skip the details.
\end{proof}

\section{Technical Inequalities}
\label{sec:technical-inequalities}
We present the technical inequalities used throughout the proofs.
For inequalities from existing literature, we provide detailed references for readers' convenience.

\begin{lemma}[Matrix Chernoff, Corollary 5.2 in \citet{tropp-2011-user-friendly}]\label{lemma:matrix-chernoff}
  Consider a finite sequence \(\set{\vec{X}_k}\) of independent, random, self-adjoint matrices with dimension \(d\).
  Assume that each random matrix satisfies
  \[\vec{X}_k \succeq \vec{0} \quad \text{and} \quad \lambda_{\text{max}}(\vec{X}_k) \leq R \quad\text{almost surely.}\]
  Define
  \[\vec{Y}:=\sum_{k} \vec{X}_k \quad \text{and} \quad \mu_{\text{min}} := \lambda_{\min}\ab(\E[\vec{Y}]) = \lambda_{\min}\ab(\sum_{k} \E[\vec{X}_{k}]).\]
  Then, for any \(\delta \in (0,1)\),
  \[\Pr\left[\lambda_{\min}\ab(\sum_{k} \vec{X}_k) \leq (1-\delta)\mu_{\min}\right] \leq d\ab[\frac{e^{-\delta}}{(1-\delta)^{1-\delta}}]^{\mu_{\min}/R}.\]
\end{lemma}

\begin{lemma}[Lemma 8 in \citet{li-2018-online-clustering}]\label{lemma:bound-sum-of-bernoulli}
  Let \(x_1, x_2, \dots, x_n\) be independent Bernoulli random variables with mean \(0<p\leq \frac{1}{2}\). Let \(\delta>0, B>0\), then with probability at least \(1-\delta\),
  \begin{align*}
    \sum_{s=1}^{t} x_s \geq B, \quad \forall t \geq \frac{16}{p} \log\ab(\frac{1}{\delta}) + \frac{4B}{p}.
  \end{align*}
\end{lemma}

\begin{lemma}\label{lemma:basic-arithmetic}
  For \(a>0, b>0, ab \geq 1\), if \(t\geq 2a \log(ab)\), then
  \begin{align*}
    t\geq a\log(1+bt).
  \end{align*}
\end{lemma}
\begin{proof}
  Since \(t\) increases faster than \(a\log(1+bt)\), it suffices to prove \(t\geq a\log(1+bt)\) for \(t = 2a\log(ab)\).
  Equivalently, we only need to show:
  \begin{align*}
    2a\log(ab) \geq a\log(1+2ab\log(ab)) = a\log(ab) + a\log\ab(2\log\ab(e^{-2ab}ab)),
  \end{align*}
  which follows by observing that \(ab \geq 2\log(e^{-2ab}ab) = \frac{1}{ab} + 2\log(ab)\) holds when \(ab\geq1\).
\end{proof}

\begin{lemma}\label{lemma:uniform-select-maintain-distribution}
Let \(X_1,X_2,\dots,X_n\) be a sequence of \(n\) independent and identically distributed (i.i.d.) random variables, each following a distribution \(P\). Let \(Y\) be a random variable that is uniformly selected from the set \(\set{X_1,X_2,\dots,X_n}\), then \(Y\) follows the same distribution \(P\).
\end{lemma}
\begin{proof}
It suffices to show \(\Pr[Y \in A] = \Pr[X_1 \in A]\) for any measurable set \(A\).
By the law of total probability, we can express the probability that \(Y\) falls into the set \(A\) as:
\begin{align*}
  \Pr[Y \in A]
  &= \sum_{i=1}^n \Pr[Y \in A \mid Y=X_i] \Pr[Y=X_i]\\
  &= \sum_{i=1}^n \Pr[X_i \in A] \Pr[Y=X_i]\\
  &= \Pr\left[X_1 \in A\right],
\end{align*}
where we use the fact that \(X_i\) are i.i.d. and \(\Pr\left[Y=X_i\right]=1/n\).
\end{proof}

\begin{lemma}[Determinant-trace inequality, Lemma 10 in \citet{abbasi-2011-improved}]\label{lemma:det-trace-inequality}
Suppose \(\vec{X}_1, \vec{X}_2, \dots, \vec{X}_t \in \RR^d\) and for any \(1\leq s\leq t\), \(\|\vec{X}_s\|_{2} \leq L\). Let \(\overline{\vec{V}}_t = \lambda \vec{I} + \sum_{s=1}^{t} \vec{X}_s \vec{X}_s^\mathsf{T}\) for some \(\lambda > 0\). Then,
\[\det(\overline{\vec{V}}_t) \leq \ab(\lambda+\frac{tL^2}{d})^d.\]
\end{lemma}

\section{More Experiments}\label{sec:ablation-study}
\subsection{Evaluation Under the Smoothed Adversarial Context Setting}\label{subsec:exper-adversarial}
To implement the smoothed adversarial contexts, for each arm in the synthetic dataset and three real-world recommendation system datasets, we add a Gaussian noise vector sampled from \(\mathcal{N}(0,\vec{I}/10)\).
As illustrated in \Cref{fig:regret-adversarial-context}, for all the datasets, \SACLUB and \SASCLUB (which are essentially adapted versions of CLUB and SCLUB) outperform LinUCB-Ind and LINUCB-One.  Additionally, under the smoothed adversarial setting, LinUCB-Ind and LINUCB-One exhibit more significant fluctuations compared to \SACLUB and \SASCLUB, highlighting the increased complexity of the smoothed adversarial setting relative to the stochastic setting. This corroborates our theoretical results, showing that with some minor changes, the existing algorithms CLUB and SCLUB can be extended to the more practical smoothed adversarial setting, which is closer to the original setting of contextual linear bandits~\citep{abbasi-2011-improved}.

\begin{figure}[htbp]
  \centering
  \includegraphics[width=\linewidth]{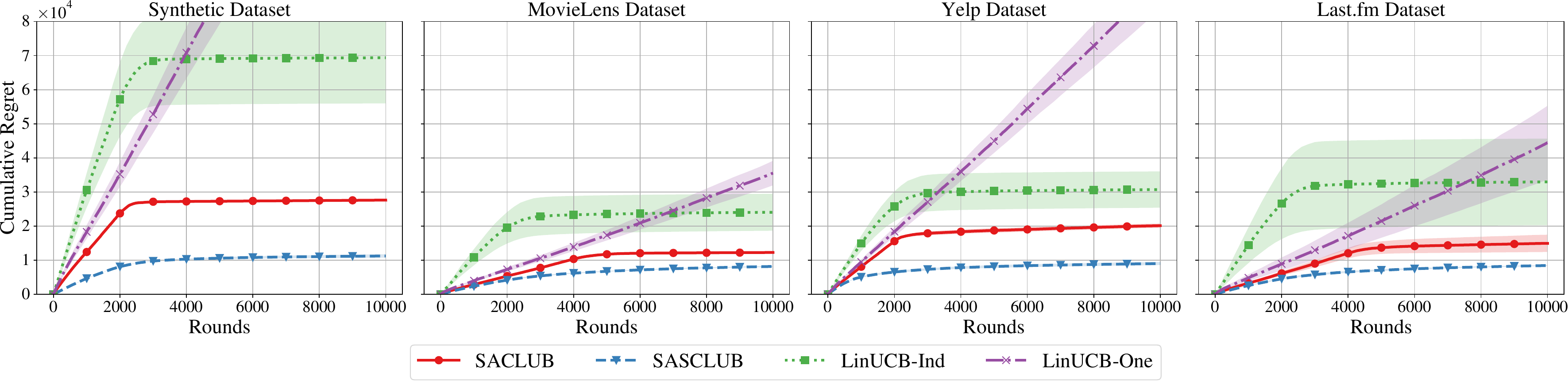}
  \caption{\label{fig:regret-adversarial-context} Comparison of cumulative regrets under the smoothed adversarial context setting.}
\end{figure}

\subsection{Cumulative Regret with Different Arm Set Sizes}
Under the stochastic context setting, we evaluate the impact of the arm set size \(\ab|\mathcal{A}_t| = K\) by adjusting \(K=80, 100, 120, 140\) using the Yelp dataset, since it has the largest number of users.
As demonstrated in \Cref{fig:regret-vs-arm}, there is no substantial amplification of the cumulative regrets as the arm set size \(K\) increases.
This observation validates our theoretical results (\Cref{thm:regret-uniclub,thm:regret-unisclub}), where the regret upper bounds of \UniCLUB and \UniSCLUB do not involve $K$.

\begin{figure}[htbp]
  \centering
  \includegraphics[width=\linewidth]{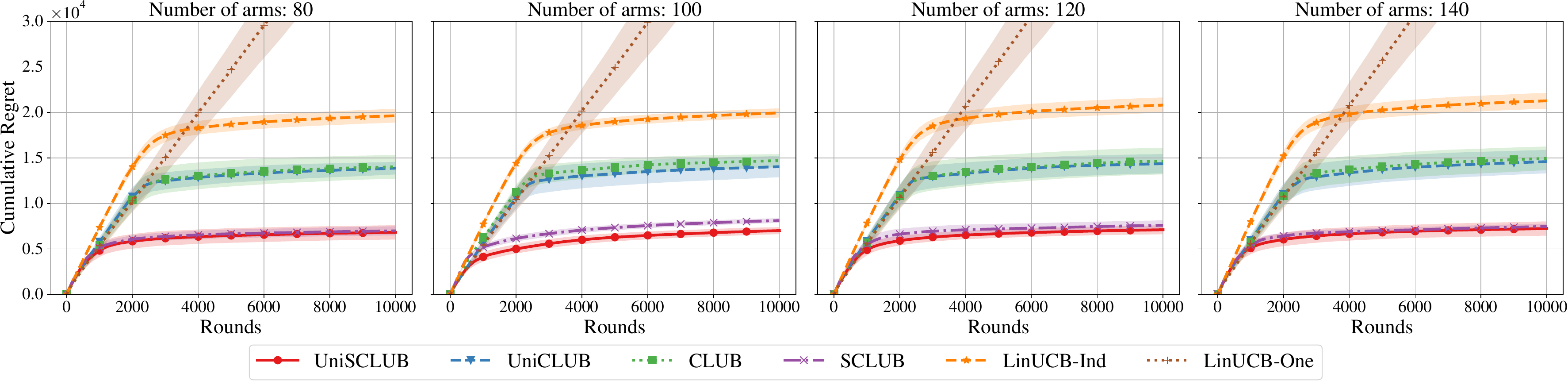}
  \caption{\label{fig:regret-vs-arm} Comparison of cumulative regrets with different arm set sizes.}
\end{figure}

\begin{figure}[htbp]
  \centering
  \includegraphics[width=\linewidth]{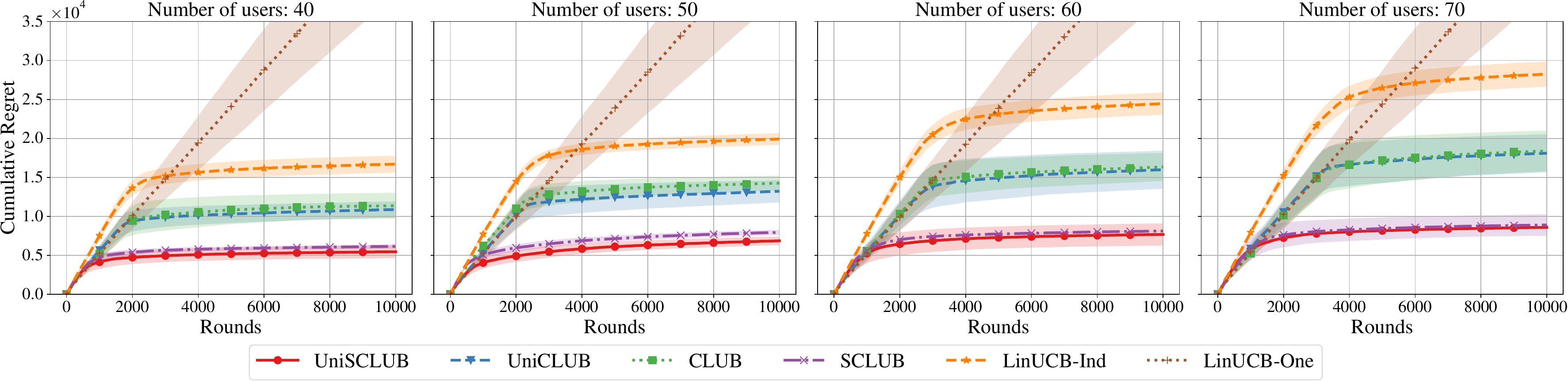}
  \caption{\label{fig:regret-vs-user} Comparison of cumulative regrets with different numbers of users.}
\end{figure}

\subsection{Cumulative Regret with Different User Numbers}
To examine the effect of the number of users, we adjust the number of users \(u=40, 50, 60, 70\) while keeping the number of clusters \(m=10\) using the Yelp dataset.
As shown in \Cref{fig:regret-vs-user}, our proposed algorithms exhibit significant advantages compared to the baselines.
Additionally, as the number of users increases, the cumulative regrets also increase.
This is expected because a larger number of users poses a greater challenge in learning their preference vectors and identifying the cluster structures.

\subsection{Comparison with Non-Clustering-Based Baselines}
In this subsection, we compare our algorithms with two additional graph-based baseline algorithms, GOB.Lin~\citep{cesa-2013-gang} and GraphUCB~\citep{yang-2020-laplacian}, within the stochastic context setting.
While both GOB.Lin and GraphUCB leverage user similarities to enhance preference estimation, their methodologies fundamentally differ from ours (and all the other clustering-based algorithms) in two aspects: (1) Neither GOB.Lin nor GraphUCB assume the existence of user clusters or explicitly perform clustering, and (2) both algorithms assume prior knowledge of a user relationship graph. 
Additionally, it is worth noting that both GOB.Lin and GraphUCB have significantly higher computational complexity compared to clustering-based algorithms, as they require operations such as multiplication and inversion of high-dimensional matrices of size \(\RR^{ud \times ud}\) for item recommendation, where \(u\) is the number of users and \(d\) is the feature dimension.
In contrast, clustering-based algorithms (including ours) only involve matrix manipulations of size \(\RR^{d \times d}\).

For the experiments, due to the high computational overhead of GOB.Lin and GraphUCB, we randomly select 20 users (instead of 50 as in other experiments) and divide them into 4 clusters.
At each round \(t\), a user \(i_t\) is uniformly drawn from the 20 users, and 20 items are randomly sampled from the full set of arms to form the arm set \(\mathcal{A}_t\).
For GOB.Lin and GraphUCB, the user graph is constructed by connecting users within the same cluster.

The cumulative regret results are shown in \Cref{fig:regret-more-baselines}.
Both LinUCB-One and LinUCB-Ind exhibit significantly higher cumulative regret than the other algorithms that incorporate user similarity.
Among all the clustering-based algorithms, our algorithms \UniCLUB and \UniSCLUB consistently outperform their respective counterparts, CLUB and SCLUB, across all four datasets, demonstrating the effectiveness of our proposed uniform exploration strategy.
When compared to the new baselines, \UniSCLUB achieves superior performance compared to GOB.Lin and GraphUCB across all four datasets, and \UniCLUB also outperforms GOB.Lin.
Although GraphUCB slightly surpasses \UniCLUB on the Last.fm dataset and achieves competitive results on the synthetic dataset, it incurs higher regret on the Movielens and Yelp datasets compared to \UniCLUB.
These results align with expectations, as our setting assumes that users are clustered, which allows clustering-based methods to explicitly exploit this structure for improved performance.

\begin{figure}[htbp]
  \centering
  \includegraphics[width=\linewidth]{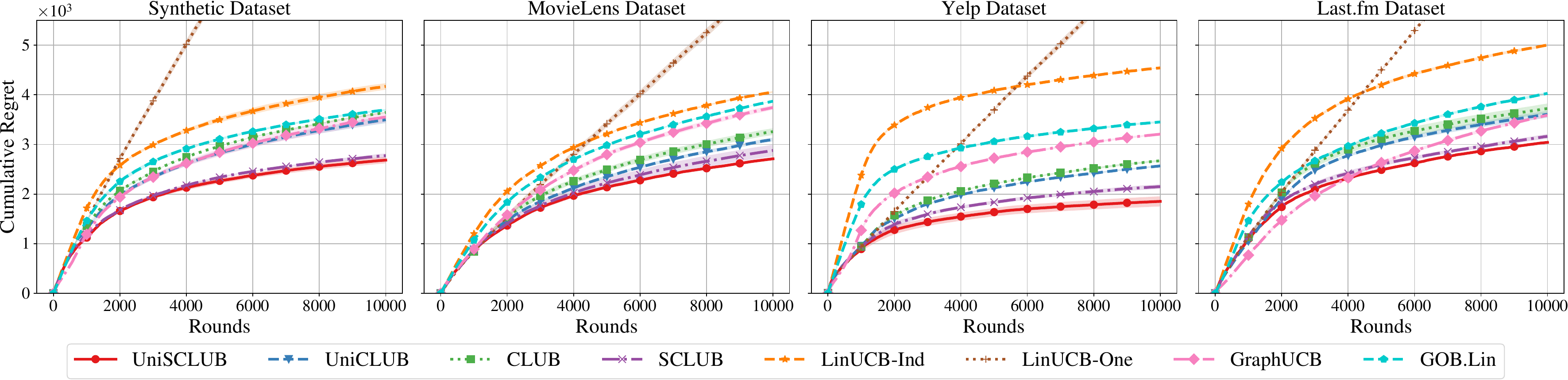}
  \caption{\label{fig:regret-more-baselines} Comparison of cumulative regrets with more baselines.}
\end{figure}

To demonstrate computational complexity, we also measure the average running time of all algorithms across datasets, with the results summarized in \Cref{tab:runtime-comparison}.
As shown in the table, GOB.Lin and GraphUCB exhibit significantly higher average running times, exceeding those of other baselines by more than an order of magnitude.
In contrast, our algorithms, \UniCLUB and \UniSCLUB, achieve computational efficiency comparable to that of the other baseline methods.

\begin{table}[htb]
\centering
\setlength{\extrarowheight}{0pt}
\addtolength{\extrarowheight}{\aboverulesep}
\addtolength{\extrarowheight}{\belowrulesep}
\setlength{\aboverulesep}{0pt}
\setlength{\belowrulesep}{0pt}
\caption{Comparison of average running time (in seconds).}
\label{tab:runtime-comparison}
\resizebox{\linewidth}{!}{%
\begin{tabular}{lrrrrrrrr} 
\toprule
\diagbox{\textbf{Datasets}}{\textbf{Time(s)}}{\textbf{Algorithms}} & \textbf{LinUCB-One} & \textbf{LinUCB-Ind} & \textbf{CLUB} & \textbf{SCLUB} & {\cellcolor[rgb]{0.753,0.753,0.753}}\textbf{UniCLUB} & {\cellcolor[rgb]{0.753,0.753,0.753}}\textbf{UniSCLUB} & \textbf{GraphUCB} & \textbf{GOB.Lin}  \\ 
\midrule
\textbf{Synthetic}                                                 & 2.79                & 2.61                & 3.02          & 7.70           & {\cellcolor[rgb]{0.753,0.753,0.753}}3.02             & {\cellcolor[rgb]{0.753,0.753,0.753}}8.28              & 128.32            & 204.55            \\
\textbf{MovieLens}                                                 & 9.78                & 9.62                & 10.11         & 16.96          & {\cellcolor[rgb]{0.753,0.753,0.753}}9.96             & {\cellcolor[rgb]{0.753,0.753,0.753}}16.80             & 117.25            & 214.96            \\
\textbf{Yelp}                                                      & 9.81                & 9.59                & 10.17         & 15.05          & {\cellcolor[rgb]{0.753,0.753,0.753}}10.19            & {\cellcolor[rgb]{0.753,0.753,0.753}}12.13             & 119.86            & 216.35            \\
\textbf{Last.fm}                                                   & 4.66                & 1.44                & 5.00          & 12.18          & {\cellcolor[rgb]{0.753,0.753,0.753}}5.00             & {\cellcolor[rgb]{0.753,0.753,0.753}}12.42             & 106.04            & 231.25            \\
\bottomrule
\end{tabular}
}
\end{table}

\end{document}